\documentclass[11pt]{article}
\usepackage[T1]{fontenc}
\usepackage[english]{babel}
\usepackage[margin=1in]{geometry}

\usepackage[utf8]{inputenc} % allow utf-8 input
\usepackage{hyperref}       % hyperlinks
\usepackage{url}            % simple URL typesetting
\usepackage{booktabs}       % professional-quality tables
\usepackage{amsfonts}       % blackboard math symbols
\usepackage{nicefrac}       % compact symbols for 1/2, etc.
\usepackage{microtype}      % microtypography

\title{Contextual Blocking Bandits}

\author{%
  Soumya Basu \thanks{These authors contributed equally to this work.} \\
  Department of ECE\\
  The University of Texas at Austin\\
  Austin, TX 78712, USA\\
  \texttt{basusoumya@utexas.edu} \\
  \and
  Orestis Papadigenopoulos{\normalfont \textsuperscript{*}} \\
  Department of Computer Science\\
  The University of Texas at Austin\\
  Austin, TX 78712, USA\\
  \texttt{papadig@cs.utexas.edu} \\
  \and
  Constantine Caramanis\\
Department of ECE\\
The University of Texas at Austin\\
Austin, TX 78712, USA\\
\texttt{constantine@utexas.edu} \\
\and 
    Sanjay Shakkottai \\
    Department of ECE\\
    The University of Texas at Austin\\
    Austin, TX 78712, USA\\
    \texttt{sanjay.shakkottai@utexas.edu}
}

\usepackage[square,sort,comma,numbers]{natbib}
\usepackage{amsmath}
\usepackage{amsthm}
\usepackage{graphicx}
\usepackage{subcaption}
\usepackage[noend,ruled,vlined]{algorithm2e}
\newcommand{\oracle}{\textsc{fi-cbb}\xspace}
\newcommand{\ucb}{\textsc{ucb-cbb}\xspace}
\usepackage{amssymb}
\usepackage{bm}
\usepackage{xspace}
\usepackage{xcolor}
\usepackage{mathtools}
\usepackage{comment}
\usepackage{float} 
\usepackage{thmtools}
\usepackage{thm-restate}
\usepackage{enumitem}
\newtheorem{theorem}{Theorem}

\newtheorem{fact}{Fact}

\newtheorem{definition}{Definition}
\usepackage{tabularx}

\DeclareMathOperator*{\A}{\mathcal{A}}
\DeclareMathOperator*{\C}{\mathcal{C}}

\newcommand\event[1]{\mathop{\mathbb{I}\left(#1\right)}}

\newcommand\Ex[2]{\mathop{\underset{#1}{\mathbb{E}}\left[#2\right]}}

\newcommand\Pro[1]{\mathop{\mathbb{P}\left(#1\right)}}

\DeclareMathOperator{\extr}{\mathcal{Z}}
\DeclareMathOperator{\extrsub}{\mathcal{Z}^S}
\DeclareMathOperator{\pit}{\tilde{\pi}\xspace}
\DeclareMathOperator{\Reg}{\textbf{Reg}\xspace}
\DeclareMathOperator{\Rew}{\textbf{Rew}\xspace}

\DeclareMathOperator*{\dummy}{B_1}
\DeclareMathOperator*{\dummyy}{B_2}

\newcommand*{\QED}{\hfill\ensuremath{\blacksquare}}%
\usepackage{float}
\usepackage{bookmark}

\begin{document}

\maketitle

\begin{abstract}

We study a novel variant of the multi-armed bandit problem, where at each time step, the player observes an independently sampled context that determines the arms' mean rewards. However, playing an arm blocks it (across all contexts) for a fixed and known number of future time steps. The above contextual setting, which captures important scenarios such as recommendation systems or ad placement with diverse users, invalidates greedy solution techniques that are effective for its non-contextual counterpart (Basu et al., NeurIPS19). Assuming knowledge of the context distribution and the mean reward of each arm-context pair, we cast the problem as an online bipartite matching problem, where the right-vertices (contexts) arrive stochastically and the left-vertices (arms) are blocked for a finite number of rounds each time they are matched. 
This problem has been recently studied in the full-information case, where competitive ratio bounds have been derived.
We focus on the bandit setting, where the reward distributions are initially unknown; we propose a UCB-based variant of the full-information algorithm that guarantees a $\mathcal{O}(\log T)$-regret w.r.t. an $\alpha$-optimal strategy in $T$ time steps, matching the $\Omega(\log(T))$ regret lower bound in this setting. Due to the time correlations caused by blocking, existing techniques for upper bounding regret fail. For proving our regret bounds, we introduce the novel concepts of delayed exploitation and opportunistic sub-sampling and combine them with ideas from combinatorial bandits and non-stationary Markov chains coupling.
\end{abstract}

\section{Introduction}
There has been much interest in variants of the stochastic {\em multi-armed bandit} (MAB) problem to model the phenomenon of {\em local} performance loss, where after each play, an arm either becomes unavailable {\color{black} for several subsequent rounds}~\cite{BSSS19}, or its mean reward {\color{black} temporarily} decreases~\cite{KI18, CCB19}. These studies provide state-of-the-art finite time regret guarantees. However, many (if not most) practical applications of bandit algorithms are contextual in nature (e.g., in recommendation systems, task allocations), and these studies do not capture such scenarios where the rewards depend on a task-dependent context.
Our paper focuses on a contextual variant of the blocking bandits model~\cite{BSSS19}.

We consider a set of {\em arms} such that, once an arm is pulled, it cannot be played again (i.e., is blocked) for a known and fixed number of consecutive rounds. At each round, a unique {\em context} is sampled according to some known distribution over a finite set of contexts and the player observes this context before playing an arm. The reward of each arm is drawn independently from a different distribution, depending on the context of the round under which the arm is played. The objective of the player is to maximize the expected cumulative reward collected within an unknown time horizon.

Applications of the above model include scheduling in data-centers, where the reward for assigning a task to a particular server depends on the workload of the task (e.g., computation, memory or storage intensive), or task assignment in online and physical service systems (e.g., Mechanical Turk for crowdsourcing tasks, ride sharing platforms for matching customers to vehicles). In these settings, both the contextual nature as well as  transient unavailability (e.g., a vehicle currently unavailable due to being occupied by a previous customer) of resources are central to the resource allocation tasks. 

\subsection{Key challenges} 
We introduce and study the problem of contextual blocking bandits (CBB). In this setting, greedy approaches that play the best available arm fail. Instead, for  adapting to unknown future contexts, a combination of randomized arm selection and selective round skipping (meaning, not play any arm in some rounds)  is required for achieving optimal competitive guarantees. This technique, that ensures sufficient future arm availability, has been noted in \cite{DSSX18} and \cite{CHMS10,AHL12}. 

Prior work in the full-information case where the mean rewards are known \cite{DSSX18}, devises a randomized LP rounding algorithm that is based on round skipping. Critically, the round skipping probabilities are time-dependent and computed offline given the LP solution (see Section \ref{sec:oracle}).
%for more on this
The skip probabilities however cannot be precomputed in a bandit setting, thus requiring some form of online learning.

To address the challenges of a bandit setting, a natural idea is to use a (dynamic) LP. This LP would use upper confidence bound (UCB) values (that vary over time) in place of the true mean values that would be available in the full information setting (as in \cite{AN14,SS18}). 
% The (sample dependent) solutions to the LP can then be used to compute the skip probabilities. 
This strategy, however, creates a significant technical hurdle: the LP is now a function of the trajectory, and the availability state of the system depends on the dynamically changing LP solution several steps into the future. This correlates past and future decisions and, thus, the techniques for analyzing the impact of skipping rounds to arm availability can no longer be applied.

The LP using UCB values has a further challenge: An action derived from the LP might not be available in a particular round (due to blocking); thus no action would be taken leading to no new sample of reward, and thus, no evolution of the information state (maintained by the bandit to learn the environment).

\subsection{Our contribution}
\begin{itemize}
\item We develop an efficient time-oblivious bandit algorithm that achieves $\mathcal{O}\left(\tfrac{k m (k+m) \log(T)}{\Delta}\right)$ regret bound, for $k$ arms, $m$ contexts and $T$ time steps, where $\Delta$ is the difference between the optimal and the best suboptimal extreme point solution of the LP. This requires two key technical innovations:
\begin{itemize}
\item {\em Delayed Exploitation}. At each time $t$, our algorithm uses the UCB from the (past) time $(t-M_t)$, where $M_t = \Theta(\log(t))$, for computing a new solution to the LP. Introducing this delay is crucial -- it ensures that the dynamics of the underlying Markov chain over the interval $[t-M_t, t]$ have mixed, and decorrelates the UCB from each arm's availability at time $t$. We believe that this technique might be of independent interest.
\item {\em LP Convergence under Blocking}. We leverage techniques from combinatorial bandits~\citep{CWYW16, WC17} and combine them with an {\em opportunistic subsampling} scheme, in order to ensure a sufficient rate of new samples associated with suboptimal LP solutions.
\end{itemize}
\item For the full-information case, we prove an unconditional hardness of $\frac{d_{\max}}{2d_{\max}-1}$, where $d_{\max}$ is the maximum blocking time, establishing that our algorithm (and the one in \cite{DSSX18}) achieves the optimal competitive guarantee. This improves on the $0.823$-hardness result of \cite{DSSX18}.

\item As a byproduct of our work, we improve on \cite{DSSX18}, in the special case where the blocking times are deterministic and time-independent. Specifically, our algorithm (a) does not require knowledge of $T$, (b) involves a (smaller) LP that can be optimized via fast combinatorial methods and (c) leads to a slightly improved competitive guarantee (asymptotically) for finite blocking times.

\item Although our work focuses on the theoretical aspects of the problem, we provide simulations of our algorithm on synthetic instances in Section \ref{simulations}.
\end{itemize}

\subsection{Related work} 
From the advent of stochastic MAB~\citep{T33} and later \citep{LR85}, decades of research in stochastic MAB have culminated in a rich body of research (c.f. \citep{BCB12, LS18}). Focusing on directions which are relevant to our work, we first note that our problem differs from {\em contextual bandits} as in~\citep{LZ08, BLLR11, AHKL14}. Although, these works face the challenge of arbitrarily many contexts, they do not handle blocking.

Our problem lies in the space of stochastic {\em non-stationary} bandits, where the reward distributions (states) of the arms can change over time. Two important threads in this area are: {\em rested bandits}~\citep{Gittins79, TL12, CDKMY17}, where the arm state (hence, reward distribution) changes only when the arm is played, and {\em restless bandits}~\citep{Whittle88, TL12}, where the state changes at each time, independently of when the arm is pulled. Our problem differs from these settings (and from {\em sleeping bandits} \cite{KNMS10}), as our reward distributions change in a very special manner, both during arm playing (becoming blocked) and not playing (i.i.d. context and becoming available). Our problem also falls into the class of {\em controlled Markov Decision Processes}~\citep{Alt99} with unknown parameters. However, the exponentially large state space ($\mathcal{O}(d_{\max}^k)$), makes this approach highly space and time consuming, and the finite time regret of known algorithms~\citep{AO07, TB08,GOA19} non-admissible.

In recent works~\citep{KI18, BSSS19, CCB19, PBG19}, the reward distribution changes are determined by some fixed special functions. Our setting belongs to this line of work, as blocking can be translated w.l.o.g.\ into deterministically zero reward. However, our problem differs from the above, as the optimal algorithm in hindsight must adapt to random context realizations. The models in~\cite{GKSS07,PBG19} also assume stochastic side information and arm delays, but consider different notions of regret, comparing to our work. 

From an algorithmic side, the full-information case of our problem has been studied in~\cite{DSSX18}, in the context of {\em online bipartite matching} with stochastic arrivals and reusable nodes (see also~\cite{JKK17} for an interesting, yet unrelated to ours, combination of matching and learning). In addition, the non-contextual case of our problem~\cite{BSSS19} is related to the literature on {\em periodic scheduling} \cite{HMRTV89,BNBNS02,SST09}. 

The idea of combining UCB~\cite{ACBF02} and LP formulations also appears in {\em bandits with knapsacks} \cite{BKS18,SS18,AN14,ADL16}. Our problem differs from this model (and from {\em bandits with budgets}~\cite{Sliv13,CJS15}), as we assume both resource consumption and budget renewal (i.e., arm availability) that depend on the player's actions. Due to blocking, our problem differs from {\em combinatorial bandits} and {\em semi-bandits}~\citep{CTPL15,CWY13, CWYW16, KWAEE14, KWAS15}. However, we draw from the techniques in \citep{WC17} for analyzing the regret of our LP-based algorithm.
\section{Problem definition} \label{sec:definition}

\paragraph{Model.} 
Let $\A$ be a set of $k$ {\em arms} (or {\em actions}), $\C$ be a set of $m$ {\em contexts} and $T \in \mathbb{N}_+$ be the time horizon of our problem. At every round $t \in \{1, 2, \dots, T\}$, a context $j \in \C$ is sampled by {\em nature} with probability $f_j$ (such that $\sum_{j \in \C} f_j = 1$). The {\em player} has prior knowledge of context distribution $\{f_j\}_{j \in \C}$, while she observes the realization of each context at the beginning of the corresponding round, before making any decision on the next action. When arm $i \in \A$ is pulled at round $t$ under context $j \in \C$, the player receives a {\em reward} $X_{i,j,t}, \forall t \in \{1,2, \dots, T\}$. We assume that the (context and arm dependent) rewards $\{X_{i,j,t}\}_{t \in [T]}$ are i.i.d.\ random variables with mean $\mu_{i,j}$ and bounded support in $[0,1]$. In the {\em blocking} bandits setting, each arm is in addition associated with a {\em delay} $d_i \in \mathbb{N}$, indicating the fact that, once the arm $i$ is played on some round $t$, the arm becomes unavailable for the next $d_i-1$ rounds (in addition to round $t$), namely, in the interval $\{t, \dots, t+d_i-1\}$. The player is unaware of the time horizon, but has a priori knowledge of the arm delays. A specific problem instance $I$ is defined by the tuple $(\A, \C, \{d_i\}_{\forall i \in \A}, \{f_j\}_{\forall j \in \C}, \{X_{i,j,t}\}_{\forall i,j,t})$, with each element as defined above. 

We refer the reader to Appendix \ref{appendix:notation} for additional technical notation.

\paragraph{Online algorithms.}
In our setting, an {\em online algorithm} is a strategy according to which, at every round $t$, the player observes the context of the round, and chooses to play one of the available arms (or skip the round). Specifically, the decisions of an online algorithm depend only on the observed context of each round and the availability state of the system.
We are interested in constructing an online algorithm $\pi$, that maximizes the {\em expected cumulative reward} over the randomness of the nature and of the algorithm itself, in the case of a randomized algorithm. Let $A^{\pi}_t \in \A \cup~\emptyset$ be the arm played by algorithm $\pi$ at time $t$, $C_t$ is the context of the round and $\mathcal{R}_{N,\pi}$ is the randomness due to the contexts/rewards realization and the possible random bits of $\pi$.  For any instance $I$ and time horizon $T$, the expected reward can be expressed as follows:
\begin{align*}
\Rew^{\pi}_I(T) = \Ex{\mathcal{R}_{N,\pi}}{\sum_{t \in [T]} \sum_{j \in \C} \sum_{i \in \A} X_{i,j,t} \event{A^{\pi}_t = i, C_t = j}}. \end{align*}

\paragraph{Oracle.}
In order to characterize an optimal online algorithm, one way is to formulate it as a Markov Decision Process (MDP) on a state space of size $\mathcal{O}(d_{\max}^{k})$, which is exponential in the number of arms. Instead, we take a different route by comparing our algorithms with an an offline {\em oracle}, an optimal (offline) algorithm that has a priori knowledge of the realizations of the contexts of all rounds and infinite computational power (a.k.a. {\em optimal clairvoyant algorithm}). It is clear that the expected reward of the {\em oracle}, denoted by $\Rew^{*}_I(T)$, upper bounds the reward of any online algorithm. 

\paragraph{Competitive ratio.}
The {\em competitive ratio}, $\rho^{\pi}(T)$, of an algorithm $\pi$ for $T$ time steps is defined as the (worst case over the problem instance) ratio between the expected reward collected by $\pi$ and the expected reward of the oracle, and is a standard notion in the field of online algorithms \footnote{Formally, the competitive ratio is defined as
$\rho^{\pi}(T) = \inf_{I} \frac{\Rew^{\pi}_I(T)}{\Rew^{*}_I(T)}$.}. 
An algorithm $\pi$ is called $\alpha$-{\em competitive} if there exists some $\alpha \in (0,1]$ such that $\rho^{\pi}(T)\geq \alpha, \forall T \in \mathbb{N}_{+}$. Thus, an $\alpha$-competitive algorithm achieves at least $\alpha \cdot \Rew^{*}_I(T)$ expected reward.

\paragraph{Approximate regret.}
Let $\pi^*$ be the oracle. Note that, for any finite $T$, and due to the finiteness of the number of contexts and actions, such an algorithm is well-defined. The {\em $\alpha$-regret} of an algorithm $\pi$ is the difference between $\alpha$ times the expected reward of an optimal online policy\footnote{In fact, we use a stronger notion of $\alpha$-regret by assuming that the optimal online algorithm is clairvoyant.} and the reward collected by $\pi$, for $\alpha \in (0,1]$, namely,
\begin{align*}
    \alpha\Reg^{\pi}_I(T) = \alpha \cdot \Rew_I^*(T) - \Rew_I^{\pi}(T).    
\end{align*}
The notion of $\alpha$-regret is widely accepted in the combinatorial bandits literature~\citep{CWYW16,WC17}, for problems where an efficient 
algorithm does not exist, even for the case where the mean rewards $\{\mu_{i,j}\}_{\forall i,j}$ are known a priori, thus leading inevitably to linear regret in the standard definition.
\section{The full-information problem} \label{sec:oracle}
We begin by considering the {\em full-information} (non-bandit) variant of the problem, where the mean rewards $\{\mu_{i,j}\}_{i \in \A,j \in \C}$ are known to the player a priori. Note that in both variants, the distribution of contexts $\{f_j\}_{j \in \C}$ and the delays $\{d_i\}_{i \in \A}$ are known to the player, but the time horizon is unknown. This case of our problem has been also studied in \cite{DSSX18}, in the setting where the delays can be stochastic and time-dependent, but the time horizon is known. 

\paragraph{LP upper bound.} 
Our first step towards proving the competitive ratio of our algorithm is to upper bound the reward of an optimal clairvoyant policy, $\Rew^{*}_I(T)$, which uses an optimal schedule of arms for each context realization. Consider the following linear program:

\begin{figure}[H]
\vspace*{-0.5em}
\begin{minipage}[t]{1.0\textwidth}
\begin{equation}
\textbf{maximize: } \sum_{i \in \A} \sum_{j \in \C} \mu_{i,j} z_{i,j} ~~~\textbf{ s.t. } \tag{\textbf{LP}} \label{lp:LP}
\end{equation}
\end{minipage}
\vspace*{-1.5em}

\begin{minipage}[t]{0.33\textwidth}
\begin{align}
    \sum_{j \in \C} z_{i,j} \leq \frac{1}{d_i}, \forall i \in \A \tag{\textbf{C1}} \label{flp:window}
\end{align}
\end{minipage}
\begin{minipage}[t]{0.33\textwidth}
\begin{align}
    \sum_{i \in \A} z_{i,j} \leq f_j, \forall j \in \C \tag{\textbf{C2}} \label{flp:conditional}
\end{align}
\end{minipage}
\begin{minipage}[t]{0.33\textwidth}
\begin{align}
    z_{i,j} \geq 0, \forall i \in \A, \forall j \in \C. \notag 
\end{align}
\end{minipage}
\end{figure}

In \eqref{lp:LP}, each variable $z_{i,j}$ can be thought of as the (fluidized) average rate of playing arm $i$ under context $j$. Intuitively, constraints \eqref{flp:window} indicate the fact that each arm $i \in \A$ can be pulled at most once every $d_i$ steps, due to the blocking constraints. Similarly, constraints \eqref{flp:conditional} suggest that playing (any arm) under context $j \in \C$ happens with probability at most $f_j$. As we show in the proof of Theorem \ref{online:theorem:competitive}, \eqref{lp:LP} provides an (approximate) upper bound to the expected reward collected by an optimal clairvoyant policy (when we multiply its objective value by $T$), and this approximation becomes tighter as $T$ increases. Finally, we remark that, as opposed to the LP used in \cite{DSSX18}: (a) We do not require knowledge of the time horizon $T$, in order to compute an optimal solution to \eqref{lp:LP}, and (b) its structural simplicity allows the efficient computation of an optimal extreme point solution, using fast combinatorial methods (see Appendix \ref{appendix:discuss:lps}).

\paragraph{Online randomized rounding.} Our algorithm, \oracle, rounds an optimal solution to \eqref{lp:LP} in an online randomized manner (as in \cite{DSSX18}, but for a different LP), and serves as a basis for the bandit algorithm we design in the next section (see Appendix \ref{appendix:online:pseudo} for a pseudocode):

\oracle : The algorithm initially computes an optimal solution, $\{z^*_{i,j}\}_{i,j}$, to \eqref{lp:LP}. At any round $t$, and after observing the context $j_t \in \C$ of the round, the algorithm {\em samples} an arm, based on the marginal distribution $\{z^*_{i,j_t}/{f_{j_t}}\}_{i\in \A}$. At this phase, any arm can be sampled, independently of its availability state. If no arm is sampled (because $\sum_{i \in \A} z^*_{i,j_t}/{f_{j_t}} < 1$), the round is skipped and no arm is played. Let $i_t \in \A$ be the sampled arm of this phase. If the arm $i_t$ is available, the algorithm plays the arm with probability $\beta_{i_t,t}$ (formally defined shortly)-- otherwise, the round is skipped.

For any arm $i\in \A$ and round $t$, we set $\beta_{i,t} = \min \{1, \tfrac{d_i}{2d_i -1}\tfrac{1}{q_{i,t}}\}$, where $q_{i,t}$ is the a priori probability of $i$ being available at time $t$ (i.e., before observing any context realization). The value of $q_{i,t}$, can be recursively computed as follows: 
\begin{align}
q_{i,1} = 1 \text{ and }q_{i,t+1} = q_{i,t}(1 - \beta_{i,t} \sum_{j \in \C} z^*_{i,j}) + \event{t\geq d_i} q_{i,t-d_i+1} \beta_{i,t-d_i+1}\sum_{j \in \C} z^*_{i,j}. \label{online:exante:formula} 
\end{align}
In the above algorithm, the arm sampling at the beginning of each round, ensures that, on average, each arm-context pair, $(i,j)$, is selected a $z^*_{i,j}$-fraction of time. Moreover, $\{\beta_{i, t}\}_{\forall i,t}$ correspond to the {\em non-skipping} probabilities-- their role is to ensure a constant rate of arm availability over time. The technique of precomputing these probabilities as a function of the expected arm availability has been proven useful for achieving optimal competitive guarantees in various online optimization settings (see, e.g., \cite{DSSX18,CHMS10,AHL12}), where other approaches (such as greedy LP rounding) fail. 

In the following theorem, we provide the competitive guarantee of our algorithm \oracle . Due to space constraints and the partial overlapping with \cite{DSSX18}, its proof has been moved to Appendix \ref{appendix:online}.

\begin{restatable}{theorem}{restateTheoremOracle}\label{online:theorem:competitive}
For any $T$, the competitive ratio of \oracle against any optimal clairvoyant algorithm is at least $\frac{d_{\max}}{2d_{\max} - 1} \left(1 - \frac{d_{\max}-1}{d_{\max}-1 + T}\right)$, where $d_{\max} = \max_{i \in \A} d_i$.
\end{restatable}

%\paragraph{Proof sketch of Lemma~\ref{lemma:overview:upperbound}.}

%\subsection{Online Non-Blocking Rounding: Oracle-CBB}
%We propose the online algorithm \oracle (see, Algo.~\ref{alg:oracle} in the Appendix). The execution of \oracle can be separated into two main phases: the {\em initialization} phase, and the {\em online} phase.

%This completes the description of the algorithm.

%\subsection{Analysis of Oracle-CBB} 

%The combination of Lemmas \ref{lemma:overview:upperbound} and \ref{lemma:oracle:exact} leads to the proof of Theorem \ref{online:theorem:competitive}. All the proofs of this section are deferred to Appendix~\ref{appendix:online} due to lack of space.
\section{The bandit problem} \label{section:ucbregret}
In the bandit setting of our problem, where the mean rewards $\{\mu_{i,j}\}_{\forall i,j}$ are initially unknown, we design a bandit variant of \oracle, that attempts to learn the mean values of the distributions $\{X_{i,j,t}\}_{\forall t}$ for all $i \in \A, j \in \C$, while collecting the maximum possible reward. Our objective is to achieve an $\alpha$-regret bound growing as $\mathcal{O}(\log(T))$, for $\alpha = \frac{d_{\max}}{2 d_{\max}-1}$. Due to space constraints, the proofs of this section are deferred to Appendix~\ref{appendix:omitted}.

\subsection{The bandit algorithm: \ucb}

Our algorithm, named \ucb, maintains UCB indices for all arm-context pairs, and uses them (in place of the actual means) to compute a new optimal solution to \eqref{lp:LP} at each round. Given this solution, the algorithm samples an arm in a similar way as \oracle. We expect that, as the time progresses, the LP solution computed using the UCB estimates will converge to the optimal solution of \eqref{lp:LP} and, thus, the two algorithms will gradually operate in an similar manner.

However, as the UCB indices are intrinsically linked with arm sampling, the future arm availability and, thus, the sequence of LP solutions become correlated across time. This makes the precomputation of non-skipping probabilities, $\{\beta_{i,j}\}_{i,j}$, as before, no longer possible. In order to disentangle these dependencies, we introduce the novel technique of {\em delayed exploitation}, where at each round, \ucb uses UCB estimates from relatively far in the past. 
This ensures that the extreme points used in 
the meantime, 
%this interval (from this past round up to the current round)
are fixed and unaffected by the online rounding and reward realizations in the entire duration. Using this fixed sequence of extreme points, we {\em adaptively} compute non-skipping probabilities that strike the right balance between skipping and availability.

We now outline the new elements of \ucb (which we denote by $\pit$), comparing to \oracle.

\paragraph{Dynamic LP.} As opposed to the case of \oracle, where the mean rewards are initially unknown, our bandit algorithm solves at each time $t \in [T]$ a linear program $\eqref{lp:LP}(t)$. This LP has the same constraints as \eqref{lp:LP}, but uses UCB estimations, $\{\bar{\mu}_{i,j}(t)\}_{i,j}$, in place of the actual means in the objective. Following the standard UCB paradigm, the above estimations are defined as
\begin{align}\label{eq:regret:ucbupdate}
    \bar{\mu}_{i,j}(t) = \min \left\{ \hat{\mu}_{i,j,T_{i,j}(t)} + \sqrt{\frac{3 \ln{(t)}}{2 T_{i,j}(t)}}, 1\right\}, \forall i\in \A, j \in \C.
\end{align}
In the above formula, $T_{i,j}(t)$ denotes the number of times arm $i$ is played under context $j$ up to (and excluding) time $t$, and $\hat{\mu}_{i,j,T_{i,j}(t)}$ denotes the empirical estimate of $\mu_{i,j}$, using $T_{i,j}(t)$ i.i.d. samples.  

\paragraph{Delayed Exploitation.}
In order to decouple the UCB estimates and, thus, the extreme point choices, from the arm availability state of the system, our algorithm, at any round $t$, uses the UCB indices from several rounds in the past.
For any $t\in [T]$, let $Z(t) = \{z_{i,j}(t)\}_{i,j}$ be optimal extreme point solution to $\eqref{lp:LP}(t)$, i.e., using the indices $\{\bar{\mu}_{i,j}(t)\}_{i,j}$ in place of the actual mean rewards. Moreover, let $Z(0)$ be an arbitrary extreme point of \eqref{lp:LP}. For any $t \in [T]$, we fix $M_t = \Theta(\log t)$, in a way that there is a unique integer $T_c \geq 1$, such that $t \geq M_t + 1$ if and only if $t \geq T_c$ (see Appendix \ref{appendix:criticaltime}).

At any $t \in [T]$, and after observing the context $j_t \in \C$ of the round, our algorithm samples arms according to the marginal distribution $\{z_{i,j_t}(t - M_t)/f_{j_t}\}_{i \in \A}$, namely, using the solution of $\eqref{lp:LP}(t-M_t)$. In the case where $t - M_t \leq 0$, the algorithm samples arms according to the marginal distribution  $\{z_{i,j_t}(0)/f_{j_t}\}_{i \in \A}$, based on the initial extreme point $Z(0)$.

\paragraph{Conditional Skipping.}
In \ucb the non-skipping probabilities of each round $t \in [T]$, $\{\beta_{i,t}\}_{\forall i}$, now depend on the sequence of solutions of \eqref{lp:LP} up to time $t$, that are used for sampling arms. We define by $H_t$ the history up to time $t$ for any $t\geq 1$, which includes all the context realizations, pulling of arms, and reward realizations of played arms. For every arm $i \in \A$ and time $t$, the non-skipping probability is defined as $\beta_{i,t} = \min \{1, \tfrac{d_i}{2d_i -1}\tfrac{1}{q_{i,t}(H_{t-M_t})}\}$, where $q_{i,t}(H_{t-M_t})$ now corresponds to the probability of $i$ being available at time $t$, conditioned on the history up to time $H_{t-M_t}$.

For $t < T_c$, where the extreme point $Z(0)$ is used at every round until $t$, the probability $q_{i,t}(H_{0})$, for any $i \in \A$ can be recursively computed similarly as in the full-information case (using the recursive equation \eqref{online:exante:formula}, where every $z^*_{i,j}$ is replaced with $z_{i,j}(0)$ for any $i \in \A,j\in \C$).

For $t \geq T_c$, the value $q_{i,t}(H_{t-M_t})$ is the probability that arm $i$ is available at time $t$, conditioned on $H_{t-M_t}$. By definition of $T_c$, for any $\tau \in [t-M_{t}, t]$, it is the case that $\tau - M_{\tau} \leq t - M_t$ and, thus, $H_{\tau-M_{\tau}} \subseteq H_{t-M_{t}}$. This implies that all the extreme points in the trajectory of $\eqref{lp:LP}(\tau - M_{\tau})$ for $\tau \in [t-M_t,t]$, as well as the involved non-skipping probabilities $\{\beta_{i,\tau}\}_{i \in \A}$ are deterministic and, thus, computable, conditioned on $H_{t-M_t}$. The computation of $q_{i,t}(H_{t-M_t})$ can be done recursively similarly to \eqref{online:exante:formula}. However, the extreme point solutions depend on arm mean estimates that vary over time, thus requiring a more involved recursion (see Appendix \ref{appendix:regret:compexante} for more details). Our choice of $M_t = \Theta(\log t)$ is large enough to guarantee sufficient decorrelation of the extreme point choices and the future arm availability, but also small enough to incur a small additive loss in the regret bound.

The above changes are summarized in Algorithm \ref{alg:ucb}. In Appendix \ref{appendix:ucb:exantepseudo}, we provide a routine, called \textsc{compq}(i,t,H), for the computation of $q_{i,t}(H_{t-M_t})$.

\begin{algorithm}[h]
\SetAlgoLined
\DontPrintSemicolon
Set $\bar\mu_{i,j}(0) \mathtt{\gets} 1$ for all $i \in \A, j \in \C$ and compute an initial solution $Z(0)$ to \eqref{lp:LP}.\;
 \For{$t = 1, 2, \dots$}{
  Set $M \gets \lfloor2\log_{c_1}(t)\rfloor  + 2\cdot d_{\max} + 8$, where $c_1= e^2/(e^2-1)$.\;
  \textbf{if} $t \leq M$ \textbf{then} Set $M = t$.\;
  Compute solution $Z(t-M) = \{z_{i,j}\}_{i\in\A, j \in \C}$ to $\eqref{lp:LP}(t-M)$.\;
  Observe context $j_t \in \C$ and sample arm $i_t \in \A$ with probability $z_{i_t,j_t}/f_{j_t}$.\;
  \If{$i_t \neq \emptyset$ \textbf{and} $i_t$ is available}{
   Set $q_{i_t,t}(H_{t-M}) \mathtt{\gets} \textsc{compq}(i_t,t, H_{t-M})$ and $\beta_{i_t,t} \mathtt{\gets} \min\{1 , \frac{d_i}{2d_i-1} \frac{1}{q_{i_t,t}(H_{t-M})}\}$.\;
   \textbf{if} $u \leq \beta_{i_t,t}$, for $u \sim U[0,1]$ \textbf{then} Play $i_t$.\;}
Update the UCB indices according to Eq. \eqref{eq:regret:ucbupdate}.\;
   
 }
 \caption{\ucb}
\label{alg:ucb}
\end{algorithm}

\subsection{Analysis of the $\alpha$-regret}
We define the family of extreme point solutions $Z = \{z^{Z}_{i,j}\}_{i,j}$ of \eqref{lp:LP} as $\extr$. We note that, as $\eqref{lp:LP}(t)$ varies from \eqref{lp:LP} only in the objective, the {\em family} of extreme points remains fixed and known to the player. We denote by $Z^* = \{z^*_{i,j}\}_{\forall i,j}$ any optimal extreme point of \eqref{lp:LP} with respect to the mean values $\{\mu_{i,t}\}_{\forall i,t}$, and we denote by $\extrsub$ the set of suboptimal extreme points. We now define the relevant gaps of our problem, by specializing the corresponding definitions of~\cite{WC17}, in the case where the family of feasible solutions coincides with the extreme points solutions of \eqref{lp:LP}. As we discuss in Appendix \ref{appendix:discuss:gaps}, the following suboptimality gaps are complex functions of the means, $\{\mu_{i,j}\}_{i\in \A,j \in \C}$, arm delays, $\{d_i\}_{i \in \A}$, and context distribution, $\{f_j\}_{j \in \C}$.

\begin{definition}[Gaps~\cite{WC17}]
For any extreme point $Z \in \extrsub$ the suboptimality gap is $\Delta_Z = \sum_{i \in \A} \sum_{j \in \C} \mu_{i,j} (z^*_{i,j} - z^Z_{i,j})$ and $\Delta_{\max} = \sup_{Z \in \extr} \Delta_Z$. 
For any arm-context pair $(i,j)$, we define $\Delta^{i,j}_{\min} = \inf_{Z \in \extrsub, z^Z_{i,j}>0} \Delta_Z$, i.e., the minimum $\Delta_Z$ over all $Z \in \extrsub$ such that $z^Z_{i,j}>0$.
\end{definition}

The first step of our analysis is to show that delayed exploitation, indeed, ensures that the dynamics of the underlying Markov Chain (MC) over the interval $[t-M_t, t]$ have mixed. This weakens the dependence between online rounding and extreme point choices and, thus, decorrelates the UCB from arm availability at time $t$. Let $F^{\pit}_{i,t}$ be the event that arm $i$ is available at time $t$. Using techniques from \emph{non-homogeneous MC coupling}, we prove the above weakening formally in the following lemma. 

\begin{restatable}{lemma}{restateLemmaMixing}\label{lemma:regret:mixing}
For any arm $i \in \A$ and rounds $t, t' \in [T]$ such that $0 < t - t' < d_i$ and $t\geq T_c$, we have:
\begin{align*}
  \frac{\Pro{F^{\pit}_{i,t'} | H_{t-M_t}}}{\Pro{F^{\pit}_{i,t'} | H_{t'-M_{t'}}}} \leq 1 + c_0 \cdot c_1^{-M_t}, \text{  for  }c_0 = e \left(\frac{e^2}{e^2-1}\right)^{2 d_{max}} \text{ and  } c_1 = \frac{e^2}{e^2-1}.
\end{align*}
\end{restatable}
\paragraph{Proof sketch.}
The key idea of the proof is to link the quantities $\Pro{F^{\pit}_{i,t'} | H_{t-M_t}}$ and $\Pro{F^{\pit}_{i,t'} | H_{t'-M_{t'}}}$ to the evolution of a fast-mixing non-homogeneous MC. Let us fix an arbitrary run of the \ucb algorithm upto time $t - M_t$, which fixes the sequence of extreme points in the run as $z_{ij}(\tau - M_{\tau})$, and the skipping probabilities as  $\beta_{\tau}$ for $1 \leq \tau \leq t$ (see  Appendix~\ref{appendix:regret:compexante} for details). For this run and any fixed arm $i$, we construct the non-homogeneous MC with state space $\{0, 1, \dots, d_i-1\}$, where each state represents the number of remaining rounds until the arm becomes available. At time $\tau\geq 1$, the MC transitions from state $0$ to state $(d_i-1)$ w.p. $\beta_{i,\tau}\sum_{j}z_{i,j}(\tau-M_{\tau})$, and from state $d>0$ to state $(d-1)$ w.p. $1$. 
Let $\nu(\tau)$, be the first time on or after $\tau$ when arm $i$ becomes available.  We show that $\Pro{F^{\pit}_{i,t'} | H_{s-M_s}}$ denotes the  probability that an independent copy of the above MC which starts from state $0$ (available) at time $\nu(s- M_s)$, namely $\mathcal{X}_s$, is available at time $t'$. As the two independent MCs $\{\mathcal{X}_s, s=t,t'\}$ evolve on the same non-homogeneous MC, using ideas from coupling we show at time $t'$ the $L1$ distance between their distributions decays exponentially with $M_{t}$. Specifically, we construct a Doeblin coupling~\cite{L02} of the two MCs, where at each time  $\tau \geq (t-M_t+d_i)$ w.p. at least $1/e^2$, the two MC meet at state $0$, thus are coupling exponentially fast. \QED

As we show below, Lemma \ref{lemma:regret:mixing} allows us to relate $\alpha\Reg^{\pit}_I(T)$ to the suboptimality gaps of the sequence of LP solutions used by \ucb. This comes with an additive $\Theta(\log(T) \Delta_{\max})$ cost in the regret.

\begin{restatable}{lemma}{restateLemmaMapping}\label{lemma:regret:mapping}
For the $\alpha$-regret of \ucb, for $\alpha = \frac{d_{\max}}{2d_{\max} - 1}$ and $M = \Theta(\log T + d_{\max})$, we have
\begin{align*}
\alpha\Reg^{\pit}_I(T) \leq \Ex{\mathcal{R}_{N,\pit}}{\sum^{T-M}_{t = 1} \sum_{i \in \A} \sum_{j \in \C} \mu_{i,j} \left(z^*_{i,j} - z_{i,j}(t)\right)} +\frac{1}{3}\ln(T)\Delta_{\max} + 6 d_{\max} + 71.
\end{align*}
\end{restatable}
\paragraph{Proof sketch.}
Starting from the definition of $\alpha\Reg^{\pit}_I(T)$:  We upper bound $\alpha\cdot\Rew^{*}_I(T)$ using Theorem \ref{online:theorem:competitive}, while we incur regret in four distinct ways. 
(a) We incorporate the $\left(1- \Theta(T^{-1})\right)$-multiplicative loss as a $\Theta(d_{\max})$ additive term in the regret.
(b) We upper bound the total regret during time $1$ to $T_c$ by $(\max_{ij}\mu_{ij}) T_c = \Theta(d_{\max})$.
(c) We separate the rounds $t\geq T_c$, when $M_t$ is increased (and, thus, the same UCB values are used more than once). This happens $\Theta(\log(T))$ times, adding another $\Theta(\log T \Delta_{\max})$ term to the regret.
(d) For the rest of the ``synchronized'' rounds $t \geq T_c$ (i.e., where each one uses strictly updated UCB estimates), using Lemma \ref{lemma:regret:mixing}, we show that $i\in \A$ is played under $j \in \C$ with probability ``close'' to $\frac{d_i}{2d_i-1} z_{i,j}(t-M_t)$, where the total approximation loss leads to an additive $\Theta(1)$ term in the regret.
\QED

By Lemma \ref{lemma:regret:mapping}, we can see that \ucb accumulates only constant regret in expectation, once all the extreme points of $\extrsub$ are eliminated with high probability. For this to happen, we need enough samples from each of the arm-context pairs in the support of any $Z \in \extrsub$ (i.e., ${\rm supp}(Z) = \{(i,j)~| z^Z_{i,j}>0\}$). Once the algorithm computes a point $Z \in \extrsub$ (as a solution of $\eqref{lp:LP}(t)$), each pair $(i,j) \in {\rm supp}(Z)$ is played with probability $z^Z_{i,j} > 0$, assuming there is no blocking or skipping. Leveraging this observation, we draw from the techniques in combinatorial bandits with {\em probabilistically triggered arms} \cite{CWYW16,WC17}\footnote{The papers \cite{CWYW16, WC17} capture a much more general setting, which we omit for brevity.}. In this direction, following the paradigm of \cite{WC17}, we define the following subfamily of extreme points called {\em triggering probability} (TP) groups:
\begin{definition}[TP groups\,\cite{WC17}]
For any pair $(i,j) \in \A \times \C$ and integer $l\geq 1$, we define the TP group $\extr_{i,j,l} = \{Z \in \extr~|~2^{-l} < z^Z_{i,j} \leq 2^{-l+1}\}$, where $\{\extr_{i,j,l}\}_{l \geq 1}$ forms a partition of $\{Z \mathtt{\in} \extr~|~ z^Z_{i,j} > 0\}$.
\end{definition}

The regret analysis relies on the following counting argument (known in literature as {suboptimality charging}) -- now standard in the combinatorial bandits literature \cite{KWAS15,CWYW16, WC17}: For each TP group $\extr_{i,j,l}$, we associate a counter $N_{i,j,l}$. The counters are all initialized to $0$ and are updated as follows: At every round $t$, where the algorithm computes the extreme point solution $Z(t)$, we increase by one every counter $N_{i,j,l}$, such that $Z(t) \in \extr_{i,j,l}$. We denote by $N_{i,j,l}(t)$ the value of the counter at the beginning of round $t$. 

\paragraph{Opportunistic subsampling.} In the absence of blocking, it can be shown~\citep{WC17} that at any time $t$ and TP group $\extr_{i,j,l}$, we have $ T_{i,j}(t) \geq  \tfrac{1}{3} 2^{-l} N_{i,j, l}(t)$ with probability $1- O(1/t^3)$. This guarantees that by sampling arm-context pairs frequently enough, the algorithm learns to avoid all the points in $\extrsub$ with high probability. However, no such conclusion can be drawn in our situation, where arm blocking can potentially preclude information gain. Specifically, the naive approach of subsampling the counter increases every $d_i$ rounds, can only guarantee that $ T_{i,j}(t) \geq O(\tfrac{2^{-l}}{d_i} N_{i,j, l}(t))$ with high probability, thus, leading to a $\Theta(\sqrt{d_{\max}})$ multiplicative loss in the regret. We address the above issue via a novel {\em opportunistic subsampling} scheme, which guarantees that, even in the presence of strong local temporal correlations, we still obtain a constant fraction (independent of $d_i$) of independent samples with high probability.

\begin{restatable}{lemma}{restateLemmaSubsampling}\label{lemma:regret:counter-to-samples}
For any time $t\in [T]$, TP group $\extr_{i,j,l}$ and $ \mathcal{O}(2^l\log(t)) \leq s \leq t-1$, we have:
$$\Pro{N_{i,j, l}(t) = s, T_{i,j}(t) \leq  \tfrac{1}{24 e} 2^{-l} N_{i,j, l}(t)}\leq \tfrac{1}{t^3}.$$
\end{restatable}
\paragraph{Proof sketch.}
Due to blocking, there is no uniform lower bound for playing a pair $(i,j)$ each time $N_{i,j,l}(t)$ is increased. Therefore, we subsample the increases of $N_{i,j, l}(t)$ in a way that: (a) the subsampled instances of increases are at least $d_i$ rounds apart, and (b) the subsampled sequence captures a constant fraction (independent of $d_i$) of non-skipped rounds of the original sequence. The two properties ensure that, in the subsampled sequence, the number of times a pair $(i,j)$ is played concentrates around its mean. For a TP group $\extr_{i,j,l}$, we consider blocks of $(2d_i -1)$ contiguous counter increases. From each block we obtain one sample in the first $d_i$ counter increases, opportunistically picking a non-skipped round if there is one. By construction, the samples remain $d_i$ rounds apart, ensuring property (a). Also, we show there is at least one non-skipped round per block with probability at least $\tfrac{(2d_i-1)}{8e}2^{-l}$, ensuring  property (b). \QED

As we observe, the small size of \eqref{lp:LP} implies that all its extreme points are {\em sparse}. This makes it less sensitive to the error in the estimates; which, in turn, leads to tighter regret bounds (see Theorem \ref{theorem:regret:bound}).
\begin{restatable}{lemma}{restateRegretSparse}\label{lemma:regret:sparse}
For any  $Z \in \extr$,  $|\text{supp}(Z)| \leq k + m$. 
\end{restatable}

By combining Lemmas~\ref{lemma:regret:mapping},~\ref{lemma:regret:counter-to-samples} and~\ref{lemma:regret:sparse}, along with suboptimality charging arguments of~\citep{WC17} (as described above), we provide our final regret upper bound in the following theorem.

\begin{restatable}{theorem}{restateTheoremFinalRegret}\label{theorem:regret:bound}
The $\alpha$-regret of \ucb for $\alpha = \frac{d_{\max}}{2d_{\max}-1}$, can be upper bounded as
\begin{align*}
\hspace{-1.5em}\alpha\Reg_I^{\pit}(T)\leq \sum_{i \in \A}\sum_{j \in \C} \frac{C\left(k+m\right) \log{(T)}}{\Delta^{i,j}_{\min}} + \frac{\pi^2}{6}\sum_{i\in\A}\sum_{j\in \C} \log{\left(\frac{2\left(k+m\right)}{\Delta^{i,j}_{\min}}\right)} \Delta_{\max} + 6\cdot d_{\max},
\end{align*}
where $C>0$ is some universal constant.
\end{restatable}

\section{Hardness of the online problem}\label{sec:hardness}

The NP-hardness of the full-information CBB problem follows by~\cite{SST09,BSSS19}, even in the non-contextual (offline) setting~\cite{BSSS19}. In the following theorem, we provide unconditional hardness for the contextual case of our problem (see Appendix~\ref{appendix:hardness} for the proof). This result implies that the competitive guarantee of \oracle is (asymptotically) optimal, even for the single arm case. Moreover, since the construction in our proof involves deterministic rewards, the theorem also implies the optimality of the algorithm in~\cite{DSSX18}, thus, improving on the $0.823$-hardness presented in that work.
\begin{restatable}{theorem}{hardnesscompetitive}\label{hardness:thm:competitive}
For (asymptotic) competitive ratio of the full-information CBB problem, it holds: 
$$\lim_{T \rightarrow +\infty} \sup_{\pi} \rho^{\pi}(T) \leq  \frac{d_{\max}}{2d_{\max}-1}.$$
\end{restatable}

\section{Simulations} \label{simulations}
We simulate the \ucb algorithm for $60$ sample paths and $10k$ iterations on different instances, and report the mean, $25\%$ and $75\%$ trajectories of cumulative $\alpha$-regret. The $\alpha$-regret is defined empirically, using the solution of the LP as an upper bound on the optimal average reward  

In addition, we report three other quantities:

(i) The empirical probability that the LP solution causes round skipping. Recall at any time $t$ and having observed context $j_t \in \C$, \ucb samples arms using the extreme point $\{z_{j_t i}([t-M_t]^+): i \in \A \}$, and may return no arm if $\frac{1}{f_{j_t}}\sum_{i\in \A} z_{j_t i}([t-M_t]^+) < 1$. We denote this time-series by {\em lp skip} in the figures.

(ii) The empirical probability that the adaptive skipping technique actually skips a round to ensure future availability, even after an arm is sampled using the extreme point. We denote this time-series by {\em skip} in the figures.

(iii) The empirical blocking probability, namely, the time-average number of attempts to play an arm that fail due to blocking. We call it {\em block} in the figures. 

\paragraph{UCB Greedy:} We compare our algorithm with a UCB Greedy algorithm that plays the available arm which has the highest UCB index, given the observed context $j_t \in \C$, namely, $i^{g}_t = \arg\max_{i\in \A \text{ s.t. }F_{i,t}} \bar{\mu}_{i,j_t}(t)$, where $j_t$ is the context and $F_{i,t}$ is the event that any arm $i\in \A$ is available at time $t$. We do not use delayed exploitation for this algorithm, since there is no adaptive rounding, unlike \ucb. For this algorithm {\em lp skip} and {\em skip} both equal $0$ by construction, whereas {\em blocking} may occur.

\begin{figure}[h]
    \centering
\begin{subfigure}{0.3\textwidth}
    \includegraphics[width = 1\textwidth]{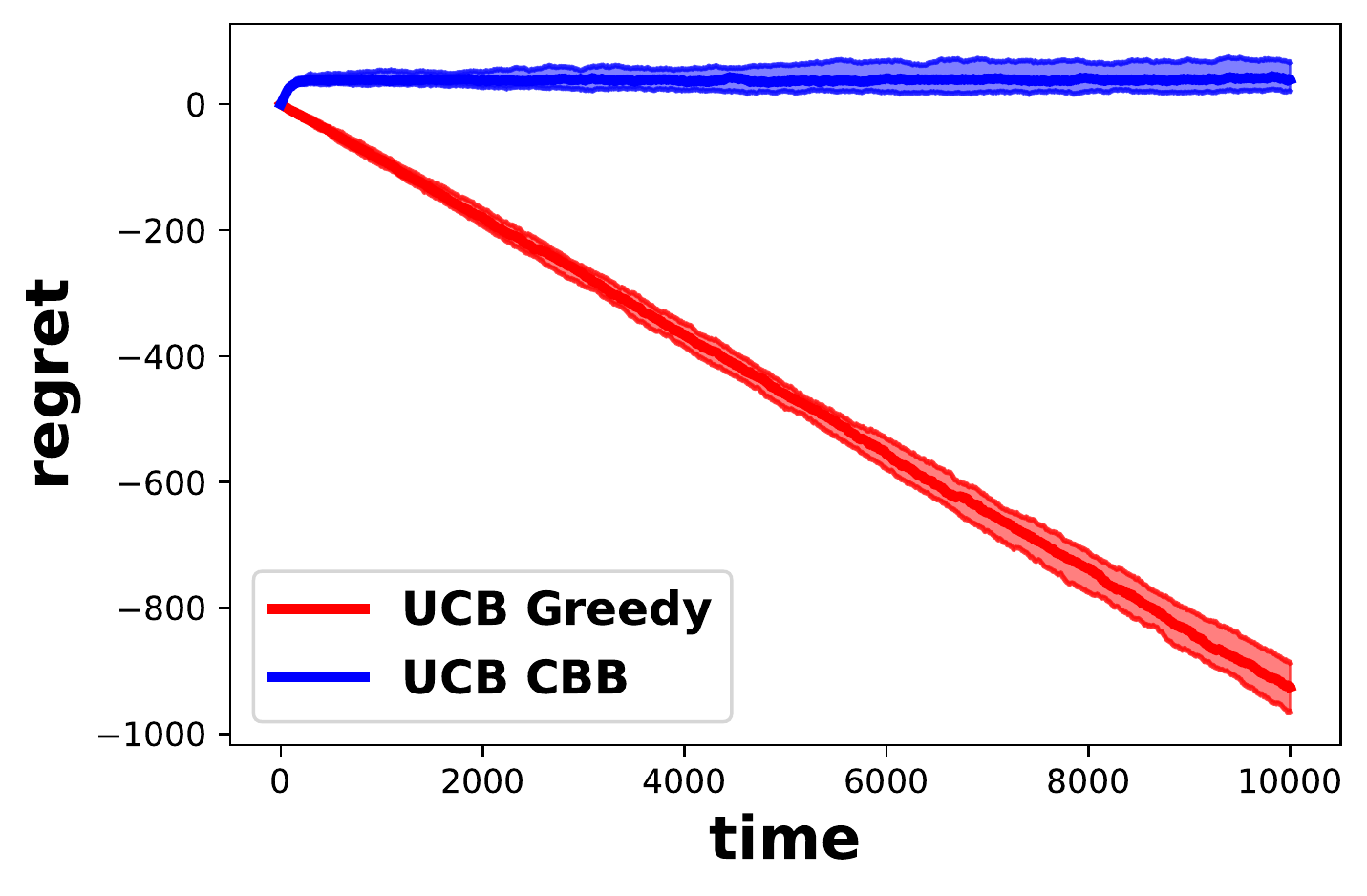}
    \caption{\small Cumulative Regret, $gap\mathtt{=}0.4$}
    \label{fig:integral_1}
\end{subfigure}
\begin{subfigure}{0.68\textwidth}
    \includegraphics[width = 1\textwidth]{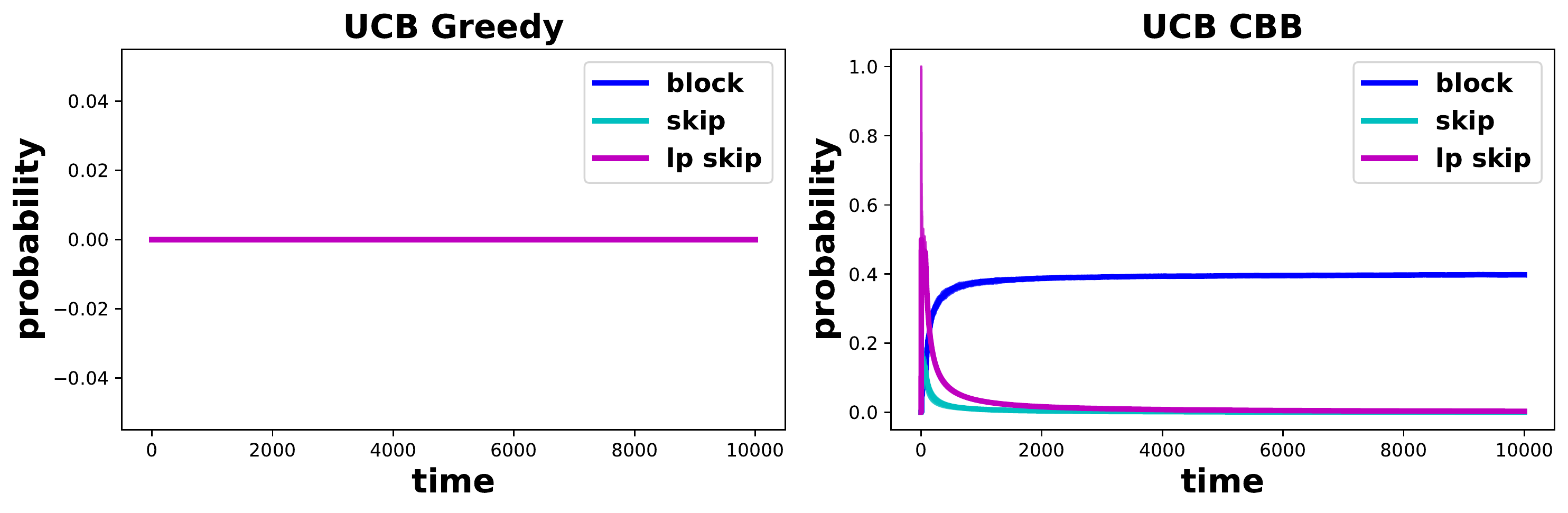}
    \caption{\small LP skipping, skipping, and blocking, $gap\mathtt{=}0.4$}
    \label{fig:integral_1b}
\end{subfigure}

\begin{subfigure}{0.3\textwidth}
    \includegraphics[width = 1\textwidth]{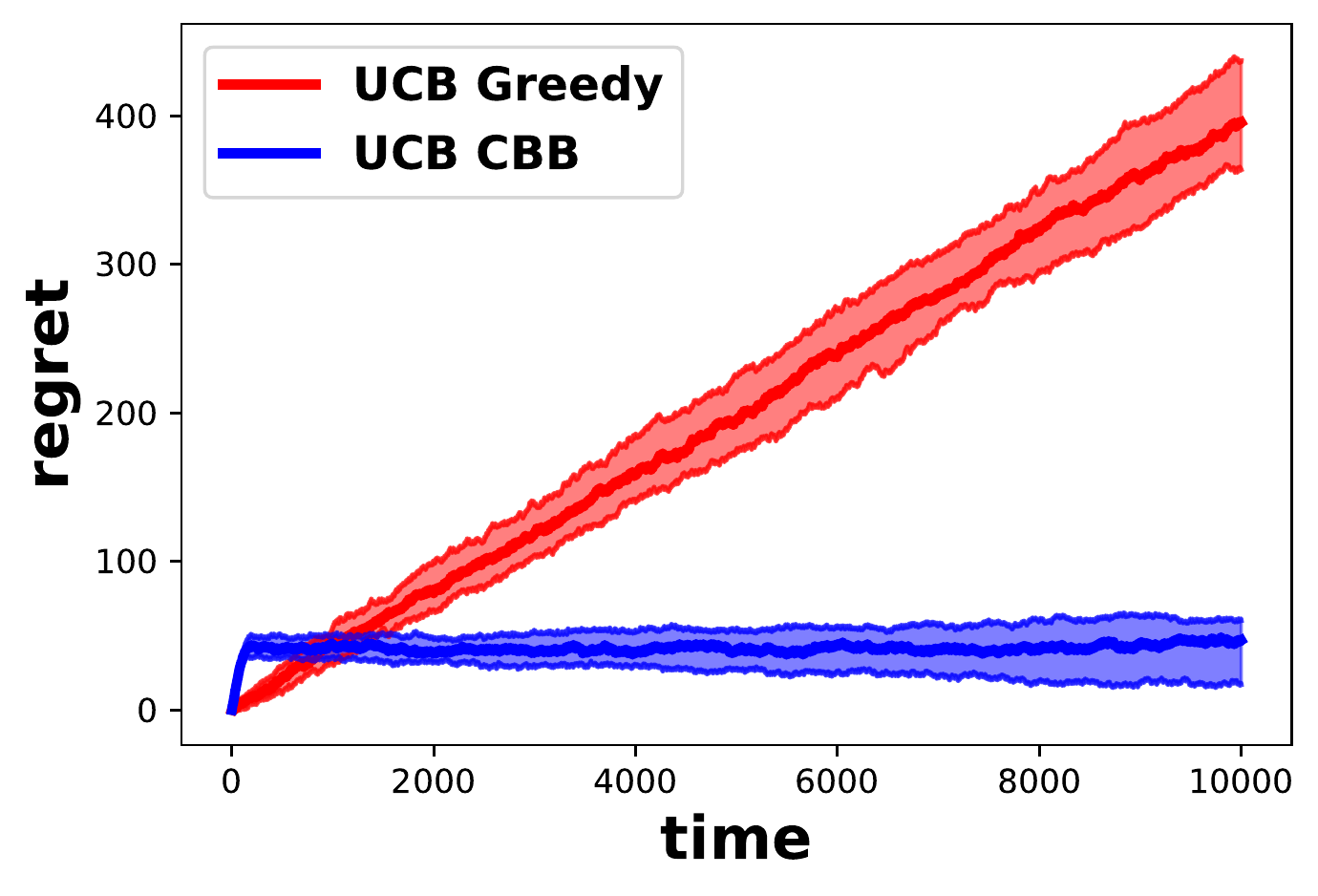}
    \caption{\small Cumulative Regret, $gap\mathtt{=}0.6$}
    \label{fig:integral_2}
\end{subfigure}
\begin{subfigure}{0.68\textwidth}
    \includegraphics[width = 1\textwidth]{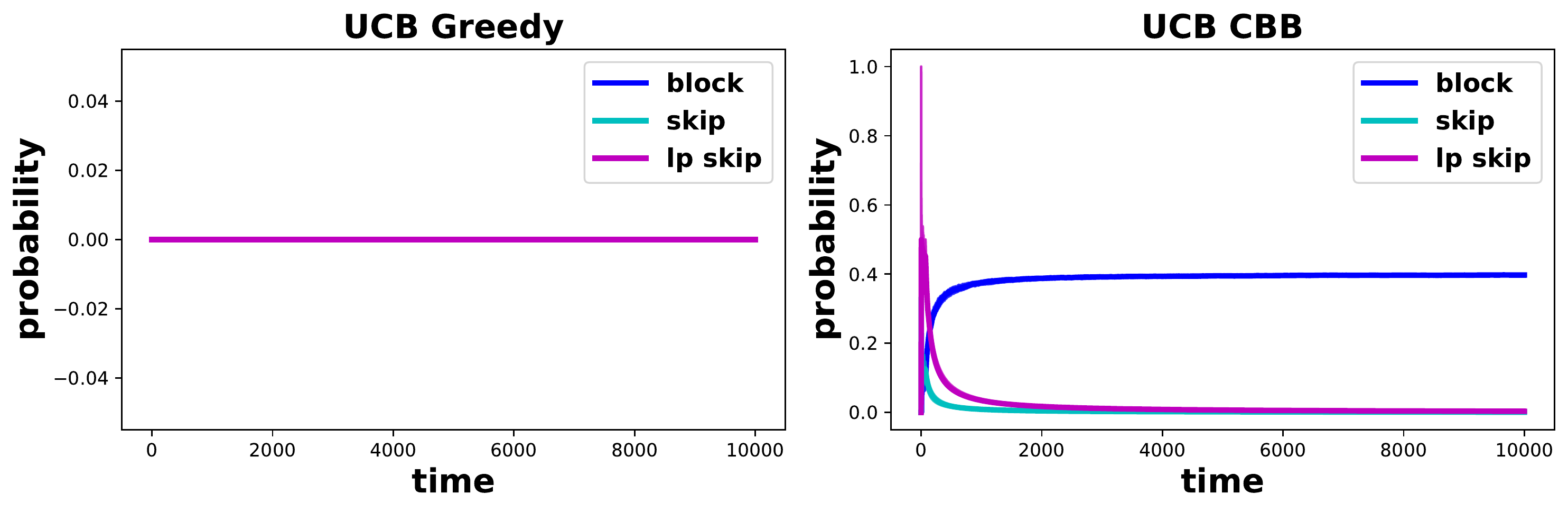}
    \caption{\small LP skipping, skipping, and blocking, $gap\mathtt{=}0.6$}
    \label{fig:integral_2b}
\end{subfigure}

\begin{subfigure}{0.3\textwidth}
    \includegraphics[width = 1\textwidth]{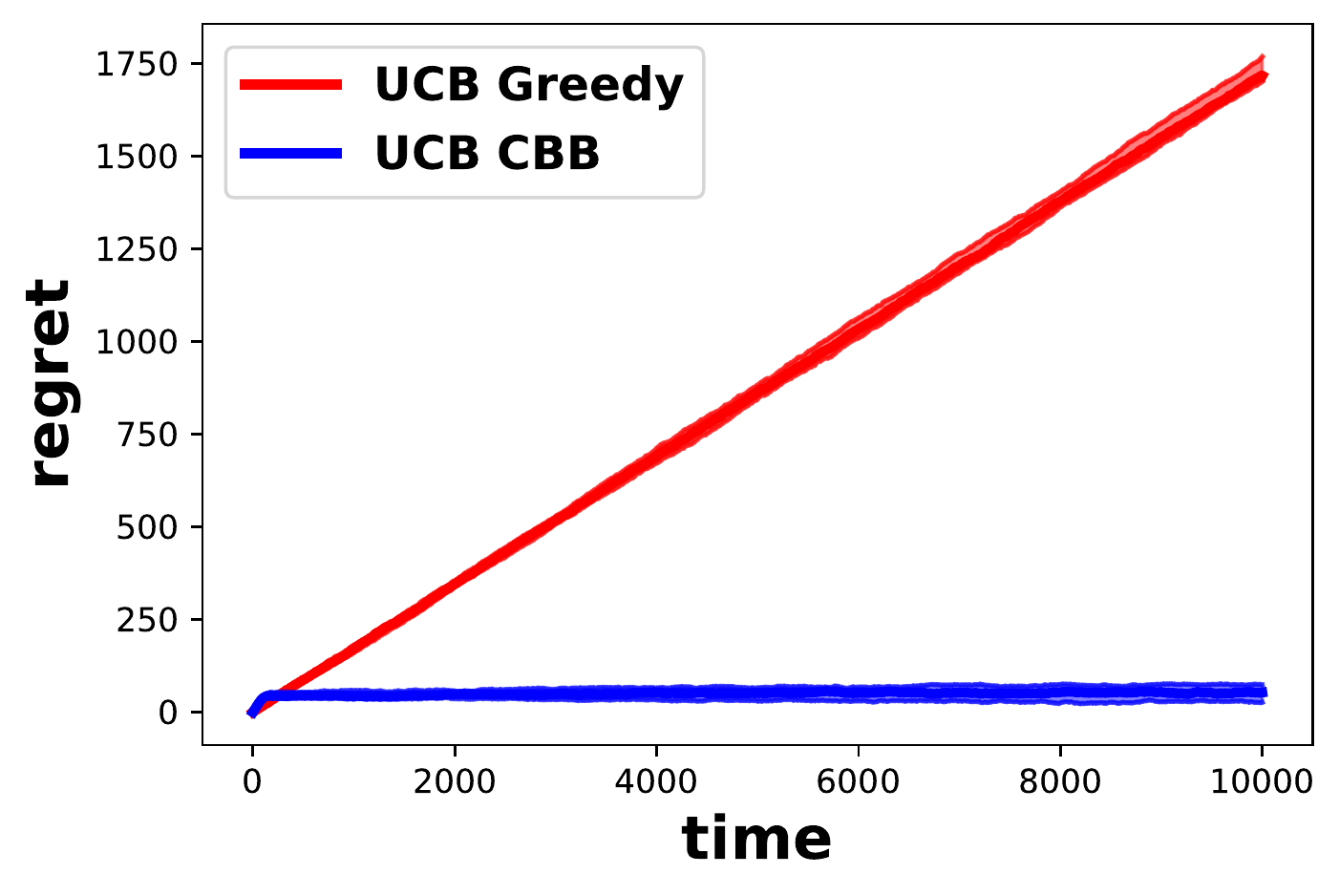}
    \caption{\small Cumulative Regret, $gap\mathtt{=}0.8$}
    \label{fig:integral_3}
\end{subfigure}
\begin{subfigure}{0.68\textwidth}
    \includegraphics[width = 1\textwidth]{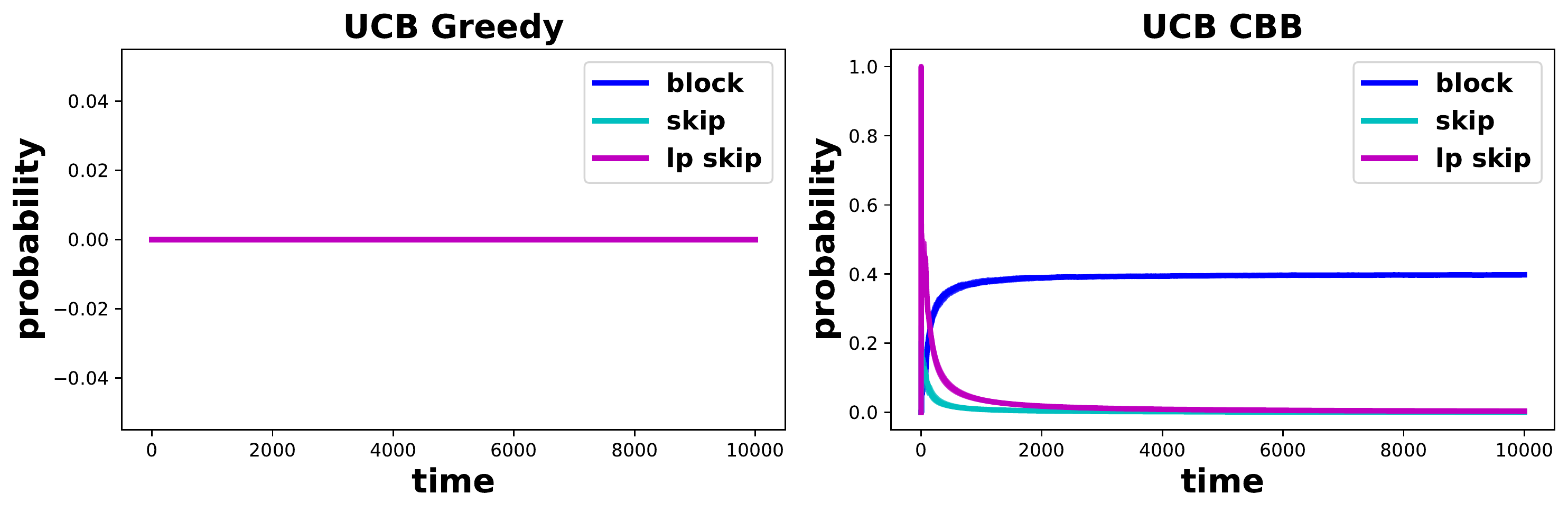}
    \caption{\small LP skipping, skipping, and blocking, $gap\mathtt{=}0.8$}
    \label{fig:integral_3b}
\end{subfigure}
\caption{\small Integral instance on $3$ arms with delay $3$ each, and $3$ equiprobable contexts with varying gap.}
\label{fig:integral}
\end{figure}

\paragraph{Integral Instances:} In this class, we consider $3$ arms each of delay $3$, and $3$ contexts that appear with equal probability.  For Figure~\ref{fig:integral} we have the $i$-th arm having a reward $0.9$ for the $i$-th context for all $i\in [3]$. Whereas, all the remaining arm-context pairs have mean $(0.9- gap)$ with $gap = 0.4$ for Figure~\ref{fig:integral_1},~\ref{fig:integral_1b}, $gap = 0.6$ for Figure~\ref{fig:integral_2},~\ref{fig:integral_2b}, and $gap = 0.8$ for Figure~\ref{fig:integral_3},~\ref{fig:integral_3b}. The rewards are generated by Bernoulli distributions. 

In these cases, the \eqref{lp:LP} admits a solution whose support yields a matching between arms and contexts, where arm $i$ is matched to context $i$ for $i\in [3]$. As a result, the marginal probabilities used by \ucb for sampling arms are integral. We see the \ucb algorithm has a $0.6$-Regret that grows logarithimically for all instances. Whereas, for the UCB Greedy algorithm the $0.6$-Regret is positive linear for $gap = 0.8$, and  $0.6$; but is negative linear for $gap = 0.4$. The Greedy algorithm beats the \ucb algorithm in the cumulative regret for $gap = 0.4$, as the effect of choosing the optimal matching in \ucb is countered by the effect of adaptively skipping at a rate $\tfrac{2}{5}$. On the other end, for $gap = 0.8, 0.6$ the \ucb algorithm performs better in the cumulative regret as the effect of choosing the optimal matching outweighs the effect of adaptive skipping. We note that this instance is dense, as $\sum_i \tfrac{1}{d_i} = 1$. Therefore, it is natural that the Greedy performs better when facing instances of smaller gaps among the rewards. In all the cases, the UCB Greedy algorithm incurs no blocking, whereas the \ucb algorithm converges to an empirical blocking rate of $\tfrac{2}{5}$.

\paragraph{Non-Integral Instances:} In this class, we consider three instances. The first two instances have $3$ arms and $3$ contexts, whereas the third has $10$ arms and $10$ contexts.

\begin{figure}[h]
    \centering
\begin{subfigure}{0.3\textwidth}
    \includegraphics[width = 1\textwidth]{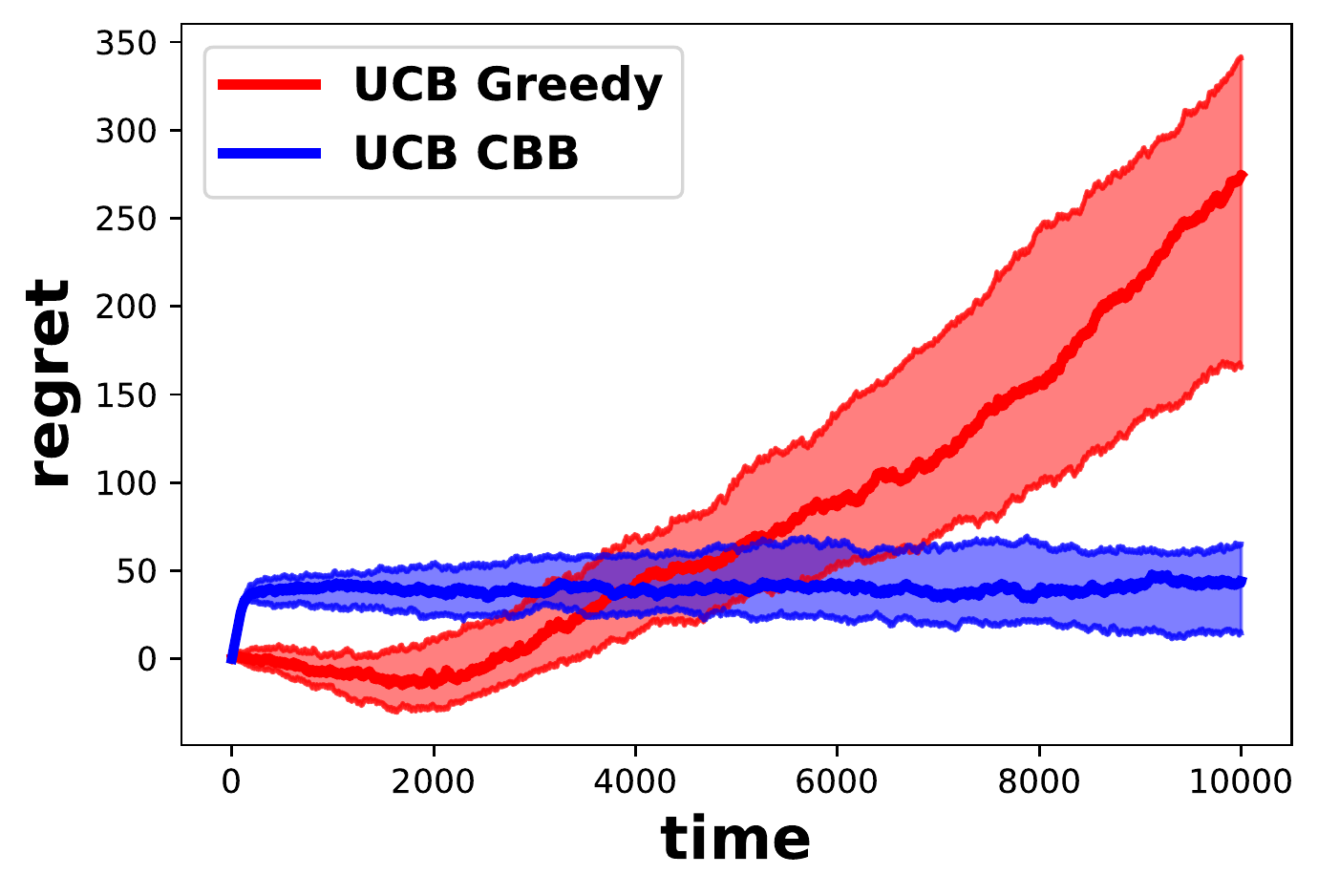}
    \caption{\small Cumulative Regret}
    \label{fig:3x3}
\end{subfigure}
\begin{subfigure}{0.68\textwidth}
    \includegraphics[width = 1\textwidth]{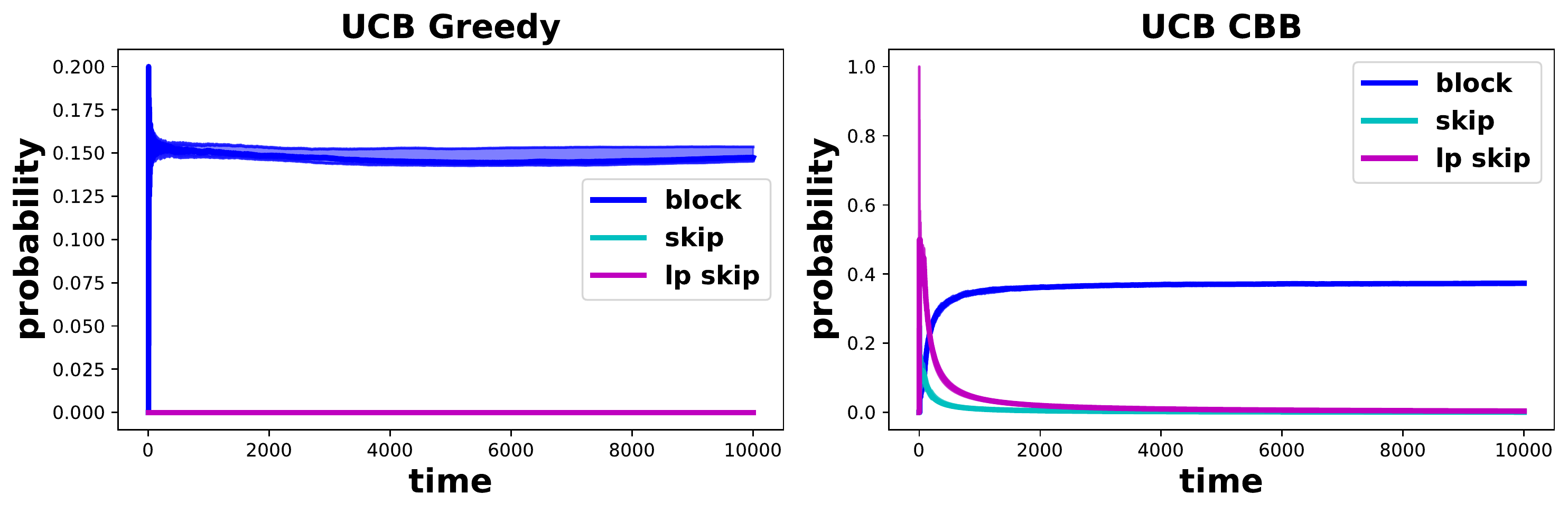}
    \caption{\small LP skipping, skipping, and blocking}
    \label{fig:3x30b}
\end{subfigure}
\caption{\small Non-integral instances with $3$ arms and $3$ contexts. The delays of the arms are either $2$, $3$, and $6$. The contexts are equiprobable. The best arm per context has arm-mean $0.9$, whereas all other arm-context pairs have means $0.3$.}
\label{fig:non_integral_3x3}
\end{figure}

In the first instance with $3$ arms and $3$ contexts, for each context $i\in [3]$ arm $i$ has mean reward $0.9$, whereas all other arm-context pairs have mean $0.3$. The contexts are equi-probable, whereas the arms have delays $2$, $3$ and $6$. For this instance, Figure~\ref{fig:non_integral_3x3} shows that the $\alpha$-regret is logarithmic for the \ucb algorithm and positive linear for UCB Greedy algorithm. We also observe the convergence in blocking probability for both algorithms in the same figure. We note that the UCB Greedy algorithm also incurs blocking for this dense ($\sum_i \tfrac{1}{d_i} = 1$) instance.

\begin{figure}[h]
    \centering
\begin{subfigure}{0.3\textwidth}
    \includegraphics[width = 1\textwidth]{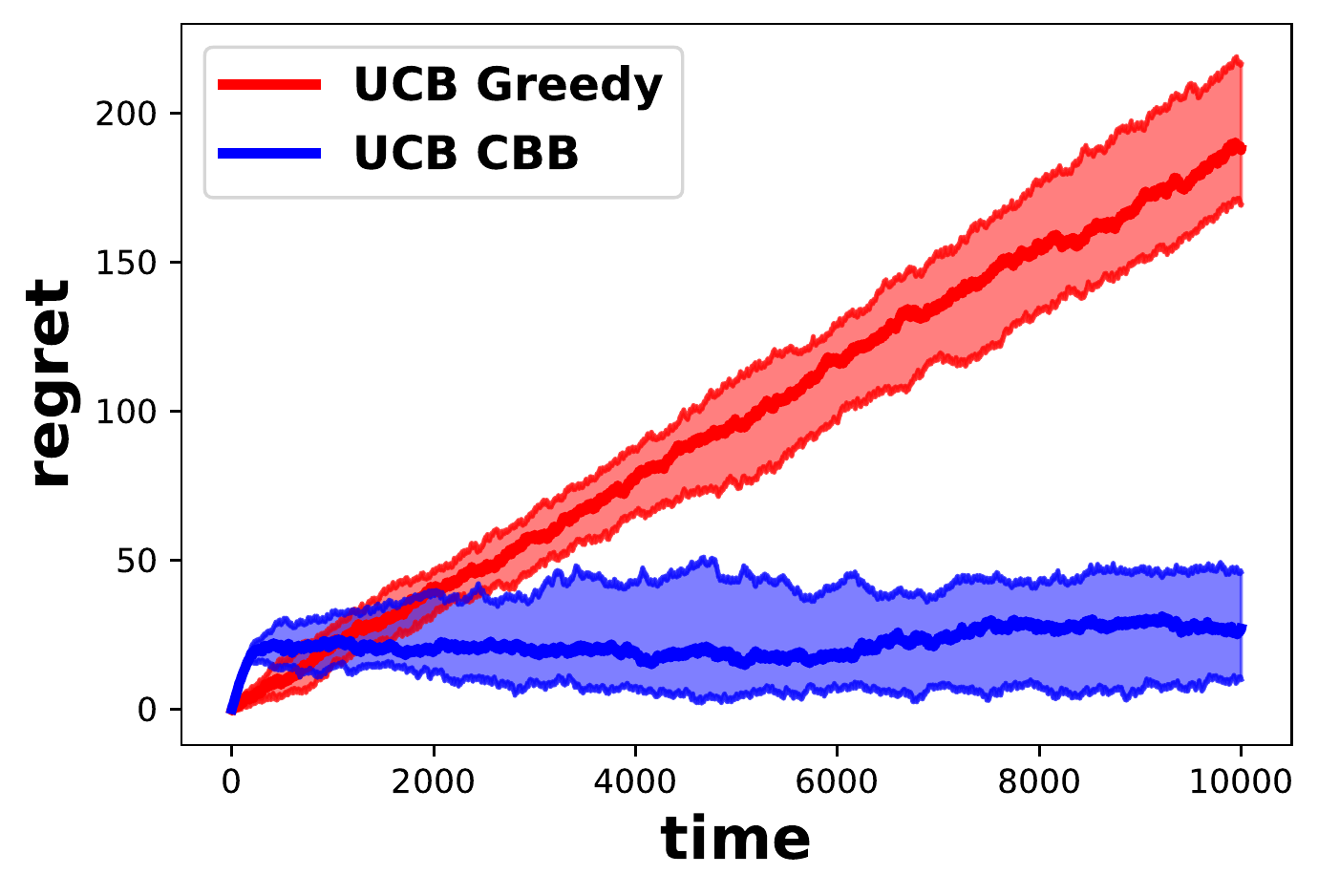}
    \caption{\small Cumulative Regret}
    \label{fig:3x3_D6}
\end{subfigure}
\begin{subfigure}{0.68\textwidth}
    \includegraphics[width = 1\textwidth]{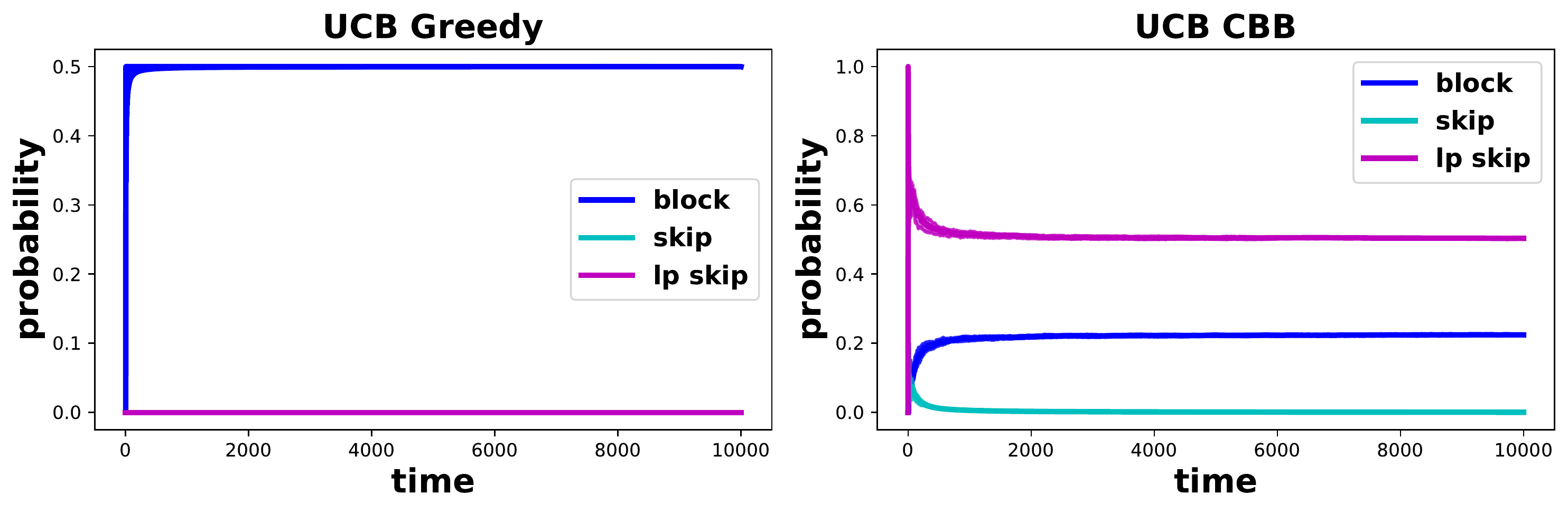}
    \caption{\small LP skipping, skipping, and blocking}
    \label{fig:3x3_D6b}
\end{subfigure}

\caption{\small Non-integral instances with $3$ arms and $3$ contexts. All the arms have delay $6$. The context probabilities are selected randomly. The best arm per context has mean in $[0.5, 0.9]$, whereas all other arm-context pairs have means in $[0, 0.3]$.}
\label{fig:non_integral_3x3_D6}
\end{figure}

In the second instance with $3$ arms and $3$ contexts, for each context $i\in [3]$ arm $i$ has mean u.a.r. $[0.5,0.9]$, whereas all other arm-context pairs have mean u.a.r. $[0,0.3]$. The context probabilities are again chosen randomly on a simplex. All the arms have delay equal to $6$. We note that this instance is non-dense, i.e. $\sum_i \tfrac{1}{d_i} = 1/2 < 1$. For this instance, Figure~\ref{fig:non_integral_3x3_D6} shows a logarithmic $\alpha$-regret for the \ucb algorithm and a linear regret for UCB Greedy. Both algorithms converge to non-zero probability of blocking. We note that the UCB Greedy algorithm incurs $0.5$ blocking for this non-dense instance, as compared to $0.22$ blocking in \ucb. This happens as \ucb conserves arm $i$ for context $i$, which can be seen through high {\em lp block} for \ucb and low regret. UCB Greedy, on the other hand, myopically plays the best available arm at each time slot, incurring high blocking and high regret.

\begin{figure}[H]
    \centering
\begin{subfigure}{0.3\textwidth}
    \includegraphics[width = 1\textwidth]{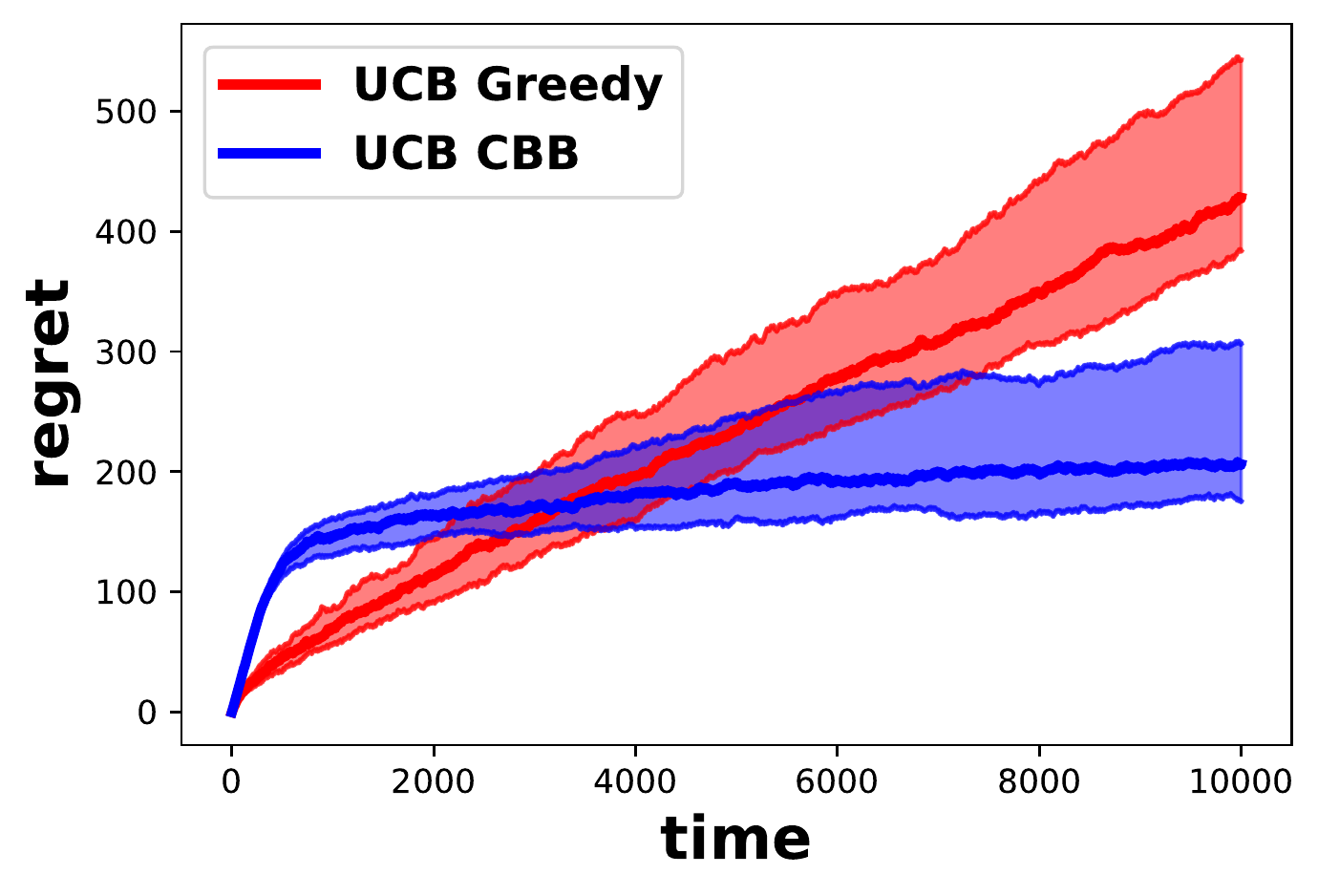}
    \caption{\small Cumulative Regret}
    \label{fig:10x10}
\end{subfigure}
\begin{subfigure}{0.68\textwidth}
    \includegraphics[width = 1\textwidth]{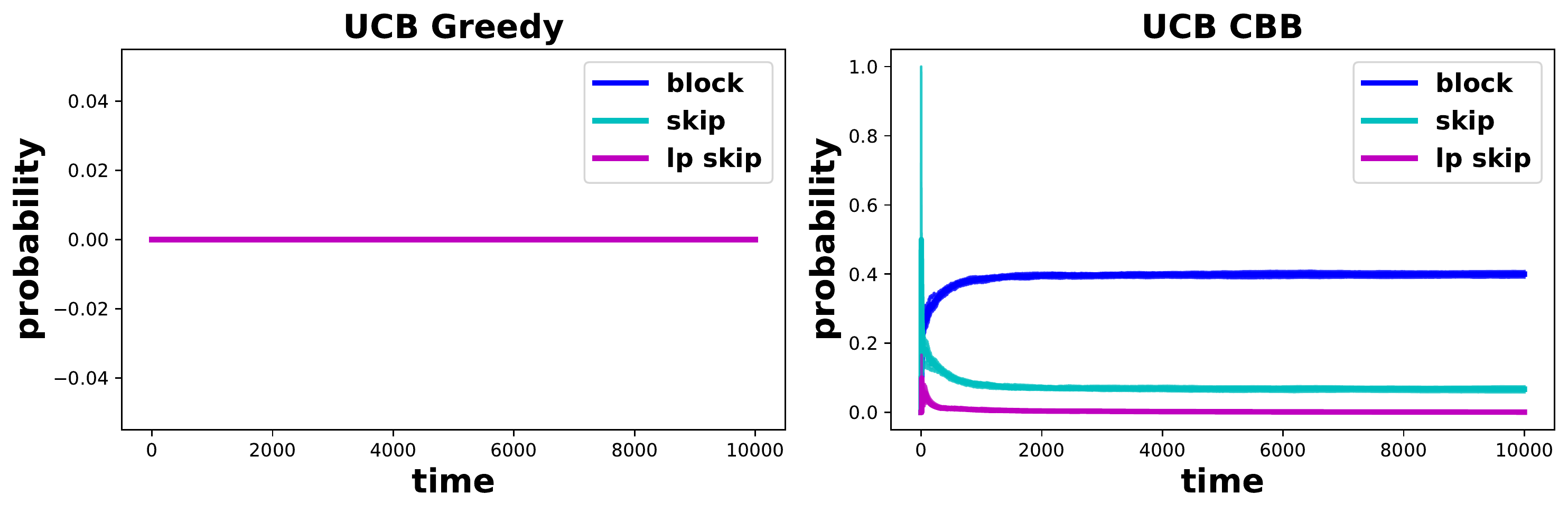}
    \caption{\small LP skipping, skipping, and blocking}
    \label{fig:10x10b}
\end{subfigure}
\caption{\small Non-integral instances with $10$ arms and $10$ contexts. The delays of the arms are either $8$ or $9$ generated randomly. The context probabilities are selected randomly. The best arm per context has mean $0.9$, whereas all other arm-context pairs have means in $[0, 0.3]$.}
\label{fig:non_integral_10x10}
\end{figure}

In the last instance with $10$ arms and $10$ contexts, for each context $i\in [10]$ arm $i$ has mean $0.9$, whereas all other arm-context pairs have mean chosen uniformly at random  (u.a.r.) from $[0, 0.3]$.  The context probabilities are chosen randomly from the $10$-D simplex. The arm delays are chosen randomly from $8$ and $9$ with equal probability. For this instance, Figure~\ref{fig:non_integral_10x10} shows similar trends as the non-integral instance with $3$ arms and $3$ contexts. Here, we observe that for \ucb algorithm the adaptive skipping converges to a non zero value ($0.08$, approx.), which plays a crucial part in balancing the instantaneous reward and the future availability. The UCB Greedy algorithm does not incur blocking, since the fact that $\sum_i \tfrac{1}{d_i} > 1$ ensures that at least one arm is always available.

\bibliographystyle{plainnat}
\bibliography{ref}

\newpage

\appendix
\section{Technical notation} \label{appendix:notation}
For any event $\mathcal{E}$, we denote by $\event{\mathcal{E}} \in \{0,1\}$ the indicator variable that takes the value of $1$ if $\mathcal{E}$ occurs and $0$, otherwise. For any number $n \in \mathbb{N}$, we define $[n] = \{1,2, \dots, n\}$ and for any integer $r \in \mathbb{Z}$, we define $[r]^+:= \max\{r,0\}$. Moreover, we use the notation $t \in [a, b]$ (for $a \leq b$) for some time index $t$, in lieu of $t \in [T] \cap \{a,a+1. \dots, b-1, b\}$. Unless otherwise noted, we use the indices $i$ or $i'$ to refer to arms, $j$ or $j'$ to refer to contexts and $t$, $t'$ or $\tau$ to refer to time. We use $\log(\cdot)$ for the logarithm of base $2$ and $\ln(\cdot)$ for the natural logarithm. Let $A^{\pi}_t \in \A \cup \{\emptyset\}$ be the arm played by some algorithm $\pi$ at time $t$ and let $F^{\pi}_{i,t}$ be the event that arm $i$ is {\em free} (i.e. not blocked) at time $t$ for some algorithm $\pi$. We denote by $C_t \in \C$ (or simply $j_t \in \C$) the observed context of round $t$. For a given instance $I$, let $d_{\max} = \max_{i \in \A}\{d_i\}$ be the maximum delay of an arm. In this reading, expectations can be taken over the randomness of the nature, including the sampling of contexts (denoted by $\mathcal{R}_C$) and the arm rewards (denoted by $\mathcal{R}_X$), as well as the random bits of the corresponding algorithm (denoted by $\mathcal{R}_{\pi}$ for an algorithm $\pi$). We denote by $\mathcal{R}_{N,\pi}$ the randomness generated by the combination of the aforementioned factors.
\section{Omitted pseudocodes}
\subsection{Pseudocode of algorithm \oracle} \label{appendix:online:pseudo}

\begin{algorithm}[H]
\SetAlgoLined
\DontPrintSemicolon
Compute an optimal solution $\{z^*_{i,j}\}_{\forall i,j}$ to \eqref{lp:LP}.\;
Initialize the non-skipping probabilities: $q_{i,1} \gets 1$, $\beta_{i,1} \gets \frac{d_i}{2d_i-1}$, $\forall i \in \A$.\;
\For{$t = 1, 2, \dots$}{
    Observe context $j_t \in \C$. \;
    Generate $u,v\sim U[0,1]$. \;
    Sample arm $i_t$ such that $u \in\left[\sum_{i'=1}^{i-1} \frac{z^*_{i',j_t}}{f_{j_t}}, \sum_{i'=1}^{i} \frac{z^*_{i',j_t}}{f_{j_t}}\right)$ (assuming a fixed arm order).\;
  \eIf{$i_t \neq \emptyset$ \textbf{and} $i_t$ is available \textbf{and} $v \leq \beta_{i_t,t}$}{
   Play arm $i_t$.\;
   }
   {
    Skip the round without playing any arm.\;
   }\;
   \For{$i \in \A$ \textbf{such that} $d_i \geq 2$}{
    $q_{i,t+1} \gets q_{i,t}\left(1 - \beta_{i,t} \sum_{j \in \C} z^*_{i,j}\right) + \event{t\geq d_i} q_{i,t-d_i+1} \beta_{i,t-d_i+1}\sum_{j \in \C} z^*_{i,j}$.\;
    $\beta_{i,t+1} \gets \min\{1 , \frac{d_i}{2d_i-1} \frac{1}{q_{i,t+1}}\}$.\;
   }
 }
 \caption{\oracle}
\label{alg:oracle}
\end{algorithm}

\subsection{Computation of the conditional non-skipping probability, \textsc{compq} $(i, t, H_{t-M_t})$} \label{appendix:ucb:exantepseudo}

\begin{algorithm}[H]
\SetAlgoLined
\DontPrintSemicolon
\If{$(i, t)$ in Cache}{
\textbf{return} Cache$[(i, t)]$ // Global Cache\;
}
Let $Z(t')$ be the solution of \eqref{lp:LP}$(t')~\forall t' \in [T]$ and $Z(0) = Z(\tau)$ $\forall \tau \leq 0$ be an initial solution. \;
Set $t_{0} \gets $ the first time on or after $\max\{1, t-M_t\}$, when arm $i$ becomes available. \;
Set $q_{i,t_0} \gets 1$. \;
\For{$t' = t_0, \dots, t-1$}{
    $t'' \gets t'-d_i+1$. \;
    $\beta_{i,\tau} \gets \min\{1, \frac{d_i}{2d_i-1}\frac{1}{\textsc{compq}~(i,\tau,H_{\tau-M_{\tau}})}\}$ for $\tau \in \{t', t''\}$. \;
    $q_{i,t'+1} \gets q_{i,t'}\left(1 - \beta_{i,t'}\sum_{j \in \C} z_{i,j}(t'-M_{t'})\right) + \event{t''\geq t_0 } q_{i,t''} \beta_{i,t''} \sum_{j \in \C} z_{i,j}(t''-M_{t''})$.\;
}
Cache$[(i, t)] =  q_{i,t}$. // Memorization \;
Remove all $(i',t')$ s.t. $t' < t - M_t$ from Cache. //Garbage Collection \;
\textbf{return}$~q_{i,t}$. \;
 \caption{\textsc{compq}$(i, t, H_{t-M_t})$}
\label{alg:exante}
\end{algorithm}

\section{Discussions}
\subsection{Optimizing over \eqref{lp:LP} using combinatorial methods} \label{appendix:discuss:lps}
The linear formulation \eqref{lp:LP} contains $k \cdot m$ variables and $ k\cdot m+k + m$ constraints (including the non-negativity constraints). 

From a practical perspective, an optimal extreme point solution to \eqref{lp:LP} can be computed efficiently using fast combinatorial methods. Indeed, every instance of the \eqref{lp:LP} can be transformed into an instance of the well-studied \textsc{maximum weighted flow} problem and solved by standard techniques such as cycle canceling \citep{GT89} or fast implementations of the dual simplex method for network polytopes \citep{OPT93}. 

We now describe the reduction: We consider a node $i$ for every arm $i \in \A$ and a node $j$ for every context $j \in \C$. We define two additional nodes: a source node $s$ and a sink node $t$. For each variable $z_{i,j}$, we associate an edge $(i,j)$ of capacity $c_{i,j} = +\infty$ and weight $w_{i,j} = \mu_{i,j}$. In addition, for each node $i \in \A$, we consider an edge $(s,i)$ of weight $w_{s,i} = 0$ and capacity $c_{s,i} = 1/d_i$, while for each node $j \in \C$, we consider an edge $(j,t)$ of weight $w_{j,t} = 0$ and capacity $c_{j,t} = f_j$. It is not hard to verify that the optimal solution to \eqref{lp:LP} coincides with a flow of maximum weight in the aforementioned network.

\subsection{Suboptimality gaps} \label{appendix:discuss:gaps}
In general, the suboptimality gaps, $\Delta^{i,j}_{\min}$, of the LP are complex functions of the means, $\{\mu_{i,j}\}_{i,j}$, arm delays, $\{d_i\}_i$, and context distribution, $\{f_j\}_j$. This fact should not be surprising-- it is the combination of all these parameters that determines how an optimal (or near-optimal) solution must behave.

Interestingly, when applied to the standard MAB\footnote{The standard MAB regret lower bound is $\mathcal{O}(k\cdot \frac{\log(T)}{\Delta})$, where $\Delta$ is the minimum gap between two arms.} problem~\citep{LR85} (i.e., single context and unit delays), the gap $\Delta^{i,j}_{\min}$ for $i>1$, matches the standard notion of gap $\Delta_i = \mu_{1,j} - \mu_{i,j}$, where $i=1$ is the arm of highest mean reward (and $j$ the unique context). 

As another example of suboptimality gaps, consider the following structured instance: Let $k>2$ arms and $m = k$ contexts. All the arms have equal delay $d_i = k, \forall i \in \A$ and all contexts appear with equal probability $f_j = \frac{1}{k}, \forall j \in \C$. We assume that $\mu_{i,j} = \Delta >0$, if $i=j$, and $\mu_{i,j} = 0$, otherwise. In the above instance, it is not hard to verify that the variables $\{z_{i,j}\}_{i,j}$ in any extreme point solution of \eqref{lp:LP} take values in $\{0,1/k\}$. Moreover, the support of the optimal extreme point solution corresponds to a maximum bipartite matching (w.r.t. the edge weights $\{\mu_{i,j}\}_{i,j}$) in the underlying bipartite graph consisting of arm (left) and context (right) nodes. 

Let $M \subset [k] \times [k]$ be the maximum matching in the above bipartite graph with respect to the mean values. Moreover, we define $M_{i,j} \subset [k] \times [k]$ for any $i \neq j$ to be a maximal matching in the above graph that necessarily contains the edge $(i,j)$ of $\mu_{i,j}=0$ (which corresponds to a matching of $k-2$ edges). In addition, we define $M_{i,i} = M \setminus (i,i)$, namely, the maximum matching with the edge $(i,i)$ removed. Using the above definitions, we can see that the optimal solution to \eqref{lp:LP} can be expressed as $\sum_{(i,j) \in M} \Delta z^*_{i,j} = k \Delta \frac{1}{k} = \Delta$. It is not hard to verify that the suboptimality gap of any pair $(i,j)$ with $i \neq j$ can be expressed as
\begin{align*}
\Delta^{i,j}_{\min} = \Delta - \sum_{(i',j')\in M_{i,j}} \Delta \frac{1}{k} = \Delta - \frac{k-2}{k}\Delta = \frac{2}{k}\Delta. 
\end{align*}
Finally, for the suboptimality gap of any pair $(i,i)$, we have 
\begin{align*}
\Delta^{i,i}_{\min} = \Delta - \sum_{(i',j')\in M_{i,i}} \Delta \frac{1}{k} = \Delta - \frac{k-1}{k}\Delta = \frac{1}{k}\Delta. 
\end{align*}

\subsection{Difference in $\alpha$-regret definition} \label{appendix:discuss:rergetdef}
We note that in Definition 5 in \cite{CWYW16}, a super-arm (which is analogous to an extreme point of \eqref{lp:LP} in our paper) is defined as {\em bad}, if the reward from this super arm is less than $\alpha$ times the reward of an optimal super arm. However, in our case an extreme point is {\em bad} if its reward is less than $1$ times (not $\tfrac{d_{\max}}{2d_{\max}-1}$ times) the optimal solution of the LP \eqref{lp:LP}. This difference is present in our paper, as we require solving the LP \eqref{lp:LP} {\em optimally with probability} $1$ at each time slot, in order to ensure a $\tfrac{d_{\max}}{2d_{\max}-1}$-approximation algorithm. This is in contrast with the combinatorial bandits literature~\cite{WC17,CWYW16}, where in each time slot the oracle provides an $\alpha$-approximate solution to the combinatorial problem with probability at least $\beta$, for $\alpha, \beta \in (0,1]$. Our approximation loss comes from the online rounding, rather than from the LP solution at each time slot.
\section{Concentration inequalities} \label{appendix:concentration}
In this section, we outline the standard concentration results that we use in our proofs.
%\begin{theorem}[Multiplicative Chernoff Bound]\footnote{A proof of this standard concentration result can be found in \citep{MU17}.}\label{appendix:concentration:standardchernoff}
%Let $X_1, \dots, X_n$ be independent identically distributed random variables taking values in $\{0,1\}$ and of mean value $\Ex{ }{X_i} = \mu, \forall i \in [n]$. Let $Y = X_1 + \dots + X_n$. Then for all $0 < \delta \leq 1$,
%\begin{align*}
%    \Pro{Y \geq \left(1+\delta\right) \mu n} \leq e^{-\frac{\delta^2 \mu n}{3}}.
%\end{align*}
%\end{theorem}

\begin{theorem}[Hoeffding's Inequality]\footnote{This is a standard concentration result and the statement can be found, e.g., in \cite{LS18}}\label{appendix:concentration:hoeffding}
Let $X_1, \dots, X_n$ be independent identically distributed random variables with common support in $[0,1]$ and mean $\mu$. Let $Y = X_1 + \dots + X_n$. Then for all $\delta \geq 0$,
\begin{align*}
    \Pro{|Y-n\mu| \geq \delta} \leq 2 e^{-2\delta^2/n}.
\end{align*}
\end{theorem}

\begin{theorem}[Multiplicative Chernoff Bound]\footnote{The result is a combination of Theorem 4.5 and Exercise 4.7 in \citep{MU17}, in the case where the $\{X_i\}_{i \in [n]}$ are independent. The authors in \citep{WC17,CWYW16} describe a slight modification that directly proves the statement.} \label{lemma:regret:chernoff}
Let $X_1, \dots, X_n$ be Bernoulli random variables taking values from $\{0,1\}$, and $\Ex{ }{X_t|X_{t-1}, \dots, X_{1}} \geq \mu$ for every $t \leq n$. Let $Y = X_1 + \dots + X_n$. Then, for all $0 < \delta < 1$,
\begin{align*}
\Pro{Y \leq (1-\delta)n\mu} \leq e^{- \frac{\delta^2 n \mu}{2}}.
\end{align*}
\end{theorem}
\section{Full-information problem and competitive analysis: omitted proofs}\label{appendix:online}

\subsection{Proof of Theorem \ref{online:theorem:competitive}}
\label{appendix:theorem:competitive}
We now prove a lower bound on the competitive guarantee of \oracle, against any optimal clairvoyant algorithm. The proofs of the lemmas we use in the proof of the following theorem are also contained in this section of the Appendix.

\restateTheoremOracle*
\begin{proof}
The first step in our analysis is to show that the optimal solution of \eqref{lp:LP}, denoted by $\Rew_I^{LP}$ yields a $\left(1-\frac{d_{\max}-1}{d_{\max}-1 + T}\right)$-approximate upper bound to the maximum (average) expected reward collected by any (clairvoyant) algorithm, denoted by $\Rew^{*}_I(T)$. Note that, since $\Rew_I^{LP}$ represents an upper bound on the average collected reward, we multiply it with $T$, in order to compare it with $\Rew^{*}_I(T)$. Finally, we emphasize that the multiplicative approximation of the upper bound asymptotically goes to $1$ as $T$ increases.
\begin{restatable}{lemma}{restateOverviewUpperbound}\label{lemma:overview:upperbound}
For any time horizon $T$, we have 
\begin{align*}
T \cdot \Rew_I^{LP} \geq \left(1 - \frac{d_{\max}-1}{d_{\max}-1 + T}\right)\Rew^*_I(T). 
\end{align*}
\end{restatable}

We denote by $F^{\pi}_{i,t}$ the event that arm $i$ is available in time $t$, and by $A^{\pi}_t$ the arm played at time $t$, where $A^{\pi}_t \in \{\A \cup ~\emptyset\}$. Moreover, we denote the event of playing arm $i$ at context $j$ at time $t$ as $\event{A^{\pi}_t = i, C_t = j}$ for all $i\in \A$ and $j \in \C$. We fix a time horizon $T$, for the purpose of the analysis.

The \oracle algorithm at each time $t$ plays an arm $i$ if it is {\em (i)} sampled, {\em (ii)} available and {\em (iii)} not skipped. The sampling of arm $i$ under context $j \in \C$ happens with probability $z^*_{i,j}$ and the arm is not skipped with probability $\beta_{i,t}$, independently. Finally, the arm is played if it is available, which happens independently of sampling and skipping, with probability
$\Pro{F^{\pi}_{i,t}}$. The above analysis leads to a recursive characterization of $\Pro{F^{\pi}_{i,t}}$. Upon inspection, this is the same characterization as for $q_{i,t}$ given in Eq.~\eqref{online:exante:formula}. We formally summarize the above in the following lemma:

\begin{restatable}{lemma}{restateOracleExante}\label{lemma:oracle:exante}
At every round $t \in [T]$ and for any arm $i \in \A$ and context $j \in \C$, it is the case that $\Pro{A^{\pi}_t = i , C_t = j} = z^*_{i,j} \beta_{i,t} \Pro{F^{\pi}_{i,t}}$. Moreover, we have $q_{i,t} = \Pro{F^{\pi}_{i,t}}$, $\forall i\in \A, t \in [T]$.
\end{restatable}

We observe that by design of the skipping mechanism $\beta_{i,t}$, the quantity $\beta_{i,t} \cdot \Pro{F^{\pi}_{i,t}}$ never exceeds $\tfrac{d_i}{2d_i -1}$. Leveraging this observation, we show that at every time $t \in [T]$, it is the case that $\Pro{F^{\pi}_{i,t}} \geq \tfrac{d_i}{2d_i -1}$. This allows us to completely characterize the behavior of the algorithm as it is shown in the following lemma: 
\begin{restatable}{lemma}{restateOracleExact}\label{lemma:oracle:exact}
At every round $t \in [T]$, the probability that \oracle plays an arm $i \in \A$ under context $j \in \C$ is exactly $\Pro{A^{\pi}_t = i , C_t = j} = \frac{d_i}{2d_i-1}z^*_{i,j}$.
\end{restatable}

In order to complete the proof of the theorem, the expected cumulative reward collected by \oracle in $T$ time steps can be expressed as 
\begin{align}
\Rew_I^{\pi}(T) &= \Ex{\mathcal{R}_{N,\pi}}{\sum_{t \in [T]} \sum_{i \in \A} \sum_{j \in \C} X_{i,j,t} \event{{A^{\pi}_{t}=i , C_{t}=j}}} \nonumber\\
&= \Ex{\mathcal{R}_{C} \mathcal{R}_{\pi}}{\sum_{t \in [T]} \sum_{i \in \A} \sum_{j \in \C} \Ex{\mathcal{R}_{X}}{X_{i,j,t}} \event{{A^{\pi}_{t}=i , C_{t}=j}}} \label{eq:on:thm:1}\\
&= \Ex{\mathcal{R}_{C} \mathcal{R}_{\pi}}{\sum_{t \in [T]} \sum_{i \in \A} \sum_{j \in \C} \mu_{i,j} \event{{A^{\pi}_{t}=i , C_{t}=j}}} \nonumber\\
&= \sum_{t \in [T]} \sum_{i \in \A} \sum_{j \in \C} \mu_{i,j} \Ex{\mathcal{R}_{C} \mathcal{R}_{\pi}}{ \event{{A^{\pi}_{t}=i , C_{t}=j}}} \nonumber\\
&= T \sum_{i \in \A} \sum_{j \in \C} \mu_{i,j} \frac{d_i}{2d_i -1} z^*_{i,j} \label{eq:on:thm:2}\\
&\geq \frac{d_{\max}}{2d_{\max} -1} T \cdot \Rew_I^{LP} \nonumber\\
&\geq \frac{d_{\max}}{2d_{\max} -1} \left(1 - \frac{d_{\max}-1}{d_{\max}-1 + T}\right) \Rew_I^{*}(T), \label{eq:on:thm:3}
\end{align} 
where \eqref{eq:on:thm:1} follows by independence of $\{X_{i,j,t}\}_{i,j,t}$, \eqref{eq:on:thm:2} follows by Lemma \ref{lemma:oracle:exact} and \eqref{eq:on:thm:3} follows by Lemma \ref{lemma:overview:upperbound}.
\end{proof}

\subsection{Proof of Lemma \ref{lemma:overview:upperbound}} \label{appendix:overview:upperbound}
\restateOverviewUpperbound*

\begin{proof}
We denote by $\Sigma: [T] \rightarrow \C$ a fixed sequence of context realizations over $T$ rounds, where, at each time step $t \in [T]$, context $j \in \C$ appears independently with probability $f_j$. Let $\mathcal{S}$ be the family of all possible sequences. Given that the context of each round is sampled independently according to the fixed probabilities $\{f_j\}_{j \in \C}$, the probability of each sequence is given by $\Pro{\Sigma} = \prod_{t \in [T]} f_{\Sigma(t)}$. Note that we overload the notation and denote by $\Sigma$ the event that the sequence is realized.

Consider the optimal clairvoyant algorithm that first observes the full context realization and, then, chooses a fixed feasible arm-pulling sequence that yields the maximum expected reward for this realization. Let $\event{A^*_t = i, C_t = j~|~\Sigma}$ be the indicator of the event that under the realization $\Sigma$, the optimal algorithm plays arm $i$ on time $t$ under context $j$. We emphasize the fact that the event ${C_t = j}$ is deterministic conditioned on the realization $\Sigma$. Finally, notice that we can assume w.l.o.g. that there exists an optimal clairvoyant policy maximizing the expected reward that ignores the realizations of the collected rewards.

We fix any realization $\Sigma \in \mathcal{S}$. In any feasible solution and for any arm $i \in \A$, we have
\begin{align*}
    \sum_{t' \in [t, t + d_i - 1]} \sum_{j \in \C} \event{A^*_t = i, C_t = j~|~\Sigma} \leq 1, \qquad\qquad \forall t \in [T],
\end{align*}
as the arm can be played at most once during any $d_i$ consecutive time steps. By summing the above inequalities over all $t \in [T]$, for any arm $i \in \A$, we get
\begin{align*}
    &\sum_{t \in [1,d_i-1]} t \sum_{j \in \C}\event{A^*_t = i, C_t = j~|~\Sigma} + \sum_{t \in [d_i,T]} d_i \sum_{j \in \C}\event{A^*_t = i, C_t = j~|~\Sigma} \leq T \\
    &\Leftrightarrow \sum_{t \in [T]} d_i \sum_{j \in \C}\event{A^*_t = i, C_t = j~|~\Sigma} \leq T + \sum_{t \in [1,d_i-1]} (d_i - t) \sum_{j \in \C}\event{A^*_t = i, C_t = j~|~\Sigma}.
\end{align*}
By feasibility of \eqref{lp:LP}, we have that $\sum_{t \in [1,d_i-1]} (d_i - t) \sum_{j \in \C} \event{A^*_t = i, C_t = j~|~\Sigma} \leq d_i - 1$. Therefore, by dividing the above inequality by $d_i \cdot T$, we get
\begin{align*}
    \frac{1}{T}\sum_{t \in [T]} \sum_{j \in \C} \event{A^*_t = i, C_t = j~|~\Sigma} \leq \frac{1}{d_i} \left(1 + \frac{d_i-1}{T}\right), \forall i \in \A.
\end{align*}
Now, by multiplying the above inequality with the probability of each context realization $\Sigma$ and taking the sum over all $\Sigma \in \mathcal{S}$, we get
\begin{align}
    \sum_{j \in \C} \sum_{\Sigma \in \mathcal{S}}\frac{1}{T}\sum_{t \in [T]} \event{A^*_t = i, C_t = j~|~\Sigma} \Pro{\Sigma} \leq \frac{1}{d_i} \left(1 + \frac{d_i-1}{T}\right), \forall i \in \A. \label{eq:flp:feasibility:window}
\end{align}

For each context $j \in \C$ and any time $t \in [T]$, we have 

$$
\sum_{i \in \A} \event{A^*_t = i, C_t = j~|~\Sigma} \leq \event{C_t = j~|~\Sigma},
$$
where the inequality follows by the fact that at most one arm is played at each time in any feasible solution. By taking the expectation in the above expression over the context realization, we get
$$
\Ex{\mathcal{R}_{C}}{\sum_{i \in \A} \event{A^*_t = i, C_t = j~|~\Sigma}} \leq \Ex{\mathcal{R}_{C}}{\event{C_t = j~|~\Sigma}} = \sum_{\Sigma \in \mathcal{S}} \event{C_t = j~|~\Sigma} \Pro{\Sigma}= f_j,
$$
where the last equality follows by the fact that the probability that any context realization sequence satisfies $C_t = j$ is exactly $f_j$. Finally, by taking the sum of the above inequality over all $t \in [T]$ and dividing by $T$ yields

\begin{align}
\sum_{i \in \A} \sum_{\Sigma \in \mathcal{S}} \frac{1}{T} \sum_{t \in [T]}\event{A^*_t = i, C_t = j~|~\Sigma} \Pro{\Sigma} \leq f_j. \label{eq:flp:feasibility:context}
\end{align}

For the expected cumulative reward of the above optimal clairvoyant policy, we have: 
\begin{align}
\Rew^{*}_I(T)=&\Ex{\mathcal{R}_{X},\mathcal{R}_C}{\max_{\text{feasible}\{A^*_t\}_{t \in [T]}}\bigg\{\sum_{t \in [T]} \sum_{i \in \A} \sum_{j \in \C} X_{i,j,t}\event{A_t^* = i, C_t = j} \bigg\}} \notag\\
&= \Ex{\mathcal{R}_{X}}{\sum_{\Sigma \in \mathcal{S}}\sum_{t \in [T]} \sum_{i \in \A} \sum_{j \in \C} X_{i,j,t}\event{A_t^* = i, C_t = j~|~\Sigma} \Pro{\Sigma} }\notag\\
&= \sum_{\Sigma \in \mathcal{S}}\sum_{t \in [T]} \sum_{i \in \A} \sum_{j \in \C} \Ex{\mathcal{R}_{X}}{X_{i,j,t}}\event{A_t^* = i, C_t = j~|~\Sigma} \Pro{\Sigma}\notag\\
&= \sum_{i \in \A} \sum_{j \in \C} \sum_{\Sigma \in \mathcal{S}}\sum_{t \in [T]} \mu_{i,j}\event{A_t^* = i, C_t = j~|~\Sigma} \Pro{\Sigma}, \label{eq:flp:feasibility:objective}
\end{align}
where the second and third equalities follow by the fact that the optimal clairvoyant policy plays a fixed arm-pulling solution for any observed context realization sequence and that this solution is independent of the observed reward realizations.

Consider now a (candidate) solution of \eqref{lp:LP}, such that:
$$
z_{i,j} = \left(1 + \frac{d_{\max} - 1}{T}\right)^{-1} \sum_{\Sigma \in \mathcal{S}}\frac{1}{T}\sum_{t \in [T]} \event{A_t^* = i, C_t = j~|~\Sigma} \Pro{\Sigma}, \forall i\in \A, j \in \C.
$$
It is not hard to verify that, for this assignment, constraints \eqref{flp:window} and \eqref{flp:conditional} are satisfied by making use of \eqref{eq:flp:feasibility:window} and \eqref{eq:flp:feasibility:context}, respectively. Moreover, for the objective of \eqref{lp:LP}, using \eqref{eq:flp:feasibility:objective}, we have: 
\begin{align*}
    T \sum_{i \in \A} \sum_{j \in \C} \mu_{i,j} z_{i,j} 
    &= \left(1 + \frac{d_{\max} - 1}{T}\right)^{-1} \sum_{i \in \A} \sum_{j \in \C} \mu_{i,j}  \sum_{\Sigma \in \mathcal{S}}\sum_{t \in [T]} \event{A_t^* = i, C_t = j~|~\Sigma} \Pro{\Sigma} \\
    &= \left(1 + \frac{d_{\max} - 1}{T}\right)^{-1} \Rew^{*}_I(T) \\
    &= \left(1 - \frac{d_{\max}-1}{d_{\max}-1 + T} \right) \Rew^{*}_I(T),
\end{align*}
where in the last equality follows by the fact that $\frac{1}{1 + \delta} = 1 - \frac{\delta}{1 + \delta}$ for any $\delta \in \mathbb{R}$. Therefore, by exhibiting a feasible solution to \eqref{lp:LP} of value $\left(1 - \frac{d_{\max}-1}{d_{\max}-1 + T} \right) \Rew^{*}_I(T)$, we can conclude that $\Rew^{LP}_I \geq \left(1 - \frac{d_{\max}-1}{d_{\max}-1 + T} \right) \Rew^{*}_I(T)$.
\end{proof}

\subsection{Proof of Lemma \ref{lemma:oracle:exante}}

\restateOracleExante*

\begin{proof}
Although our algorithm \oracle computes and uses an optimal extreme point solution to \eqref{lp:LP}, the analysis that follows holds for any feasible solution $\{z_{i,j}\}_{i,j}$. We denote by $S^{\pi}_{i,t}$ the event that arm $i$ is sampled by \oracle at round $t$ (with probability $\Pro{S^{\pi}_{i,t}}=\frac{z_{i,j_t}}{f_{j_t}}$ for a sampled context $j_t$) and by $B^{\pi}_{i,t}$ the event that arm $i$ is not skipped at round $t$. Finally, we denote by $F^{\pi}_{i,t}$ the event that arm $i$ is available at the beginning of round $t$.

In order to prove the first part of the claim, we first notice that the event $\{A^{\pi}_t = i\}$ is equivalent to $\{S^{\pi}_{i,t} , B^{\pi}_{i,t} , F^{\pi}_{i,t}\}$, namely, in order for an arm $i$ to be played during $t$, the arm needs to be sampled, not skipped and available. For any fixed $i \in \A$, $j \in \C$ and $t \in [T]$, we have: 
\begin{align}
    \Pro{A^{\pi}_t = i , C_t = j} &= \Pro{A^{\pi}_t = i| C_t=j} \Pro{C_t=j} \nonumber\\
    &= f_j \Pro{A^{\pi}_t = i| C_t=j} \label{eq:on:exante:1}\\
    &= f_j \Pro{S^{\pi}_{i,t}, F^{\pi}_{i,t} , B^{\pi}_{i,t}| C_t=j} \nonumber\\
    &= f_j \Pro{S^{\pi}_{i,t} | C_t=j} \Pro{B^{\pi}_{i,t}| C_t=j} \Pro{F^{\pi}_{i,t}| C_t=j} \label{eq:on:exante:2} \\ 
    &= f_j \Pro{S^{\pi}_{i,t} | C_t=j} \Pro{B^{\pi}_{i,t}} \Pro{F^{\pi}_{i,t}} \label{eq:on:exante:3} \\
    &= f_j \beta_{i,t} \Pro{S^{\pi}_{i,t} | C_t=j} \Pro{F^{\pi}_{i,t}} \label{eq:on:exante:4}\\
    &= f_j \frac{z_{i,j}}{ f_j} \beta_{i,t} \Pro{F^{\pi}_{i,t}} \label{eq:on:exante:5}\\
    &= z_{i,j} \beta_{i,t} \Pro{F^{\pi}_{i,t}} \nonumber.
\end{align}
In the above analysis, equality \eqref{eq:on:exante:1} follows by the fact that $\Pro{C_t = j} = f_j$, while in \eqref{eq:on:exante:2} we use the fact that the events $S^{\pi}_{i,t}$, $B^{\pi}_{i,t}$ and $F^{\pi}_{i,t}$ are mutually independent, by construction of our algorithm. Moreover, in \eqref{eq:on:exante:3} we use that the events $B^{\pi}_{i,t}$ and $F^{\pi}_{i,t}$ are independent of the observed type $j \in \C$. Finally, in \eqref{eq:on:exante:4}  and \eqref{eq:on:exante:5}, we use the fact that $\Pro{S^{\pi}_{i,t} | C_t =j} = 
\frac{z_{i,j}}{f_j}$ and $\Pro{B^{\pi}_{i,t}} = \beta_{i,t}$, by construction of our algorithm.

We now prove the second part of the statement, namely, that the computed probabilities, $\{q_{i,t}\}_{\forall i, t}$ (by the recursive formula \eqref{online:exante:formula}), indeed match the actual a priori probabilities of the events $\{F^{\pi}_{i,t}\}_{\forall i, t}$. The main idea behind the computation of $q_{i,t}$ is that an arm is available at some round $t$, if it is available but not played at time $t-1$, or if it is played at time $t-d_i$.

For any fixed arm $i \in \A$, we prove the statement by induction on the number of rounds. Note that we only consider arms such that $d_i \geq 2$, since, otherwise, we trivially have that $q_{i,t} = \Pro{F^{\pi}_{i,t}} = 1, \forall t \in [T]$. Clearly, for $t = 1$ the computed probabilities are correct, since $\Pro{F^{\pi}_{i,1}} = 1$. We assume that up to round $t-1$, the computed probabilities are correct, namely, $q_{i,t'} = \Pro{F^{\pi}_{i,t'}}$, $\forall t' \in [t-1]$. Considering the event $F^{\pi}_{i,t}$, we have: 
\begin{align}
    \event{F^{\pi}_{i,t}} &= \event{F^{\pi}_{i,t}, F^{\pi}_{i,t-1}} + \event{F^{\pi}_{i,t}, \neg F^{\pi}_{i,t-1}} \nonumber \\
    &= \event{F^{\pi}_{i,t}, F^{\pi}_{i,t-1}} + \event{F^{\pi}_{i,t}, \neg F^{\pi}_{i,t-1}, t \geq d_i + 1} + \event{F^{\pi}_{i,t}, \neg F^{\pi}_{i,t-1},t \leq d_i} \nonumber\\
    &= \event{F^{\pi}_{i,t}, F^{\pi}_{i,t-1}} + \event{F^{\pi}_{i,t}, \neg F^{\pi}_{i,t-1}, t \geq d_i + 1} \label{eq:on:exante:6}\\
    &= \event{F^{\pi}_{i,t}, F^{\pi}_{i,t-1}} + \event{F^{\pi}_{i,t}, \neg F^{\pi}_{i,t-1}, F^{\pi}_{i,t-d_i}, t \geq d_i + 1}\nonumber\\
    &\quad\quad\quad\quad\quad\quad\quad\quad\quad\quad\quad\quad\quad+ \event{F^{\pi}_{i,t}, \neg F^{\pi}_{i,t-1}, \neg F^{\pi}_{i,t-d_i}, t \geq d_i + 1} \nonumber\\
    &= \event{F^{\pi}_{i,t}, F^{\pi}_{i,t-1}} + \event{F^{\pi}_{i,t}, \neg F^{\pi}_{i,t-1}, F^{\pi}_{i,t-d_i}, t \geq d_i + 1} \label{eq:on:exante:7}.
\end{align}
In equality \eqref{eq:on:exante:6}, we use the fact that the event $\{F^{\pi}_{i,t}, \neg F^{\pi}_{i,t-1}, t \leq d_i\}$ is empty. This follows by noticing that for $t \leq d_i$, if an arm is not available at round $t-1$, then it has to be pulled during some round $t' \in [t-1] \subseteq [d_i-1]$ and, thus, cannot be available on round $t$. Similarly, in \eqref{eq:on:exante:7}, we use the fact that the event $\{F^{\pi}_{i,t}, \neg F^{\pi}_{i,t-1}, \neg F^{\pi}_{i,t-d_i}, t \geq d_i + 1\}$ is empty. The reason is that, if the arm is not available at round $t-d_i$ and neither at round $t-1$, this implies that the arm is pulled during some round $t' \in [t-d_i+1, t-2]$. However, if the arm is played at any such $t'$, then it cannot be available at round $t$.

Notice, that the event $\{F^{\pi}_{i,t}, F^{\pi}_{i,t-1}\}$ occurs with probability $(1 - \beta_{i,t-1}\sum_{j \in \C} z_{i,j}) \Pro{F^{\pi}_{i,t-1}}$, since the arm is, either not selected on round $t-1$, i.e., $\event{S^{\pi}_{i,t-1}} = 0$, or skipped, i.e., $\event{B^{\pi}_{i,t-1}} = 0$. Moreover, the event $\{F^{\pi}_{i,t}, \neg F^{\pi}_{i,t-1}, F^{\pi}_{i,t-d_i}, t \geq d_i + 1\}$, for $t \geq d_i +1$ is equivalent to the event $\{A^{\pi}_{t-d_i} = i\}$, since the arm has to be played at time $t-d_i$, in order to be available at round $t$ for the first time after $t-d_i$. By taking expectations in \eqref{eq:on:exante:7} and combining the above facts, we have: 
\begin{align}
    \Pro{F^{\pi}_{i,t}} &= \Pro{F^{\pi}_{i,t},F^{\pi}_{i,t-1}} + \event{t \geq d_i +1} \Pro{F^{\pi}_{i,t}, \neg F^{\pi}_{i,t-1}, F^{\pi}_{i,t-d_i}} \nonumber \\
    &=\left(1 - \Pro{S^{\pi}_{i,t-1}, B^{\pi}_{i,t-1}} \right) \Pro{F^{\pi}_{i,t-1}} + \event{t \geq d_i +1} \Pro{A^{\pi}_{t-d_i} = i} \nonumber \\
    &=\left(1 - \Pro{S^{\pi}_{i,t-1}, B^{\pi}_{i,t-1}} \right) \Pro{F^{\pi}_{i,t-1}} + \event{t \geq d_i +1} \Pro{S^{\pi}_{i,t-d_i}, B^{\pi}_{i,t-d_i}, F^{\pi}_{i,t-d_i}} \nonumber \\
    &=\left(1 - \sum_{j \in \C}\Pro{S^{\pi}_{i,t-1}, B^{\pi}_{i,t-1}, C_{t-1}=j} \right) \Pro{F^{\pi}_{i,t-1}} \nonumber\\
    &~~~~~~~~~~~~~~~~~~~~~~~~~~~~+ \event{t \geq d_i +1} \sum_{j \in \C}\Pro{S^{\pi}_{i,t-d_i}, B^{\pi}_{i,t-d_i}, F^{\pi}_{i,t-d_i}, C_{t-d_i}=j} \nonumber \\
    &=\left(1 - \beta_{i,t-1}\sum_{j \in \C} z_{i,j}\right) \Pro{F^{\pi}_{i,t-1}} + \event{t \geq d_i +1} \beta_{i,t-d_i}\left(\sum_{j \in \C} z_{i,j} \right) \Pro{F^{\pi}_{i,t-d_i}} \label{eq:on:exante:8},
\end{align}
where \eqref{eq:on:exante:8} follows by the analysis of the first part of this proof. By setting $t+1$ instead of $t$ in the above relation and setting $z_{i,j} = z^*_{i,j}, \forall i \in \A, j \in \C$, we can easily verify that the formula that computes these probabilities in formula \eqref{online:exante:formula} and Algorithm \ref{alg:oracle} is correct, which concludes the proof of this lemma.
\end{proof}

\subsection{Proof of Lemma \ref{lemma:oracle:exact}}

\restateOracleExact*

\begin{proof}
Similarly to the proof of Lemma \ref{lemma:oracle:exante}, the analysis of this proof holds true for any feasible solution $\{z_{i,j}\}_{\forall i,j}$ of \eqref{lp:LP}, including the optimal extreme point solution. Recall that by Lemma \ref{lemma:oracle:exante}, the probability of each event $F^{\pi}_{i,t}$ is equal to the actual probability of the event, namely, $q_{i,t} = \Pro{F^{\pi}_{i,t}}$, $\forall i \in \A, t \in [T]$. Moreover, by the same lemma, we have: 
\begin{align}
    \Pro{A^{\pi}_t = i , C_t = j} = z^*_{i,j} \beta_{i,t} \Pro{F^{\pi}_{i,t}} \label{eq:on:exact:1}
\end{align}
Recall that $\beta_{i,t} = \min\{1, \frac{d_i}{2d_i-1} \frac{1}{q_{i,t}}\}$. We now prove by induction that for every fixed arm $i \in \A$ and for every time $t \in [T]$, it is the case that: $\Pro{A^{\pi}_t = i , C_t = j}=\frac{d_i}{2d_i-1} z^*_{i,j}$. Clearly, for $t=1$, we have $\Pro{F^{\pi}_{i,1}} = 1 = q_{i,1}$ (by initialization) and, thus, $\beta_{i,1} = \frac{d_i}{2d_i-1}$, implying that $\Pro{A^{\pi}_1 = i, C_1 = j} = \frac{d_i}{2d_i-1} z^*_{i,j}$. Suppose the argument is true for any $\tau \in [t-1]$. For time $t$, we distinguish between two cases:

\textbf{Case (a).} Suppose $\beta_{i,t} < 1$. Then, by construction, it has to be that $\beta_{i,t} = \frac{d_i}{2d_i-1} \frac{1}{q_{i,t}}$, while by Lemma \ref{lemma:oracle:exante}, we have that $q_{i,t} = \Pro{F^{\pi}_{i,t}}$. By \eqref{eq:on:exact:1}, this immediately implies that $\Pro{A^{\pi}_{t}=i , C_{t}=j} = \frac{d_i}{2d_i-1} z^*_{i,j}$. 

\textbf{Case (b).} Suppose $\beta_{i,t} = 1$. Then, by definition of $\beta_{i,t}$, it has to be that $q_{i,t} \leq \frac{d_i}{2d_i-1} $, which in turn implies that $\Pro{F^{\pi}_{i,t}} \leq \frac{d_i}{2d_i-1}$. Therefore, we can upper bound the probability of interest as: $\Pro{A_{t}=i , C_{t}=j} = z^*_{i,j} \beta_{i,t}\Pro{F^{\pi}_{i,t}} \leq \frac{d_i}{2d_i-1} z^*_{i,j}$. In order to complete the induction step, it suffices to also show that $\Pro{A^{\pi}_{t}=i , C_{t}=j} \geq \frac{d_i}{2d_i-1} z^*_{i,j}$. By a simple union bound, we can lower bound the probability of arm $i$ being available using the probabilities that the arm has been played within time $[t-d_i+1,t-1]$:
\begin{align*}
\Pro{A^{\pi}_{t}=i , C_{t}=j} &\geq z^*_{i,j} \Pro{F^{\pi}_{i,t}}\\
&= z^*_{i,j} \left(1 - \Pro{\neg F^{\pi}_{i,t}} \right)\\
& \geq z^*_{i,j} \left(1 - \sum_{t' \in [t-d_i+1,t-1]}\sum_{j' \in \C} \Pro{A^{\pi}_{t'} = i , C_{t'} = j'} \right).
\end{align*}
However, by induction hypothesis we know that $\forall t' \in [t - d_i+1, t-1]$ and $\forall j' \in \C$, it is the case that $\Pro{A^{\pi}_{t'} = i , C_{t'} = j'} = \frac{d_i}{2d_i-1} z^*_{i,j'}$. Moreover, by constraints \eqref{flp:window} of \eqref{lp:LP}, we know that $\sum_{t' \in [t-d_i+1, t-1]} \sum_{j' \in \C} z^*_{i,j'} \leq \frac{d_i-1}{d_i}$. Combining the above facts, we have: 
\begin{align*}
\Pro{A^{\pi}_{t}=i , C_{t}=j}  &\geq z^*_{i,j}\left(1 - \frac{d_i}{2d_i-1} \sum_{t' \in [t-d_i,t-1]}\sum_{j' \in \C} z^*_{i,j'}\right)\\
&\geq z^*_{i,j}\left(1 - \frac{d_i}{2d_i-1} \frac{d_i-1}{d_i}\right)\\
&\geq \frac{d_i}{2 d_i -1} z^*_{i,j},
\end{align*}
which completes our induction step, since the combination of two inequalities implies that $\Pro{A^{\pi}_t = i, C_t = j} = \frac{d_i}{2d_i-1} z^*_{i,j}$.
\end{proof}

\newpage
\section{Bandit problem and regret analysis: omitted proofs} \label{appendix:omitted}
\subsection{Properties of $M_t$ and the critical time $T_c$} \label{appendix:criticaltime}
We consider the delayed exploitation parameter, $M_t$, specifically defined as
\begin{align*}
    M_t =  \lfloor2\log_{c_1}(t)\rfloor  + \lceil\log_{c_1}(c_0)\rceil +1 = \lfloor2\log_{c_1}(t)\rfloor + 2d_{\max} + 8,
\end{align*}
where $c_0 = e \left(\frac{e^2}{e^2-1}\right)^{2d_{\max}}$ and $c_1 = \frac{e^2}{e^2-1}$. 

We define the {\em critical} round $T_c$, as the smallest integer such that $T_c - M_{T_c} \geq 1$. It is not hard to verify that, by definition of $M_t$, this implies that $t - M_{t} \geq 1$ for all $t\geq T_c$ (see the next paragraph), and $t - M_{t} \leq 0$ for all $t\leq T_c -1$ (by definition of $T_c$). By definition of the algorithm, at each round $t \geq T_c$ and in order to sample the next arm to be played, \ucb uses an extreme point solution computed with respect to the UCB estimates before exactly $M_t$ time steps (i.e., at round $t-M_t$). For $t \leq T_c$, where $t - M_t \leq 0$, the algorithm uses an initially computed extreme point solution $Z(0) = \{z_{i,j}(0)\}_{i,j}$. In this section, we study several useful properties of $M_t$ and $T_C$.

\paragraph{Bounded increases of $M_t$.} We now show that for $t \geq T_c$, the value of $M_t$ increases by at most one unit per round. This fact significantly simplifies our proofs and results in the analysis of the $\alpha$-regret.

Let us first compute the condition that must be satisfied for any $t \in [T]$, such that the value $t - M_t$ is strictly positive. Formally
\begin{align*}
    t - M_t \geq 1 \Longleftrightarrow t - \lfloor2\log_{c_1}(t)\rfloor  \geq 2d_{\max} + 9.
\end{align*}

By noticing that $d_{\max} \geq 1$, we can easily verify that by the time $t - \lfloor2\log_{c_1}(t)\rfloor  \geq 2d_{\max} + 9$, it also holds that $t - \lfloor2\log_{c_1}(t)\rfloor  \geq 11$, which, in turn, implies that $t \geq 69$.

We are now looking for the smallest $t$, after which $M_t$ increases by at most one unit per round. Consider the (fractional) {\em breakpoints} of the form $c^{i/2}_1$ for any positive integer $i$. These breakpoints, corresponds to the points such that the value of $M_t$ increases, when time $t$ passes them. We consider intervals of the form $C_i = [c^{i/2}_1, c^{(i+1)/2}_1)$. The first step is to find the smallest $i$, such that there is at least one integral point in $[c^{i/2}_1, c^{(i+1)/2}_1)$. Notice that the above condition is true if
\begin{align*}
    c^{(i+1)/2}_1 - c^{i/2}_1  \geq 1 \Longleftrightarrow c^{i/2}_1 \geq \frac{1}{\sqrt{c_1} - 1} \approx 13.25.
\end{align*}
Therefore, for any $t \geq 14$, the value of $M_t$ increases by at most one unit per time step.

We can conclude that, for any $t \in [T]$ such that $t - M_t \geq 1$ (in other words $t \geq T_c$), the value of $M_t$ changes by at most one unit per time step, since in that case $t \geq 69 \geq 14$.

\begin{fact}\label{appendix:fact:unitincreases}
For any $t \geq T_c$, the value of $M_t$ increases by at most one unit per round, namely, $M_{t} \leq M_{t-1}+1, \forall t \geq T_c$.
\end{fact}

Notice that the above fact implies that for $t \geq T_c$, the value $t - M_t$ is nondecreasing.

\paragraph{Upper and lower bounds on $T_c$.} We would like to compute some non-trivial upper and lower bounds on the value of $T_c$. The lower bound is used in the proof of Lemma \ref{lemma:regret:mixing}, while the upper bound is used in Lemma \ref{lemma:regret:mapping}.

We first compute an upper bound on $T_c$. Recall, that, by definition, $T_c$ is the smallest positive integer such that $T_c - \lfloor2\log_{c_1}(T_c)\rfloor  \geq 2d_{\max} + 9$. Therefore, for $T_c-1$ it has to be the case that
\begin{align*}
    T_c-1 \leq 2d_{\max} + 8 + \lfloor2\log_{c_1}(T_c-1)\rfloor \leq 2d_{\max} + 8 + 2\log_{c_1}(T_c)
\end{align*}

We can get an upper bound to $T_c$, by noticing that for any $t\geq 1$, it is the case that $2\log_{c_1}(t) \leq t/3 + 38$. Using this, we can see that
\begin{align*}
    T_c \leq 2d_{\max} + T_c/3 + 47.
\end{align*}
By the above, we conclude that $T_c \leq \frac{3}{2} \left(2d_{\max} + 47 \right) \leq 3 d_{\max} + 71$.

We are now looking for a lower bound on $T_c$. Since $T_c$ satisfies $T_c - M_{T_c} \geq 1$, by the analysis of the previous paragraph (on the boundedness of $M_t$), it has to be that $T_c \geq 69$. Using that, we have
\begin{align*}
T_c \geq 2d_{\max} + 9 + \lfloor2\log_{c_1}(T_c)\rfloor \geq 2d_{\max} + 9 + \lfloor2\log_{c_1}(69)\rfloor \geq 2d_{\max} + 67.
\end{align*}

\begin{fact}\label{appendix:fact:bounds}
We can bound $T_c$ as $2d_{\max} + 67 \leq T_c \leq 3 d_{\max} + 71$.
\end{fact}

Consider now any $t \geq T_c$ and any $t' \in [t-d_{\max},t-1]$. We have:
\begin{align*}
    t' \geq t - d_{\max} \geq T_c -d_{\max} \geq 2d_{\max} + 67 - d_{\max} \geq d_{\max} + 67,
\end{align*}
where in the second inequality we use Fact \ref{appendix:fact:bounds}. Therefore, since $t' \geq 67 \geq 14$, then by the above paragraph (on the bounded increases of $M_t$), we have that for any $\tau \in [t', t]$, the value of $M_t$ is increased by at most one.

\begin{fact}\label{appendix:fact:smallincreases}
For any $t \geq T_c \geq d_{\max}$ and $t' \in [t-d_{\max},t-1]$, then for any $\tau \in [t', t]$ we have that $M_{\tau} \leq M_{\tau-1} + 1$. This also implies that $t' - M_{t'} \leq \tau - M_{\tau} \leq t- M_{t}$ for any $\tau \in [t',t]$.
\end{fact}

\paragraph{Correctness of delayed exploitation.} Finally, we present one additional property that is proved useful in proving the correctness of the routine $\textsc{compq}(i,t,H_{t-M_t})$ and the overall correctness of our algorithm. Specifically, we would like to prove the following inequality for any $t \in [T]$, $t' \in [t-d_{\max},t-1]$ and $\tau \in [t' - M_{t'}, t'-1]$:
\begin{align*}
 \max\{ \tau - M_{\tau}, 0\} \leq \max\{t' - M_{t'} , 0\} \leq \max\{t - M_{t} , 0\}.
\end{align*}

Consider any fixed $t,t'$ and $\tau$ that satisfy $t' \in [t-d_{\max},t-1]$ and $\tau \in [t' - M_{t'}, t'-1]$. We first notice that if $\tau - M_{\tau} \leq 0$, then we trivially have that $\max\{ \tau - M_{\tau}, 0\} \leq \max\{t' - M_{t'} , 0\}$ and $\max\{ \tau - M_{\tau}, 0\} \leq \max\{t - M_{t} , 0\}$. We focus on the case where $\tau - M_{\tau} \geq 1$. By the above analysis, we can see that for any $\tau$ such that $\tau - M_{\tau} \geq 1$, it has to be the case that $\tau \geq 69 \geq 14$ and, thus, for any time step $\tau'$ in the interval $\tau' \in [\tau, t-1]$ the value of $M_{\tau'}$ increases by at most one unit. This immediately guarantees that $\tau - M_{\tau} \leq t' - M_{t'} \leq t - M_t$. 

Consider now the remaining case, where $\tau - M_{\tau} \leq 0$, thus, $\max\{ \tau - M_{\tau}, 0\} \leq \max\{t' - M_{t'} , 0\}$ and $\max\{ \tau - M_{\tau}, 0\} \leq \max\{t - M_{t} , 0\}$. We still have to verify that $\max\{t' - M_{t'} , 0\} \leq \max\{t - M_{t} , 0\}$. Following the same reasoning, if $t' - M_{t'} \leq 0$, then the inequality is trivially satisfied. On the other hand, if $t' - M_{t'} \geq 1$, then $t' \geq 69 \geq 14$ and, thus, the value of $M_{\tau'}$ for any round $ \tau' \in [t', t - 1]$ can be increased by at most one unit. This suffices to conclude that $t' - M_{t'} \leq t - M_t$.

\begin{fact}\label{appendix:fact:delayedexploit}
For any $t \in [T]$, $t' \in [t-d_{\max},t-1]$ and $\tau \in [t' - M_{t'}, t'-1]$, we have
\begin{align*}
 \max\{ \tau - M_{\tau}, 0\} \leq \max\{t' - M_{t'} , 0\} \leq \max\{t - M_{t} , 0\}.
\end{align*}
\end{fact}

\subsection{Computing the probability $q_{i,t}(H_{t-M_t})$}
\label{appendix:regret:compexante}
In this section, we show that during a run of \ucb, each value of the form $q_{i,t}(H_{t-M_{t}})$, as computed by $\textsc{compq}(i,t,H_{t-M_t})$ (Algorithm \ref{alg:exante}), is equal to the probability of arm $i$ being available at round $t$, conditioned on $H_{t-M_t}$, that is, $q_{i,t}(H_{t-M_t}) = \Pro{F^{\pit}_{i,t} |~H_{t-M_t}}, \forall i \in \A, t \in [T]$ (assuming that $t - M_t \geq 1$). For any integer $t$, we define $[t]^+:= \max\{t,0\}$. Therefore, for any $t \in [T]$ such that $t - M_{t} \leq 0$, we have $H_{[t-M_{t}]^+} = H_0$ (recall that, in this case, \ucb samples arms according to an initial extreme point solution $Z(0) = \{z_{i,j}(0)\}_{i,j}$). In the following, we fix any arm $i \in \A$ and any point in time $t \in [T]$. Recall that $T_c$ is defined as the smallest $t \in [T]$ such that $t - M_t \geq 1$. 

We first consider the case where $t < T_c$ (and, thus, $t - M_t \leq 0$ and $H_{[t-M_t]^+} = H_0$). In that case, for every round $\tau \in [t]$, the algorithm uses the initially computed extreme point $Z(0)$ in order to sample arms. Following the same reasoning as used in Lemma \ref{lemma:oracle:exante} for the full-information case of our problem, we can see that $q_{i,t}(H_0)$ (and, thus, the conditional probability $\Pro{F^{\pit}_{i,t} |~H_0}$) can be computed by the following recursive formula: We set $q_{i,1}(H_0) = 1$ and  
\begin{align*}
q_{i,t'+1}(H_0) &= q_{i,t'}(H_0)\left(1 - \beta_{i,t'}\sum_{j \in \C} z_{i,j}(0)\right) \\
&            + \event{t'\geq d_i } q_{i,t'-d_i+1}(H_0) \beta_{i,t'-d_i+1} \sum_{j \in \C} z_{i,j}(0),
\end{align*}
where each $\beta_{i,\tau}$ is by construction equal to $\min\{1, \frac{d_i}{2d_i-1} \frac{1}{q_{i,t'}(H_{[\tau-M_{\tau}]^+})}\}$.

It is not hard to verify that in the above recursive formula, $q_{i,t}(H_0) = q_{i,t}(H_{[t-M_t]^+})$ for $t < T_c$ is indeed equal to $\Pro{F^{\pit}_{i,t} |~H_{[t-M_t]^+}}$, and that $\textsc{compq}(i,t,H_{t-M_t})$ computes exactly this value. The correctness of this computation follows by the fact that for any $t'\leq t$, we have that $q_{i,t'}(H_{[t'-M_{t'}]^+}) = q_{i,t'}(H_{[t-M_t]^+})$, since for all rounds $t' \leq t < T_c$, we have $H_{[t' - M_{t'}]^+} = H_{[t - M_{t}]^+} = H_0$. Therefore, all the non-skipping probabilities $\beta_{i,t'}$ for $t' < t$ are deterministic and, thus, computable, conditioned on $H_{[t - M_{t}]^+}$ (thus, the algorithm can simulate them recursively at time $t$).

We now consider the case, where $t \geq T_c$ (and, thus, $t - M_t \geq 1)$). In this case, the algorithm uses the extreme point $Z(t - M_t)$ for sampling arms. Recall that the skipping probability of each round $t'$, is defined given the value of $q_{i,t'}$, as computed, conditioned on $H_{[t' - M_{t'}]^+}$, namely, $\beta_{i,t'} = \min\{1, \frac{d_i}{2d_i-1} \frac{1}{q_{i,t'}(H_{[t'-M_{t'}]^+})}\}$. Therefore, for being able to compute (i.e., simulate) $\beta_{i,t'}$, while being at some round $t > t'$, it suffices to show that $[t' - M_{t'}]^+ \leq [t - M_t]^+$.

In the case where $t' < T_c \leq t$, the extreme point solution used for sampling arms at time $t'$ is $Z(0)$ and, thus, is computable conditioned on $H_{t-M_{t}}$. The same holds for the non-skipping probability, $\beta_{i,t'}$, used at time $t'$. On the other hand, consider the case where $T_c \leq t' \leq t$. By the analysis in Appendix \ref{appendix:criticaltime} (see Fact \ref{appendix:fact:unitincreases}), since $t' \geq T_c$, we know that for any $\tau$ in the interval $\tau \in [t' , t]$, the value of $M_{\tau}$ can increase by at most one unit per round, namely, $M_{\tau + 1} \leq M_{\tau} +1$. By using this argument, we can directly show by induction, that $t' - M_{t'} \leq t - M_{t}$ and, thus, $H_{t'-M_{t'}} \subseteq H_{t-M_{t}}$. Therefore, both the extreme point $Z(t' - M_{t'})$ and the non-skipping probability $\beta_{i,t'}$ used at time $t'$ can be computed (recursively) by the algorithm at time $t$.

The above discussion leads to the following recursive computation of $q_{i,t}(H_{t-M_t})$ for any arm $i$ and time $t \geq T_c$. Let $\nu(i,t-M_t,H_{t-M_t})$ be the first time $\tau \geq t-M_t$ that arm $i$ is deterministically available, conditioned on the history $H_{t-M_t}$. We set $q_{i,\nu(i, t - M_t,H_{t-M_t})}(H_{t-M_t}) = 1$ and for any $t' \geq \nu(i, t - M_t, H_{t-M_t})$, we set
\begin{align*}
&q_{i,t'+1}(H_{t-M_t}) = q_{i,t'}(H_{t-M_t})\left(1 - \beta_{i,t'}\sum_{j \in \C} z_{i,j}\left([t'-M_t']^+\right)\right) \\
&+ \event{t' - d_i +1\geq \nu(i, t - M_t, H_{t-M_t}) } q_{i,t'-d_i+1}(H_{t-M_t}) \beta_{i,t'-d_i+1} \sum_{j \in \C} z_{i,j}([t'-d_i+1 - M_{t'-d_i+1}]^+).
\end{align*}
It is easy to verify that $\textsc{compq}(i,t,H_{t-M_t})$ produces exactly the same result as the above recursive formula for $t \geq T_c$.

Given the above analysis, we have now established the correctness of $\textsc{compq}(i,t,H_{t-M_t})$. We remark that in the pseudocode provided in Algorithm \ref{alg:exante}, the recursive computation of the non-skipping probabilities is implemented efficiently by caching and reusing past values.

\subsection{Proof of Lemma \ref{lemma:regret:mixing} \label{appendix:lemma:regret:mixing}}

\restateLemmaMixing*

\begin{proof}
Recall from Section~\ref{appendix:regret:compexante} that for any fixed arm $i$, the quantity $q_{i,s}(H_{t- M_t})$ in Algorithm~\ref{alg:exante} equals $\Pro{F^{\pit}_{i,s} | H_{t-M_{t}}}$ for $t-d_i \leq s \leq t$. Therefore, we are interested in the ratio 
$\tfrac{q_{i,t'}(H_{t- M_t}) }{ q_{i,t'}(H_{t'- M_{t'}})}$. In the rest of this proof and for simplicity of notation, we assume that $H_{\tau} = H_{0}$ and $\{z_{i,j}(\tau)\}_{i,j} = \{z_{i,j}(0)\}_{i,j}$, for any $\tau \leq 0$.

Let us fix any run of the \ucb algorithm upto time $t$ as $h_t$. The sequence of random variables $\{Z(\tau), \beta_{i,\tau}: 1 \leq \tau \leq t-M_t\}$ is computable at time $t$ given the history $H_{t-M_t}$ (see Fact~\ref{appendix:fact:delayedexploit} in Appendix~\ref{appendix:criticaltime}). Therefore, fixing a run of the \ucb algorithm up to any time $t$ (in terms of sampling and non-skipping probabilities), corresponds to fixing $H_{t-M_t} = h_{t-M_t}$, which, in turn, fixes the sequence $\{Z(\tau -M_{\tau}), \beta_{i,\tau}: 1 \leq \tau \leq t\}$. This follows from the computability of $\beta_\tau$ as discussed in Appendix~\ref{appendix:regret:compexante}.

For a particular run $h_{t-M_t}$ upto time $t-M_t$ and a specific arm $i$, the computation of $q_{i,\tau}(H_{t- M_t})$ for any $\tau \leq t$ corresponds to simulating a specific Markov chain as detailed next. We consider the time-nonhomogeneous Markov transition probability matrices (TPM)
$\mathcal{M} = \{\mathbf{P}_{\tau}: 1\leq \tau \leq t\}$, that at any time $\tau \leq t$ makes transitions as follows. If it is in state $0$ it moves to state $d_i$ w.p. $\beta_{i,\tau} \sum_{j\in \C} z_{i,j}([\tau - M_{\tau}]^+)$, otherwise it stays in state $0$. In the case the Markov chain is in state $d>0$, then it moves to state $(d-1)$ w.p. $1$. Here, we denote the TPM at time $\tau$ as $\textbf{P}_\tau$.

Let us also denote the first time on or after time $\tau$ where the arm $i$ becomes available as $\nu(\tau)$ (which is fixed for a run $h_t$, as it is computable using $H_t$). Using this definition, we denote by $\mathcal{X}_{t'}$ (resp. $\mathcal{X}_{t}$) the Markov chain that lies in state $0$ (w.p. 1) at time $\nu(t'- M_{t'})$ (resp. $\nu(t-M_t)$), and moves following the TPM $\mathcal{M}$. We emphasize the fact that both $\mathcal{X}_{t}$ and $\mathcal{X}_{t'}$ have the same transition probabilities for all time steps between $\max\{\nu(t-M_t), \nu(t'- M_{t'})\}$ and $t'$ (see Fact \ref{appendix:fact:delayedexploit} in Appendix \ref{appendix:criticaltime}).

We claim that the probability that the Markov chain $\mathcal{X}_{t'}$ is in state $0$ at time $t'$ equals $q_{i,t'}(H_{t'- M_{t'}})$, namely, $\Pro{\mathcal{X}_{t'}(t') = 0} = q_{i,t'}(H_{t'- M_{t'}})$. This follows induction on $\tau$ for the statement $$\Pro{\mathcal{X}_{t'}(\tau) = 0} = q_{i,\tau},$$ where $q_{i,\tau}$ is as given in Algorithm~\ref{alg:exante}. 
As the base case, at $t_0= \nu(t'- M_{t'})$ we have by construction that $\Pro{\mathcal{X}_{t'}(t_0)=0} = q_{i,t_0} = 1.$ Let us assume that the argument is true for all time up to $\tau$. Then we have,
\begin{align*}
&\Pro{\mathcal{X}_{t'}(\tau+1)=0}\\
&= \Pro{\mathcal{X}_{t'}(\tau)=0}\left(1 - \beta_{i,\tau}\sum_{j \in \C} z_{i,j}(\tau-M_{\tau})\right)\\
&+ \event{\tau -d_i +1\geq t_0 } \Pro{\mathcal{X}_{t'}(\tau -d_i +1)=0} \beta_{i,\tau -d_i +1} \sum_{j \in \C} z_{i,j}(\tau -d_i +1-M_{\tau -d_i +1}).\\
&= q_{i,\tau}\left(1 - \beta_{i,\tau}\sum_{j \in \C} z_{i,j}(\tau-M_{\tau})\right)\\
&+ \event{\tau -d_i +1\geq t_0 } q_{i, \tau - d_i+1} \beta_{i,\tau -d_i +1} \sum_{j \in \C} z_{i,j}(\tau -d_i +1-M_{\tau -d_i +1})\\
&= q_{i,\tau+1},
\end{align*}
which proves our claim. Using similar arguments we have that $\Pro{\mathcal{X}_{t}(t') = 0} = q_{i,t'}(H_{t- M_{t}})$. 

The rest of the proof relies on showing that $\Pro{\mathcal{X}_{t}(t') = 0}\approx \Pro{\mathcal{X}_{t}(t') = 0}$ for large enough time (specifically, for time $\min\{t' - \nu(t'- M_{t'}), t'- \nu(t- M_{t})\}$). We accomplish that by the use of a Doeblin type coupling argument for  the two Markov chains $\mathcal{X}_t$ and $\mathcal{X}_{t'}$.

\paragraph{Doeblin Coupling of two Markov chains.}
The argument of the rest of the proof relies on a Doeblin type coupling of the above two MCs. Let $\mathcal{X}_t(\tau)$ and $\mathcal{X}_{t'}(\tau)$ be the states of the MC $\mathcal{X}_t$ and $\mathcal{X}_{t'}$ at time $\tau$, respectively. Recall that $\mathcal{X}_t$ starts from state $0$ at time $\nu(t-M_t)$, and $\mathcal{X}_{t'}$ starts from state $0$ at time  $\nu(t'-M_{t'})$. Given the fact that the transition functions are common in both MCs, the two chains evolve independently up until the point they meet for the first moment. Afterwards, they get coupled and evolve together.

We consider the evolution of the bi-variate Markov chain 
$\{(\tilde{\mathcal{X}}_t(\tau),\tilde{\mathcal{X}}_{t'}(\tau))\}$, where, for $\tau \geq \nu_{\max}:=\max\{ \nu(t-M_t), \nu(t'-M_{t'})\}$, we have the following evolution of the two Markov chains,
\begin{align*}
    &\mathbb{P}\left(\tilde{\mathcal{X}}_{t}(\tau+1)= s_1, \tilde{\mathcal{X}}_{t'}(\tau+1)= s_2~|~
    \tilde{\mathcal{X}}_t(\tau)= s'_1, \tilde{\mathcal{X}}_{t'}(\tau)= s'_2\right) \\
    &~~~~~~~~~~~~~~~~~~~~~~~~~~~~~~~~~~~~~~~~~~~~~~~~~~= \begin{cases}
    \mathbf{P}_{\tau}(s'_1, s_1) \mathbf{P}_{\tau}(s'_2, s_2), &\text{ if } s'_1 \neq s'_2,\\
    \mathbf{P}_{\tau}(s'_1, s_1), &\text{ if } s'_1 = s'_2 \wedge s_1 = s_2,\\
    0, &\text{ otherwise. }
    \end{cases}
\end{align*}

It is easy to check that the bi-variate MC has the property $\tilde{\mathcal{X}}_t(\tau)\stackrel{d}{=} \mathcal{X}_t(\tau)$ and $\tilde{\mathcal{X}}_{t'}(\tau)\stackrel{d}{=} \mathcal{X}_{t'}(\tau)$ for all integers $\tau \geq \nu_{\max}$ (here, $\stackrel{d}{=}$ indicates equality in distribution). 

Let the random variable $R_c = \inf\{{r \geq \nu_{\max}}~|~\mathcal{X}_t(\tau) = \mathcal{X}_{t'}(\tau)\}$ denote the first time after $\nu_{\max}$, when the two chains $\mathcal{X}_t$ and $\mathcal{X}_{t'}$ become coupled. From standard arguments in Doeblin coupling~\cite{L02}, we have
$
    |\Pro{\mathcal{X}_{t'}(t') = 0} - \Pro{\mathcal{X}_{t}(t') = 0}| \leq  \Pro{\mathcal{X}_{t'}(t') \neq \mathcal{X}_t(t')}
    \leq \Pro{R_c > t'}.
$

We now make a claim that under the Markov TPM $\mathcal{M}$ at any time $\tau \geq \nu(t-M_{t})$ we have $\Pro{\mathcal{X}_{t}(\tau) = 0} \geq \frac{1}{e}$. The claim follows by noticing that arm $i$ is sampled by \ucb with probability at most $1/d_i$ at each time, and it is available, if not sampled in the last $(d_i-1)$ time slots. Formally, for all $\tau\geq \nu(t-M_{t})$ we consider the event 
$E = \{\mathcal{X}_{t}(\tau') \neq 0, \forall \tau' \in [\tau - d_i, \tau-1]\}$, and derive the following
\begin{align*}
    &\Pro{\mathcal{X}_{t}(\tau) = 0} 
    = \Pro{E \wedge \mathcal{X}_{t}(\tau) = 0} + \Pro{E^c \wedge \mathcal{X}_{t}(\tau) = 0}\\
    &\stackrel{(i)}{=}\Pro{E} + \sum_{\tau'=\tau-d_i}^{\tau-1}\Pro{\mathcal{X}_{t}(\tau') = 0}\Pro{\mathcal{X}_{t}(\tau) = 0 | \mathcal{X}_{t}(\tau') = 0}\\
    &\stackrel{(ii)}{\geq}\Pro{E} + \sum_{\tau'=\tau-d_i}^{\tau-1}\Pro{\mathcal{X}_{t}(\tau') = 0} \prod_{\tau''\in [\tau', \tau-1]} \left(1- \beta_{\tau''}\sum_{j\in \C} z_{i,j}(\tau'' - M_{\tau''})\right)\\
    &\stackrel{(iii)}{\geq}\Pro{E} + \left(\sum_{\tau'=\tau-d_i}^{\tau-1}\Pro{\mathcal{X}_{t}(\tau') = 0}\right) \prod_{\tau''\in [\tau-d_i, \tau-1]} \left(1- \beta_{\tau''}\sum_{j\in \C} z_{i,j}(\tau'' - M_{\tau''})\right)\\
    &\stackrel{(iv)}{\geq} \Pro{E} + (1-\Pro{E}) (1- 1/d_i)^{(d_i-1)}
    \stackrel{(v)}{\geq} \frac{1}{e}. 
\end{align*}
In the equality (i), we use if the $i$-th arm is unavailable for a contiguous stretch of length $d_i$ before $\tau$ (given by event $E$) then it will be available on $\tau$. Also, we break $E^c$ into mutually exclusive events.  The inequality (ii) uses the events that the MC stays in state  $0$ from time $\tau''=\tau'$ to $\tau$ to lower bound the probabilities.  In inequality (iii) we further lower bound these probabilities by replacing $\tau'$ with $\tau-d_i$.
For inequality (iv) we use $\beta_{\tau'}\sum_{j\in \C} z_{i,j}(\tau' - M_{\tau'}) \leq 1/d_i$ due to the LP constraint~\eqref{flp:window}, and the fact that $\beta_{\tau'} \leq 1$. Also, $\sum_{\tau'=\tau-d_i}^{\tau-1}\Pro{\mathcal{X}_{t}(\tau') = 0}\Pro{E^c}$. Finally, in (v) we minimize over $\Pro{E}$ and $d_i$ to obtain the bound.   

Similar results hold for the MC $\mathcal{X}_{t'}$. Thus, we obtain that for any time $\tau \geq \nu_{\max}$, we have $\min\{\Pro{\mathcal{X}_{t}(\tau) = 0}, \Pro{\mathcal{X}_{t'}(\tau) = 0}\} \geq 1/e$.

Therefore, at each time $\tau \geq \nu_{\max}$, we know that the two chains get coupled with probability at least $\tfrac{1}{e^2}$. Formally,
\begin{align*}
    \Pro{R_c \geq t'+1} &\stackrel{(i)}{=} \Pro{\mathcal{X}_t(t') \neq \mathcal{X}_{t'}(t')~|~R_c \geq t'} \Pro{R_c \geq t'}\\
    &\stackrel{(ii)}{=} \left(1 - \Pro{\mathcal{X}_t(t') = \mathcal{X}_{t'}(t')~|~R_c \geq t'} \right) \Pro{R_c \geq t'} \nonumber\\
    &\stackrel{(iii)}{\leq} \left(1 - \Pro{\mathcal{X}_t(t') = \mathcal{X}_{t'}(t')=0~|~R_c \geq t'} \right) \Pro{R_c \geq t'}\\
    &\stackrel{(iv)}{\leq} \left(1 - \Pro{\mathcal{X}_t(t') = 0~|~R_c \geq t'} \Pro{\mathcal{X}_{t'}(t')=0~|~R_c \geq t'} \right) \Pro{R_c \geq t'} \\
    &\stackrel{(v)}{\leq} \left(1 - \frac{1}{e^2} \event{t' \geq \nu_{\max}} \right) \Pro{R_c \geq t'},
\end{align*}
where (i) follows by definition of coupling, (iii) follows by the fact that $\{\mathcal{X}_t(t') = \mathcal{X}_{t'}(t') = 0\} \subseteq \{\mathcal{X}_t(t') = \mathcal{X}_{t'}(t')\}$ and (iv) follows by the fact that the two MCs evolve independently before round $R_c$. Finally, (v) follows by the fact that the probability of $\mathcal{X}_t$ (resp. $\mathcal{X}_{t'}$) being at state $0$ is at least $\frac{1}{e}$, for any time $\tau \geq \nu_{\max}$ as shown above.

By repeating the arguments leading to (v) until we reach the event $\{R_c \geq \nu_{\max}-1\}$ we have
\begin{align*}
    \Pro{R_c \geq t'+1} &\leq \left(1 - \frac{1}{e^2} \event{t'\geq \nu_{\max}} \right) \Pro{R_c \geq t'}\\
    &\leq \Pro{R_c \geq \nu_{\max}-1} \prod^{t'}_{\tau = \nu_{\max}}\left(1 - \frac{1}{e^2}\right) \\
    &\stackrel{(vi)}{\leq}  \left(1 - \frac{1}{e^2}\right)^{t' - \nu_{\max}+1}\\
    &\stackrel{(vii)}{\leq} \left(1 - \frac{1}{e^2}\right)^{M_{t}-2d_i}, 
\end{align*}
where in (vi), we use the fact that $\Pro{R_c \geq 2d_i} \leq 1$. In (vii) we use the following derivations 
\begin{align*}
t' - \max\{\nu(t- M_t), \nu(t' - M_{t'}) \} 
&\stackrel{(a)}{=} t' - \nu(t- M_t) \\
&\stackrel{(b)}{\geq} M_t +t' - t - d_i +1. \\
&\stackrel{(c)}{\geq} M_t  - 2d_i +2.
\end{align*}
The equality (a) in the above derivation holds since for $t \geq T_c$ and $t' \in [t-d_i+1,t-1]$, then by Fact~\ref{appendix:fact:smallincreases} in Appendix~\ref{appendix:criticaltime}, it has to be that $t' - M_{t'} \leq t - M_t$ and, thus, $\nu(t'-M_{t'}) \leq \nu(t-M_t)$. Inequality (b) holds since $i$ becomes deterministically available in at most $d_i-1$ time steps after $t-M_t$, i.e. $\nu(t-M_t) \leq t-M_t +d_i -1$. The last inequality (c) holds as $t-t' \leq d_i - 1$.

Therefore, for concluding the proof of the lemma, we have:
\begin{align*}
    &\frac{\Pro{F^{\pit}_{i,t'}~|~H_{t-M_t}}}{\Pro{F^{\pit}_{i,t'}~|~H_{t'-M_{t'}}}} =\frac{q_{i,t'}(H_{t- M_t}) }{ q_{i,t'}(H_{t'- M_{t'}})}
    \leq 1 + \left|1- \frac{q_{i,t'}(H_{t- M_t})}{ q_{i,t'}(H_{t'- M_{t'}})}\right|
    \leq 1 + \left|1- \frac{\Pro{\mathcal{X}_{t}(t') = 0}}{ \Pro{\mathcal{X}_{t'}(t') = 0}}\right|\\
    &\leq 1 + 
    \frac{\Pro{R_c > t'}} {\Pro{\mathcal{X}_{t}(\tau) = 0}}
    \leq 1 + \frac{\left(1 - \frac{1}{e^2} \right)^{M_t-2d_i}}{\frac{1}{e}}
    \leq 1 + e \left(\frac{e^2}{e^2-1} \right)^{2 d_{\max}}\left(\frac{e^2}{e^2-1} \right)^{-M_t}.
\end{align*}
The above results follow by use of triangle inequality and substituting the bounds derived so far.
\end{proof}

\subsection{Proof of Lemma \ref{lemma:regret:mapping}}

\restateLemmaMapping*
\begin{proof} 

In the following proof, we start from the definition of $\alpha$-regret and we prove the regret upper bound of the statement, by applying a sequence of transformations: First, we incorporate the $\left(1 - \frac{d_{\max}-1}{d_{\max}-1+T}\right)$-multiplicative loss, due to the use of \eqref{lp:LP} as an upper bound, into an $\mathcal{O}(d_{\max})$ additive term in the regret. Second, we upper bound the total regret due to the rounds such that $t - M_t \leq 0$, by another $\mathcal{O}(d_{\max})$ term in the regret. Then, focusing on each round such that $t \geq M_t$, we apply Lemma \ref{lemma:regret:mixing} in order to (approximately) express the regret of any such round by $\frac{d_{i}}{2d_{i}-1}\left(z^*_{i,j} - z_{i,j}(t-M_{t}) \right)$, for any $i \in \A$ and $j \in \C$. We show that the total approximation loss for that case can be transformed into a constant additive loss in the regret. Finally, we notice that in the rounds, such that $t \geq M_t$, where $M_t$ is increased (by one unit as we show in Appendix \ref{appendix:criticaltime}), the arm sampling is performed using the same extreme point solution as in the previous rounds. By observing that this can happen at most $\mathcal{O}(\log(T))$ times, we separate the rounds that use strictly updated UCB estimates, while we incorporate the rest as an $\mathcal{O}(\log(T) \Delta_{\max})$-additive loss in the regret bound.

In the following, we denote by $S^{\pit}_{i,t}$ the event that \ucb samples arm $i \in \A$ at round $t$ and by $B^{\pit}_{i,t}$ the event that arm $i$ is not skipped at the round. Finally, we denote by $F^{\pit}_{i,t}$ the event that arm $i$ is available at round $t$.

\paragraph{Incorporating time-dependent approximation loss.} The first step in proving the bound is to incorporate the $\left(1 - \frac{d_{\max}-1}{d_{\max}-1 + T}\right)$-multiplicative loss, due to the use of \eqref{lp:LP}, into the regret. By definition of $\alpha$-regret, we have
\begin{align*}
    \alpha\Reg^{\pit}_I(T) &= \alpha\Rew^*_I(T) - \Rew^{\pit}_I(T) \\
    &= \frac{d_{\max}}{2d_{\max}-1} \left(1- \frac{d_{\max}-1}{d_{\max}-1+T} + \frac{d_{\max}-1}{d_{\max}-1+T}\right) \Rew^*_I(T) - \Rew^{\pit}_I(T) \\
    &\leq \frac{d_{\max}}{2d_{\max}-1} \left(1- \frac{d_{\max}-1}{d_{\max}-1+T}\right) \Rew^*_I(T) - \Rew^{\pit}_I(T) + \frac{2}{3} \left(d_{\max} -1\right),
\end{align*}
where in the last inequality, we use the fact that 
$$
\frac{d_{\max}}{2d_{\max}-1} \frac{d_{\max}-1}{d_{\max}-1+T} \Rew^*_I(T) \leq \frac{d_{\max}}{2d_{\max}-1} \frac{d_{\max}-1}{T} \Rew^*_I(T) \leq  \frac{d_{\max}}{2d_{\max}-1} \left(d_{\max}-1\right), 
$$
using that $\Rew^*_I(T) \leq T$ and the fact that for any possible $d_{\max}$, we have $\frac{d_{\max}}{2d_{\max}-1} \left(d_{\max}-1\right) \leq \frac{2}{3} \left(d_{\max}-1\right)$.

Now by applying the result of Theorem \ref{online:theorem:competitive}, we can further upper bound the $\alpha$-regret by using the fact that the algorithm \oracle produces, in expectation, a constant rate of regret over time. More specifically, by denoting $\Rew^{LP}_I$ the optimal solution to \eqref{lp:LP}, we have
\begin{align}
    \alpha\Reg^{\pit}_I(T) &\leq \frac{d_{\max}}{2d_{\max}-1} \left(1- \frac{d_{\max}-1}{d_{\max}-1+T}\right) \Rew^*_I(T) - \Rew^{\pit}_I(T) + \frac{2}{3} \left(d_{\max} -1\right) \nonumber\\
    &\leq \frac{d_{\max}}{2d_{\max}-1} T \cdot \Rew^{LP}_I - \Rew^{\pit}_I(T) + \frac{2}{3} \left(d_{\max} -1\right) \label{eq:multloss:1}\\
    &\leq \sum_{t \in [T]} \sum_{i \in \A} \sum_{j \in \C} \mu_{i,j} \frac{d_i}{2d_i - 1} z^*_{i,j} - \Rew^{\pit}_I(T) + \frac{2}{3} \left(d_{\max} -1\right), \label{eq:mapping:competloss}
\end{align}
where \eqref{eq:multloss:1} follows by Lemma \ref{lemma:overview:upperbound} and \eqref{eq:mapping:competloss} by the fact that $\frac{d_{\max}}{2d_{\max}-1} \leq \frac{d_i}{2d_i-1}$ for any $i \in \A$.

\paragraph{Simplifying the expected reward of \ucb.}
By the independence of the rewards $\{X_{i,j,t}\}_{\forall i,j,t}$, we have:
\begin{align*}
&\Ex{\mathcal{R}_{N,\pit}}{\sum_{t \in [T]} \sum_{i \in \A} \sum_{j \in \C} X_{i,j,t} \event{A^{\pit}_t=i , C_t =j}}\\
&= \sum_{t \in [T]} \sum_{i \in \A} \sum_{j \in \C} \Ex{\mathcal{R}_{N,\pit}}{X_{i,j,t} \event{A^{\pit}_t=i , C_t =j}} \\
&= \sum_{t \in [T]} \sum_{i \in \A} \sum_{j \in \C} \Ex{\mathcal{R}_{N,\pit}}{\Ex{ }{X_{i,j,t} \event{A^{\pit}_t=i , C_t =j} \bigg| A^{\pit}_t , C_t}} \\
&= \sum_{t \in [T]} \sum_{i \in \A} \sum_{j \in \C} \Ex{\mathcal{R}_{N,\pit}}{\Ex{ }{X_{i,j,t} \bigg| A^{\pit}_t , C_t}\event{A^{\pit}_t=i , C_t =j}} \\
&= \sum_{t \in [T]} \sum_{i \in \A} \sum_{j \in \C} \Ex{\mathcal{R}_{N,\pit}}{\mu_{i,j}\event{A^{\pit}_t=i , C_t =j}} \\
&= \Ex{\mathcal{R}_{N,\pit}}{\sum_{t \in [T]} \sum_{i \in \A} \sum_{j \in \C} \mu_{i,j} \event{A^{\pit}_t=i , C_t =j}}.
\end{align*}

\paragraph{Using delayed exploitation for large enough $t$.} The remainder of this proof is dedicated to bounding the difference between the expected reward collected by \oracle and \ucb. More specifically, our goal is to directly associate the loss of any round $t$ with the suboptimality of the extreme point solution of \eqref{lp:LP} computed by \ucb at the same round. 
%For the rest of this proof, we denote by $S^{\pit}_{i,t}$ the event that an arm $i \in \A$ is sampled by \ucb at some round $t$ and by $B^{\pit}_{i,t}$ the event that arm $i$ is not skipped at round $t$. 
More specifically, we are interested in upper bounding the term
\begin{align*}
   \sum_{t \in [T]} \sum_{i \in \A} \sum_{j \in \C} \mu_{i,j} \frac{d_i}{2d_i - 1} z^*_{i,j} - \Rew^{\pit}_I(T).
\end{align*}

The first step is to lower bound the expected reward of \ucb, namely,
\begin{align*}
\Rew^{\pit}_I(T) &= \Ex{\mathcal{R}_{N,\pit}}{\sum_{t \in [T]} \sum_{i \in \A} \sum_{j \in \C} X_{i,j,t} \event{A^{\pit}_t=i , C_t =j}}\\
&= \Ex{\mathcal{R}_{N,\pit}}{\sum_{t \in [T]} \sum_{i \in \A} \sum_{j \in \C} \mu_{i,j} \event{A^{\pit}_t=i , C_t =j}}.
\end{align*}

Let $T_c$ be the minimum round such that $T_c \geq M_{T_c} + 1$. By the discussion in Appendix \ref{appendix:criticaltime}, we know that $t \geq M_t + 1 \geq 2d_{\max}$ for any $t \geq T_c$.

We now fix any round $t \in [T]$ such that $t \geq T_c$. By using linearity of expectation, we can further simplify the expression of the expected reward of \ucb, by conditioning on the history up to time $t-M_t$. For any fixed $i \in \A$ and $j \in \C$ we have:
\begin{align}
&\Ex{\mathcal{R}_{N,\pit}}{{\event{A^{\pit}_t=i , C_t =j}}}\nonumber\\
&= \Ex{\mathcal{R}_{N,\pit}}{\Ex{ }{\event{A^{\pit}_t=i , C_t =j}\bigg| H_{t-M_t}}} \nonumber\\
&= \Ex{\mathcal{R}_{N,\pit}}{\Ex{ }{\event{S^{\pit}_{i,t}, B^{\pit}_{i,t}, F^{\pit}_{i,t} , C_t =j}\bigg| H_{t-M_t}}}\nonumber\\
&= \Ex{\mathcal{R}_{N,\pit}}{\Pro{S^{\pit}_{i,t}, B^{\pit}_{i,t}, F^{\pit}_{i,t} , C_t =j | H_{t-M_t} }} \nonumber\\
&= \Ex{\mathcal{R}_{N,\pit}}{\Pro{S^{\pit}_{i,t}, C_t =j| H_{t-M_t}}\Pro{B^{\pit}_{i,t}| H_{t-M_t}}\Pro{F^{\pit}_{i,t} | H_{t-M_t}}} \label{eq:ucb:mapping:2}\\
&= \Ex{\mathcal{R}_{N,\pit}}{\Pro{S^{\pit}_{i,t}| H_{t-M_t}, C_t =j} \Pro{C_t =j| H_{t-M_t}} \beta_{i,t}\Pro{F^{\pit}_{i,t} | H_{t-M_t}}} \label{eq:ucb:mapping:3}\\
&= \Ex{\mathcal{R}_{N,\pit}}{\frac{z_{i,j}(t-M_t)}{f_j} f_j \beta_{i,t}\Pro{F^{\pit}_{i,t} | H_{t-M_t}}}\nonumber\\
&= \Ex{\mathcal{R}_{N,\pit}}{z_{i,j}(t-M_t) \beta_{i,t}\Pro{F^{\pit}_{i,t} | H_{t-M_t}}},\nonumber
\end{align}
where in \eqref{eq:ucb:mapping:2}, we use the fact that the events $S^{\pit}_{i,t}, B^{\pit}_{i,t}$ and $F^{\pit}_{i,t}$ are independent conditioned on $H_{t-M_t}$. The reason is that the outcome of $S^{\pit}_{i,t}$ depends on the observed context and on the UCB indices computed before time $t-M_t$, while the outcome of the event $B^{\pit}_{i,t}$ has probability $\beta_{i,t}$, which is computable using only information from $H_{t-M_t}$. Finally, in \eqref{eq:ucb:mapping:3}, we use the fact that the observed context of round $t$ is independent of $H_{t-M_t}$, $B^{\pit}_{i,t}$ and $F^{\pit}_{i,t}$.

Clearly, by observing the history $H_{t-M_t}$, one can easily compute the first time arm $i \in \A$ becomes available after time $t - M_t$. If the arm is available at time $t-M_t$ and is not played, then we know that $\Pro{F^{\pit}_{i,t-M_t+1} | H_{t-M_t}} = 1$, while if the arm is blocked at time $t-M_t$, then it is played at some time $t' < t-M_t$ and, thus, $\Pro{F^{\pit}_{i,t'+d_i} | H_{t-M_t}} = 1$. The conditional probabilities of an arm being available, that is, $q_{i,t}(H_{t-M_t}) = \Pro{F^{\pit}_{i,t}~|~H_{t-M_t}}$ can be computed by Algorithm \ref{alg:exante}, as described in Appendix \ref{appendix:regret:compexante}. In short, given the fact that the algorithm uses at any round $t \geq T_c$ the extreme point computed in round $t-M_t$, for any $t' \in [t-M_t,t]$, the extreme points used are computable given $H_{t-M_t}$ and the algorithm can efficiently simulate any possible $\beta_{i,t'}$.

By the above analysis it follows that at any time $t \geq T_c$, we have:
\begin{align*}
    \beta_{i,t} = \min\bigg\{1 , \frac{d_i}{2d_i-1}\frac{1}{q_{i,t}(H_{t-M_t})}\bigg\} = \min\bigg\{1 , \frac{d_i}{2d_i-1}\frac{1}{\Pro{F^{\pit}_{i,t}|H_{t-M_t}}}\bigg\}.
\end{align*}

Similarly to the proof of Lemma \ref{lemma:oracle:exact}, we distinguish between two cases on the value of $\beta_{i,t}$ conditioned on $H_{t-M_t}$:

\textbf{Case (a)} In the case where $1 > \frac{d_i}{2d_i-1} \frac{1}{\Pro{F^{\pit}_{i,t}|H_{t-M_t}}}$, we immediately get that:
\begin{align*}
&\Ex{\mathcal{R}_{N,\pit}}{z_{i,j}(t-M_t)\beta_{i,t} \Pro{F^{\pit}_{i,t}|H_{t-M_t}}} \\
&=  \Ex{\mathcal{R}_{N,\pit}}{z_{i,j}(t-M_t)\frac{d_i}{2d_i-1} \frac{1}{\Pro{F^{\pit}_{i,t}|H_{t-M_t}}}\Pro{F^{\pit}_{i,t}|H_{t-M_t}}}\\
&= \Ex{\mathcal{R}_{N,\pit}}{\frac{d_i}{2d_i-1}z_{i,j}(t-M_t)}.
\end{align*}

\textbf{Case (b)} In the case where $1 \leq \frac{d_i}{2d_i-1} \frac{1}{\Pro{F^{\pit}_{i,t}|H_{t-M_t}}}$, we directly get that $\Pro{F^{\pit}_{i,t}|H_{t-M_t}} \leq \frac{d_i}{2d_i-1}$ and $\beta_{i,t} = 1$. In order to get a lower bound on $\Pro{F^{\pit}_{i,t}|H_{t-M_t}}$, we attempt to upper bound $\Pro{\neg F^{\pit}_{i,t}|H_{t-M_t}}$ by union bound over the probability of each arm $i$ being played at some round $t' \in [t-d_i+1,t-1]$. More specifically: 

\begin{align*}
&\Ex{\mathcal{R}_{N,\pit}}{z_{i,j}(t-M_t) \beta_{i,t}\Pro{F^{\pit}_{i,t}|H_{t-M_t}}}\\
&= \Ex{\mathcal{R}_{N,\pit}}{z_{i,j}(t-M_t)\Pro{F^{\pit}_{i,t}|H_{t-M_t}}} \\
&= \Ex{\mathcal{R}_{N,\pit}}{z_{i,j}(t-M_t)\left(1 - \Pro{\neg F^{\pit}_{i,t}|H_{t-M_t}}\right)} \\
&\geq \Ex{\mathcal{R}_{N,\pit}}{z_{i,j}(t-M_t)\left(1 - \sum_{t' \in [t-d_i+1, t-1]}\Pro{A^{\pit}_{t'} = i|H_{t-M_t}}\right)} \\
&\geq \Ex{\mathcal{R}_{N,\pit}}{z_{i,j}(t-M_t)\left(1 - \sum_{t' \in [t-d_i+1, t-1]} \Pro{S^{\pit}_{i,t'}, B^{\pit}_{i,t'}, F^{\pit}_{i,t'} | H_{t-M_t}}\right)}.
\end{align*}
For each $t' \in [t-d_i+1, t-1]$, the events $S^{\pit}_{i,t'}$, $B^{\pit}_{i,t'}$ and $F^{\pit}_{i,t'}$ are independent conditioned on $H_{t-M_t}$, since the outcomes of $S^{\pit}_{i,t'}$ and $B^{\pit}_{i,t'}$ depend on the extreme points computed by \ucb before time $t-M_t$. Moreover, since $M_t > d_i$, we have that $\Pro{S^{\pit}_{i,t'} | H_{t-M_t}} = \sum_{j' \in \C} \Pro{S^{\pit}_{i,t'} | C_{t'} = j', H_{t-M_t}} \Pro{C_{t'} | H_{t-M_t}} =\sum_{j' \in \C} f_{j'} \Pro{S^{\pit}_{i,t'} | C_{t'} = j', H_{t-M_t}}$, where the last equality follows by independence of $C_{t'}$ and $H_{t-M_t}$, for $M_t > d_i$. Finally, we have that $\Pro{S^{\pit}_{i,t'} | C_{t'} = j', H_{t-M_t}} = \frac{z_{i,j'}(t'- M_{t'})}{f_{j'}}$, since the probability of the event $S^{\pit}_{i,t'}$ depends on the extreme point computed at time $t'- M_{t'}$, and is computable conditioning on $H_{t-M_t}$ (see Fact~\ref{appendix:fact:delayedexploit} in Appendix~\ref{appendix:criticaltime}). By combining the aforementioned facts, we have:
\begin{align*}
&\Ex{\mathcal{R}_{N,\pit}}{z_{i,j}(t-M_t) \beta_{i,t}\Pro{F^{\pit}_{i,t}|H_{t-M_t}}}\\
&\geq \Ex{\mathcal{R}_{N,\pit}}{z_{i,j}(t-M_t)\left(1 - \sum_{t' \in [t-d_i+1, t-1]} \Pro{S^{\pit}_{i,t'}, B^{\pit}_{i,t'}, F^{\pit}_{i,t'} | H_{t-M_t}}\right)} \\
&= \Ex{\mathcal{R}_{N,\pit}}{z_{i,j}(t-M_t)\left(1 - \sum_{t' \in [t-d_i+1, t-1]} \Pro{S^{\pit}_{i,t'}| H_{t-M_t}}\Pro{B^{\pit}_{i,t'}| H_{t-M_t}} \Pro{F^{\pit}_{i,t'} | H_{t-M_t}}\right)} \\
&= \Ex{\mathcal{R}_{N,\pit}}{z_{i,j}(t-M_t)\left(1 - \sum_{t' \in [t-d_i+1, t-1]} \sum_{j' \in \C} f_{j'} \frac{z_{i,j'}(t'- M_{t'})}{f_{j'}} \beta_{i,t'} \Pro{F^{\pit}_{i,t'} | H_{t-M_t}}\right)}\\
&= \Ex{\mathcal{R}_{N,\pit}}{z_{i,j}(t-M_t)\left(1 - \sum_{t' \in [t-d_i+1, t-1]} \sum_{j' \in \C} z_{i,j'}(t'- M_{t'}) \beta_{i,t'} \Pro{F^{\pit}_{i,t'} | H_{t-M_t}}\right)}.
\end{align*}
By definition of $\beta_{i,t'}$, we have that $\beta_{i,t'} \leq \frac{d_i}{2d_i - 1}\frac{1}{ \Pro{F^{\pit}_{i,t'}|H_{t'-M_{t'}}}}$. Moreover, for any extreme point solution of \eqref{lp:LP}, by constraints \eqref{flp:window}, we have that $\sum_{j' \in \C} z_{i,j'}(t'- M_{t'}) \leq \frac{1}{d_i}$. Therefore, the above relation becomes: 

\begin{align*}
&\Ex{\mathcal{R}_{N,\pit}}{z_{i,j}(t-M_t) \beta_{i,t}\Pro{F^{\pit}_{i,t}|H_{t-M_t}}}\\
&\geq \Ex{\mathcal{R}_{N,\pit}}{z_{i,j}(t-M_t)\left(1 - \sum_{t' \in [t-d_i+1, t-1]} \sum_{j' \in \C} z_{i,j'}(t'- M_{t'}) \beta_{i,t'} \Pro{F^{\pit}_{i,t'} | H_{t-M_t}}\right)} \\
&\geq \Ex{\mathcal{R}_{N,\pit}}{z_{i,j}(t-M_t)\left(1 - \frac{1}{d_i} \sum_{t' \in [t-d_i+1, t-1]} \frac{d_i}{2d_i - 1} \frac{\Pro{F^{\pit}_{i,t'} | H_{t-M_t}}}{\Pro{F^{\pit}_{i,t'} | H_{t'-M_{t'}}}} \right)} \\
&= \Ex{\mathcal{R}_{N,\pit}}{z_{i,j}(t-M_t)\left(1 - \frac{1}{2d_i - 1} \sum_{t' \in [t-d_i+1, t-1]} \frac{\Pro{F^{\pit}_{i,t'} | H_{t-M_t}}}{\Pro{F^{\pit}_{i,t'} | H_{t'-M_{t'}}}} \right)}.
\end{align*}

For any $t \geq T_C$ and $t' \in [t-d_i+1,t-1]$, by Lemma \ref{lemma:regret:mixing}, we have:
\begin{align}
\frac{\Pro{F^{\pit}_{i,t'} | H_{t-M_t}}}{\Pro{F^{\pit}_{i,t'} | H_{t'-M_{t'}}}} \leq 1 + c_0\cdot c_1^{-M_t}.  \label{eq:mixing}
\end{align}

By using inequality \eqref{eq:mixing}, we get:
\begin{align*}
&\Ex{\mathcal{R}_{N,\pit}}{z_{i,j}(t-M_t) \beta_{i,t}\Pro{F^{\pit}_{i,t}|H_{t-M_t}}}\\
&\geq \Ex{\mathcal{R}_{N,\pit}}{z_{i,j}(t-M_t)\left(1 - \frac{1}{2d_i - 1} \sum_{t' \in [t-d_i+1, t-1]} \left(1 + c_0\cdot c_1^{-M_t}\right) \right)} \\
&= \Ex{\mathcal{R}_{N,\pit}}{z_{i,j}(t-M_t)\left(1 - \frac{d_i-1}{2d_i - 1}  + \frac{d_i-1}{2d_i - 1} c_0\cdot c_1^{-M_t} \right)} \\
&= \Ex{\mathcal{R}_{N,\pit}}{z_{i,j}(t-M_t)\left(\frac{d_i}{2d_i - 1}  + \frac{d_i-1}{2d_i - 1} c_0\cdot c_1^{-M_t} \right)} \\
&= \Ex{\mathcal{R}_{N,\pit}}{z_{i,j}(t-M_t)\left(\frac{d_i}{2d_i - 1}  + \frac{d_i-1}{2d_i - 1} c_0\cdot c_1^{-M_t} \right)}
\end{align*}

By summing over all $t \in [T_c,T]$ and using the above analysis, we have: 
\begin{align}
&\Ex{\mathcal{R}_{N,\pit}}{\sum^T_{t = T_c} \sum_{i \in \A} \sum_{j \in \C} X_{i,j,t} \event{A^{\pit}_t=i , C_t =j}}\nonumber\\ 
&= \Ex{\mathcal{R}_{N,\pit}}{\sum^T_{t = T_c} \sum_{i \in \A} \sum_{j \in \C} \mu_{i,j} \event{A^{\pit}_t=i , C_t =j}} \nonumber\\
&= \Ex{\mathcal{R}_{N,\pit}}{\sum^T_{t = T_c} \sum_{i \in \A} \sum_{j \in \C} \mu_{i,j} z_{i,j}(t-M_t) \beta_{i,t}\Pro{F^{\pit}_{i,t}|H_{t-M_t}}} \nonumber\\
&\geq \Ex{\mathcal{R}_{N,\pit}}{\sum^T_{t = T_c} \sum_{i \in \A} \sum_{j \in \C} \mu_{i,j} z_{i,j}(t-M_t)\left(\frac{d_i}{2d_i - 1}  - \frac{d_i-1}{2d_i - 1} c_0 \cdot c_1^{-M_t} \right)} \nonumber\\
&\geq \Ex{\mathcal{R}_{N,\pit}}{\sum^T_{t = T_c} \sum_{i \in \A} \sum_{j \in \C} \mu_{i,j} \frac{d_i}{2d_i - 1} z_{i,j}(t-M_t) } - \sum_{t = [T]} c_0\cdot c_1^{-M_t}, \label{eq:ucb:mapping:4}
\end{align}
where in the last inequality we use the fact that 
$$
\Ex{\mathcal{R}_{N,\pit}}{\sum^T_{t = T_c} \sum_{i \in \A} \sum_{j \in \C} \mu_{i,j} \frac{d_i-1}{2d_i - 1} z_{i,j}(t-M_t) c_0 \cdot c_1^{-M_t}} \leq \sum^T_{t = T_c} c_0\cdot c_1^{-M_t}.
$$
Furthermore, by our choice of $M_t$, we have that
$$
M_t = \lfloor 2\log_{c_1}(t)\rfloor  + \lceil\log_{c_1}(c_0)\rceil +1 \geq 2\log_{c_1}(t) + \log_{c_1}(c_0) = \log_{c_1}(c_0\cdot t^2),
$$
which implies that $\sum^T_{t = T_c} c_0\cdot c_1^{-M_t} \leq \sum_{t \in [T]} c_0 \cdot c_1^{-\log_{c_1}(t^2 \cdot c_0)} \leq \sum^{+\infty}_{t = 1} \frac{1}{t^2} = \frac{\pi^2}{6}
$.
Therefore, inequality \eqref{eq:ucb:mapping:4} becomes:
\begin{align}
&\Ex{\mathcal{R}_{N,\pit}}{\sum^T_{t = T_c} \sum_{i \in \A} \sum_{j \in \C} X_{i,j,t} \event{A^{\pit}_t=i , C_t =j}} \geq \Ex{\mathcal{R}_{N,\pit}}{\sum^T_{t = T_c} \sum_{i \in \A} \sum_{j \in \C} \mu_{i,j} \frac{d_i}{2d_i - 1} z_{i,j}(t-M_t) } - \frac{\pi^2}{6}. \label{eq:ucb:mapping:last}
\end{align}

\paragraph{Bounding small t and combining everything.}
By construction \ucb, for the first rounds where $t \leq T_c - 1$, the algorithm selects arms and constructs non-skipping probabilities with respect to an initial extreme point solution $Z(0)=\{z_{i,j}(0)\}_{\forall i,j}$ to \eqref{lp:LP}. Since we cannot bound the expected reward of \ucb for the these time steps, we accumulate this loss in the regret as follows:
\begin{align}
&\sum^{T_c-1}_{t = 1} \sum_{i \in \A} \sum_{j \in \C} \mu_{i,j} \frac{d_i}{2d_i - 1} z^*_{i,j} - \Ex{\mathcal{R}_{N,\pit}}{\sum^{T_c-1}_{t = 1} \sum_{i \in \A} \sum_{j \in \C} \mu_{i,j} \event{A^{\pit}_t=i , C_t =j}} \nonumber\\
&\leq \sum^{T_c-1}_{t = 1} \sum_{i \in \A} \sum_{j \in \C} \mu_{i,j} \frac{d_i}{2d_i - 1} z^*_{i,j} \leq T_c-1 \label{eq:mapping:smallt}
\end{align}

For the overall regret we have:
\begin{align}
&\sum_{t \in [T]} \sum_{i \in \A} \sum_{j \in \C} \mu_{i,j} \frac{d_i}{2d_i - 1} z^*_{i,j} - \Rew^{\pit}_I(T) \nonumber\\
&= \sum_{t \in [T]} \sum_{i \in \A} \sum_{j \in \C} \mu_{i,j} \frac{d_i}{2d_i - 1} z^*_{i,j} - \Ex{\mathcal{R}_{N,\pit}}{\sum_{t \in [T]} \sum_{i \in \A} \sum_{j \in \C} X_{i,j,t} \event{A^{\pit}_t=i , C_t =j}}\nonumber\\
&= \sum_{t \in [T]} \sum_{i \in \A} \sum_{j \in \C} \mu_{i,j} \frac{d_i}{2d_i - 1} z^*_{i,j} - \Ex{\mathcal{R}_{N,\pit}}{\sum_{t \in [T]} \sum_{i \in \A} \sum_{j \in \C} \mu_{i,j} \event{A^{\pit}_t=i , C_t =j}}\nonumber\\
&= \sum_{t \in [T]} \sum_{i \in \A} \sum_{j \in \C} \mu_{i,j} \frac{d_i}{2d_i - 1} z^*_{i,j} - \Ex{\mathcal{R}_{N,\pit}}{\sum^{T_c-1}_{t = 1} \sum_{i \in \A} \sum_{j \in \C} \mu_{i,j} \event{A^{\pit}_t=i , C_t =j}}\nonumber\\
&\quad\quad\quad\quad\quad\quad\quad\quad\quad\quad\quad\quad\quad - \Ex{\mathcal{R}_{N,\pit}}{\sum^T_{t = T_c} \sum_{i \in \A} \sum_{j \in \C} \mu_{i,j} \event{A^{\pit}_t=i , C_t =j}} \nonumber\\
&\leq \sum_{t \in [T]} \sum_{i \in \A} \sum_{j \in \C} \mu_{i,j} \frac{d_i}{2d_i - 1} z^*_{i,j} - \Ex{\mathcal{R}_{N,\pit}}{\sum^{T_c-1}_{t = 1} \sum_{i \in \A} \sum_{j \in \C} \mu_{i,j} \event{A^{\pit}_t=i , C_t =j}}\nonumber\\
&\quad\quad\quad\quad\quad\quad\quad\quad\quad - \Ex{\mathcal{R}_{N,\pit}}{\sum^T_{t = T_c} \sum_{i \in \A} \sum_{j \in \C} \mu_{i,j} \frac{d_i}{2d_i - 1} z_{i,j}(t-M_t)} + \frac{\pi^2}{6} \label{eq:ucb:mapping:5}
\end{align}
\begin{align}
&\leq \Ex{\mathcal{R}_{N,\pit}}{\sum^T_{t = T_c} \sum_{i \in \A} \sum_{j \in \C} \mu_{i,j} \frac{d_i}{2d_i - 1} \left(z^*_{i,j} - z_{i,j}(t-M_t)\right)} + T_c -1 + \frac{\pi^2}{6} \label{eq:ucb:mapping:6b}\\
&= \Ex{\mathcal{R}_{N,\pit}}{\sum^T_{t = T_c} \sum_{i \in \A} \sum_{j \in \C} \mu_{i,j} \frac{d_i}{2d_i - 1} \left(z^*_{i,j} - z_{i,j}(t-M_t)\right)} + 3\cdot d_{\max} + 70  + \frac{\pi^2}{6} \label{eq:ucb:mapping:6a}\\
&\leq \Ex{\mathcal{R}_{N,\pit}}{\sum^T_{t = T_c} \sum_{i \in \A} \sum_{j \in \C} \mu_{i,j}\left(z^*_{i,j} - z_{i,j}(t-M_t)\right)} + 3\cdot d_{\max} + 70  + \frac{\pi^2}{6}, \label{eq:ucb:mapping:6}
\end{align}
where \eqref{eq:ucb:mapping:5} follows by inequality \eqref{eq:ucb:mapping:last} and \eqref{eq:ucb:mapping:6b} by inequality \eqref{eq:mapping:smallt}. Finally, equality \eqref{eq:ucb:mapping:6a} follows an upper bound on $T_c$ (given in Fact \ref{appendix:fact:bounds} of Appendix \ref{appendix:criticaltime}) and inequality \eqref{eq:ucb:mapping:6} by the fact that $\frac{d_i}{2d_i - 1} \leq 1$ for any $i \in \A$.

\paragraph{Synchronizing the large time steps and completing the proof.}
For completing the proof of the lemma, we focus on the quantity 
$$
\Ex{\mathcal{R}_{N,\pit}}{\sum^T_{t = T_c} \sum_{i \in \A} \sum_{j \in \C} \mu_{i,j}\left(z^*_{i,j} - z_{i,j}(t-M_t)\right)}.
$$
Recall that for any $t\geq T_c$, the algorithm uses for arm sampling the extreme point solution $Z(t - M_t)$, computed using the indices $\{\bar{\mu}(t-M_t)\}_{i,j}$. As we show in Appendix \ref{appendix:criticaltime} (see Fact \ref{appendix:fact:unitincreases}), for $t \geq T_c$, the value $M_t$ cannot be increased by more than one unit per round. Given any time interval $[t_1,t_2]$, with $t_1 \geq T_c$, we say that the UCB indices of the interval are {\em synchronized} (or, simply, we say that the interval is synchronized), if for any $t \in [t_1,t_2]$, there exists a integer constant $M'$, such that \ucb at round $t$, uses information from time $t - M'$.

Let $t'$ be the first time that $M_t$ increases by one after time $T_c$. Clearly, the time interval $[T_c, t')$ is {\em synchronized} as the information used at each round from $T_c$ to $t'-1$ corresponds to times $T_c - M_{T_c}, T_c - M_{T_c} + 1, \dots, t' -1 - M_{T_c}$. However, at time $t'$, given the fact that $M_{t'} = M_{T_c} + 1$, the index used corresponds, again, to time $t' - 1 - M_{T_c} = t' - M_{t'}$. Hopefully, by ignoring time $t'$, we can see that the index used at $t'+1$ corresponds to time $t'+1 - M_{t'} = t' - M_{T_c}$, which remains synchronized with the interval before $t'$.

By repeating the above procedure, we ignore the non-synchronized rounds (that correspond to the unit increases of $M_t$) and we merge the remaining rounds into a single synchronized interval. Let $L$ be the number of non-synchronized time steps in $[T_c,T]$, which is formally defined as
\begin{align*}
    L = |\{t \in [T_c+1,T]~|~M_{t} = M_{t-1}+1\}|.
\end{align*}

By definition of $M_t$, the total number of non-synchronized time steps (as $t'$) can be upper bounded by $M_T$ , which, in turn, can be upper bounded by $2 \log_{c_1}(T)  + \log_{c_1}(c_0) + 2 \leq \frac{1}{3} \ln(T) + 9 + 2 d_{\max} $.

Let $\Delta_{\max} = \sup_{Z \in \extr} \Delta_Z$, be the maximum suboptimiality gap over all the extreme points of $\extr$. The regret associated with each non-synchronized time step greater than $T_c$ can be upper bounded by $\Delta_{\max}$. By the above analysis, it follows directly that
\begin{align*}
&\Ex{\mathcal{R}_{N,\pit}}{\sum^T_{t = T_c} \sum_{i \in \A} \sum_{j \in \C} \mu_{i,j}  \left(z^*_{i,j} - z_{i,j}(t-M_t)\right)} \\
&\leq \Ex{\mathcal{R}_{N,\pit}}{\sum^{T-L}_{t = T_c} \sum_{i \in \A} \sum_{j \in \C} \mu_{i,j} \left(z^*_{i,j} +z_{i,j}(t-M_{T_c})\right)} + \left(\frac{1}{3} \ln(T) + 9 + 2 d_{\max}\right) \Delta_{\max}\\
&\leq \Ex{\mathcal{R}_{N,\pit}}{\sum^{T-L-M_{T_c}}_{t = T_c-M_{T_c}} \sum_{i \in \A} \sum_{j \in \C} \mu_{i,j} \left(z^*_{i,j} +z_{i,j}(t)\right)} + \left(\frac{1}{3} \ln(T) + 9 + 2 d_{\max}\right) \Delta_{\max}\\
\end{align*}

By combining the above inequality with \eqref{eq:mapping:competloss} and \eqref{eq:ucb:mapping:6}, we can prove the following upper bound:
\begin{align*}
&\alpha\Reg^{\pit}_I(T) \leq \Ex{\mathcal{R}_{N,\pit}}{\sum^{T-L-M_{T_c}}_{t = T_c-M_{T_c}} \sum_{i \in \A} \sum_{j \in \C} \mu_{i,j} \left(z^*_{i,j} +z_{i,j}(t)\right)} + \frac{2}{3}(d_{\max}-1)\\
&\quad\quad\quad\quad\quad\quad\quad\quad\quad +\left(\frac{1}{3} \ln(T) + 9 + 2 d_{\max}\right) \Delta_{\max} + 3\cdot d_{\max} + 70  + \frac{\pi^2}{6}.
\end{align*}
By noticing that $\Delta_{\max} \leq 1$, we can simplify the less important constants of the above bound as
\begin{align*}
\alpha\Reg^{\pit}_I(T) \leq \Ex{\mathcal{R}_{N,\pit}}{\sum^{T-L-M_{T_c}}_{t = T_c-M_{T_c}} \sum_{i \in \A} \sum_{j \in \C} \mu_{i,j} \left(z^*_{i,j} +z_{i,j}(t)\right)} +\frac{1}{3}\ln(T)\Delta_{\max} + 6 \cdot d_{\max} + 71.
\end{align*}
Finally, we use that $T_c - M_{T_c} \geq 1$ and we let $M = L + M_c= \Theta(\log T + d_{\max})$, which leads to:
\begin{align*}
\alpha\Reg^{\pit}_I(T) \leq \Ex{\mathcal{R}_{N,\pit}}{\sum^{T-M}_{t = 1} \sum_{i \in \A} \sum_{j \in \C} \mu_{i,j} \left(z^*_{i,j} +z_{i,j}(t)\right)} +\frac{1}{3}\ln(T)\Delta_{\max} + 6 d_{\max} + 71.
\end{align*}
\end{proof}

\subsection{Proof of Lemma \ref{lemma:regret:counter-to-samples}}
\restateLemmaSubsampling*
\begin{proof}
We fix an an arbitrary TP group $(i,j, l)$. Let $t_k$ be the time and $Z_k$ be the suboptimal extreme point used by \ucb for sampling arms when the counter $N_{i,j, l}(t')$ is increased for the $k$-th time. Moreover, we denote by $Y_k = \event{A^{\pit}_{t_k} = i, C_{t_k} = j}$ the event that the TP group $(i,j,l)$ is triggered at time $t_k$, namely, arm $i$ is played under context $j$ and $2^{-l}\leq z^{Z_k}_{i,j} = z_{i,j}([t_k-M_{t_k}]^+) \leq 2^{-l+1}$, where $[t]^+ = \max\{t,0\}$ for any integer $t$. We require concentration bounds for $\sum_{k=1}^{N_{i,j,l}(t)}Y_k$ conditioned on $N_{i,j,l}(t)=s$. The main roadblock in the analysis, comparing to \citep{WC17}, is that, due to the blocking constraints, the random variables $Y_k$ are not {\em mutually independent}. Indeed, if $|t_{k'} - t_k| < d_i$ then $Y_k$ and $Y_{k'}$ cannot be simultaneously equal to $1$. In order to overcome the above issue, we opportunistically subsample the events $\{Y_k\}$ to ensure that the distance between two contiguous subsampled events, where $Y_k = 1$, is at least time $(d_i+1)$ apart (inclusive of the first instance). 
 
We first separate each {\em triggering} (i.e., arm pulling) event into two stages: {\em attempting to trigger} $Y'_k = \event{S^{\pit}_{i,t_k}, B^{\pit}_{i,t_k}, C_{t_k} = j}$, and {\em actual triggering} $Y_k = \event{S^{\pit}_{i,t_k}, B^{\pit}_{i,t_k}, C_{t_k} = j, F^{\pit}_{i,t_k}}$. Given this distinction, the second stage takes into account the blocking constraints, while the first stage takes into account the randomness introduced by nature and the random choices of \ucb.

We partition the sequence $\{1, 2, \dots, N_{i,j,l}(t)\}$, into $\lfloor N_{i,j,l}(t)/(2d_i-1) \rfloor$ many windows of length $(2d_i -1)$. The $\ell$-th window consists of the subsequence $\{\ell(2d_i-1)+1,\dots, (\ell+1)(2d_i-1)\}$, starting from $\ell = 0$. Notice that for $s \geq 2\cdot 2^l \geq 2 d_i$, we have at least one such window, since $2^{-l} \leq \frac{1}{d_i}$.

We now define the indicator for the {\em triggering event} in window $\ell$, denoted by $\tilde{Y}_{\ell}$, and the {\em triggering time} (if the arm is triggered) in window $\ell$, denoted by $\tilde{t}_{\ell}$. In each window $\ell$, if there exists a $k$ in the last $d_i$ steps in the window (i.e. $k \in \{\ell(2d_i-1) + d_i,\dots,(\ell+1)(2d_i-1)\}$) such that the algorithm tries to trigger at time $t_k$ (i.e. $Y'_k = 1$), we set $\tilde{Y}_{\ell} = Y_{k}$ and time $\tilde{t}_\ell = t_k$. Otherwise, we set $\tilde{Y}_{\ell} = 0$ and $\tilde{t}_\ell = \ell(2d_i-1) + d_i$. Thus, we have constructed an opportunistically subsampled sequence of tuples $(\tilde{Y}_\ell, \tilde{t}_\ell)$ for $0 \leq \ell \leq \lfloor N_{i,j,l}(t)/ (2d_i-1) \rfloor$, from the original subsequence $(Y_k, t_k)$. 
Clearly, $\sum_{\ell=0}^{\lfloor N_{i,j,l}(t)/ (2d_i-1) \rfloor}\tilde{Y}_\ell$ constructs a lower bound for $T_{i,j}(t)$.
 
To avoid repetitive notations, let us denote by $\mathcal{H}_\ell = \{(\tilde{Y}_1, \tilde{t}_1), \dots, (\tilde{Y}_{(\ell-1)}, \tilde{t}_{(\ell-1)})\}$ the subsequence from $0$ upto (and excluding) the $\ell$-th entry in the sequence. We call the first event of observing at least one $Y'_k = 1$ in the $\ell$-th window as $\mathcal{E}_\ell$. As the sampling only happens at the later part of each window, the previous subsampling ensures that the random variables $\tilde{t}_\ell$ are at least $d_i$ time steps apart. We now claim that when $\tilde{Y}_{\ell}$ is set to $Y_{k}$ then irrespective of the past $(\tilde{Y}_{\ell'}, \tilde{t}_{\ell'})$, we have 
$\mathbb{P}[\tilde{Y}_{\ell}|\mathcal{E}_\ell, \mathcal{H}_\ell] \geq (1-\tfrac{1}{d_i})^{(d_i-1)} \geq \tfrac{1}{e}$.
The above is true because, we know that when conditioned on history at least $d_i$ time steps apart we have  $\mathbb{P}[F^{\pit}_{i,\tilde{t}_\ell} | H_{\tilde{t}_{(\ell-1)}}] \geq 1/e$. Phrased differently, if an arm is not deterministically blocked, then it is available with probability at least $1/e$.  
  
Whenever the counter $N_{i,j,l}(t)$ is increased it is, by definition, due to an extreme point which plays the arm $(i,j)$ with probability at least $2^{-l}$, i.e. $\Pro{S^{\pit}_{i,t_k}, C_{t_k} = j} \geq 2^{-l}$. Moreover, $B^{\pit}_{i,t_k}$ is a Bernoulli r.v. with mean $\beta_{i,t_k} = \min\left(1, \frac{d_i}{2d_i-1} \frac{1}{ \Pro{F^{\pit}_{i,t_k}~|~H_{[t_k - M_{t_k}]^+}}}\right)$. Furthermore, it is not hard to see that $\beta_{i,t_k} \geq \tfrac{d_i}{2d_i-1}$ and, thus, $B^{\pit}_{i,t_k}$ stochastically dominates an independent Bernoulli r.v. of mean $\frac{d_i}{2d_i-1}$. Similarly, $\Pro{S^{\pit}_{i,t_k}}$ stochastically dominates an independent Bernoulli r.v. of mean $2^{-l}$. Therefore, the probability of event $\{Y_k=1\}$ (trying to trigger arm $i$ at context $j$) is at least $\tfrac{d_i}{2d_i-1} 2^{-l}$. We have:

\begin{align}
\mathbb{P}[\mathcal{E}_\ell| \mathcal{H}_\ell] &= 1 - \mathbb{P}[\mathcal{E}^c_\ell| \mathcal{H}_\ell] \nonumber\\
&\geq 1 - (1- \frac{d_i}{2d_i-1} 2^{-l})^{d_i+1} \label{eq:subs:1}\\
&\geq  \frac{d_i(d_i+1)}{2d_i-1} 2^{-l} - \frac{d_i (d_i+1)}{2} (\frac{d_i}{2d_i-1})^2 2^{-2l} \label{eq:subs:2}\\
&\geq \frac{d_i(d_i+1)}{2d_i-1} 2^{-l} - \frac{d_i (d_i+1)}{2} \frac{d_i}{2d_i-1} 2^{-l} \frac{1}{d_i} \label{eq:subs:3}\\
&= \frac{d_i(d_i+1)}{2(2d_i-1)} 2^{-l}, \nonumber
\end{align}
where \eqref{eq:subs:2} holds due to the Taylor expansion of $(1-x)^{d_i+1}$ around $x = 0$, and \eqref{eq:subs:3} follows by noticing that $2^{-l} \leq 1/d_i$ and $\tfrac{d_i}{2d_i-1}\leq 1$. Finally, \eqref{eq:subs:1} follows by the fact that for any extreme point, arm $i$ is played with probability at most $1/d_i$, given that $\sum_{j\in \C}z _{i,j} \leq 1/d_i$.

By combining the above inequalities, then for all $0\leq \ell \leq  \lfloor N_{i,j,l}(t)/ (2d_i-1) \rfloor $, we have 
 \begin{align*}
\mathbb{E}[\tilde{Y}_{\ell}|\mathcal{H}_\ell]
&\geq \mathbb{E}[\tilde{Y}_{\ell}|\mathcal{E}_\ell, \mathcal{H}_\ell] \mathbb{P}[\mathcal{E}_\ell| \mathcal{H}_\ell] \\
&\geq \frac{1}{e}\frac{d_i(d_i+1)}{2(2d_i-1)} 2^{-l} \\
&\geq  (2d_i -1) 2^{-l} \frac{1}{8 e}.
\end{align*}
The first inequality holds as $\tilde{Y}_{\ell}\geq 0$, and the second inequality is obtained by substituting the above appropriate lower bounds.
 
We next apply the multiplicative Chernoff bound for dependent random variables as stated in Theorem~\ref{lemma:regret:chernoff} to obtain the final concentration inequality. We use $\delta = 2/3$.
\begin{align*}
\Pro{ N_{i,j, l}(t) = s, T_{i,j}(t) \leq \frac{1}{3} 
\bigg\lfloor \frac{N_{i,j, l}(t)}{2d_i-1} \bigg\rfloor (2d_i-1) 2^{-l}\frac{1}{8 e} }
&\leq \exp(- \frac{2}{9} \bigg\lfloor\frac{s}{2d_i-1} \bigg\rfloor (2d_i-1) 2^{-l}\frac{1}{8 e}) \\
&\leq \exp(- 3 \ln(t)) \\
&= \frac{1}{t^3},
\end{align*}
where the second inequality holds for $s \geq 109\cdot e \cdot 2^{l} \ln(t) \geq 108\cdot e \cdot 2^{l} \ln(t) + 2 d_i - 1$, where we use the fact that $2^l \geq d_i$.
\end{proof}

\subsection{Proof of Lemma \ref{lemma:regret:sparse}}
\restateRegretSparse*

\begin{proof}
Recall that in any feasible extreme point solution of \eqref{lp:LP}, there exist $|\A| |\C| = k \cdot m$ linearly independent inequalities that are tight (i.e., they are met with equality). By the structure of \eqref{lp:LP}, we know that at most $k$ of them can be from the set \eqref{flp:window} and at most $m$ can be from the set \eqref{flp:conditional}. Therefore, the remaining tight inequalities should be nonnegativity constraints and, thus, they are of the form $z_{i,j} = 0$. This implies that at most $k+m$ variables can be nonzero and, therefore, that the support of any extreme point solution of \eqref{lp:LP} has cardinality at most $k + m$.
\end{proof}
\subsection{Proof of Theorem \ref{theorem:regret:bound} (Regret Upper Bound)}
\restateTheoremFinalRegret*
\begin{proof}
The proof of our regret bound follows closely the structure of \citep{WC17}. In the following, we present a version of their proof simplified and adapted to our setting. We start from the upper bound on the regret given by Lemma \ref{lemma:regret:mapping}. Then, we study this regret upper bound using techniques from \citep{WC17} and making use of our Lemmas \ref{lemma:regret:counter-to-samples} and \ref{lemma:regret:sparse}, in order to achieve tighter final regret bounds. 

By Lemma~\ref{lemma:regret:mapping}, we have the following upper bound on the $\alpha$-regret
\begin{align*}
\alpha\Reg^{\pit}_I(T) \leq \Ex{\mathcal{R}_{N,\pit}}{\sum^{T-M}_{t = 1} \sum_{i \in \A} \sum_{j \in \C} \mu_{i,j} \left(z^*_{i,j} +z_{i,j}(t)\right)} +\frac{1}{3}\ln(T)\Delta_{\max} + 6 d_{\max} + 71,
\end{align*}
where $M = \Theta(\log(T) + d_{\max})$.

By using our definition of suboptimality gaps, we can express the first term of the above bound as
\begin{align*}
    \Ex{\mathcal{R}_{N,\pit}}{\sum^{T-M}_{t = 1} \sum_{i \in \A} \sum_{j \in \C} \mu_{i,j} \left(z^*_{i,j} +z_{i,j}(t)\right)} = \Ex{\mathcal{R}_{N,\pit}}{\sum^{T-M}_{t = 1} \Delta_{Z(t)}},
\end{align*}
where $\Delta_{Z(t)}$ is the suboptimality gap of the extreme point solution of $\eqref{lp:LP}(t)$. 

In the above summation, notice that for the computation of every $Z(t)$ for $t \in [T-M]$, the algorithm uses strictly updated UCB indices, since we have already excluded the rounds where indices are reused, due to the increases of $M_t$ (see Lemma \ref{lemma:regret:mapping}).

We start by defining several important events that may occur during a run of our algorithm \ucb. A reader familiar with the work of \citep{WC17} should easily recognize their role. Recall that $T_{i,j}(t)$ denotes the number of times arm $i$ is played under context $j$ up to (and excluding) time $t$. Moreover, we denote by $N_{i,j,l}(t)$ the value of the counter that corresponds to the TP group $\extr_{i,j,l}$, at the beginning of round $t$.  

\begin{definition}[Nice sampling]
We say that at the beginning of round $t$, \ucb has a {\em nice sampling}, denoted by $\mathcal{N}^s_t$, if it is the case that:
\begin{align*}
    |\hat{\mu}_{i,j,T_{i,j}(t)} - \mu_{i,j}| \leq \sqrt{\frac{3 \ln{(t)}}{2 T_{i,j}(t)}}, \forall i \in \A, \forall j \in \C.
\end{align*}
\end{definition}

It is not hard to verify that, on any round $t$ such that $\mathcal{N}^s_t$ holds, we have: 
\begin{align*}
    \mu_{i,j} \leq \bar\mu_{i,j}(t) \leq \min\bigg\{1 , \mu_{i,j} + 2 \sqrt{\frac{3 \ln{(t)}}{2T_{i,j}(t)}}\bigg\}, \forall i \in \A, \forall j \in \C.
\end{align*}
The following lemma provides a lower bound to the probability that the \ucb has a nice sampling at some time $t$.

\begin{restatable}{lemma}{restateNiceSampling}\label{lemma:regret:nicesampling}
The probability that \ucb has a nice sampling at time $t$ is at least $\Pro{\mathcal{N}^s_t} \geq 1 - 2 k m t^{-2}$.
\end{restatable}

For the rest of this proof, we fix the constants $\dummy = 109\cdot e$ and $\dummyy = 24\cdot e$. Moreover, for any real number $y$, we denote by $\left[y\right]^+ = \max\{y,0\}$.

\begin{definition}[Nice triggering]
We say that at the beginning of round $t$, \ucb has a {\em nice triggering}, denoted by $\mathcal{N}^\tau_{t}$, if for any TP group $\extr_{i,j,l}$ associated with the pair $(i,j)$ and for any $1 \leq l \leq \left[ \log_2(\frac{2(k+m)}{\Delta^{i,j}_{\min}})\right]^+$, given that $\sqrt{\frac{\dummy \ln{(t)}}{N_{i,j,l}(t-1)2^{-l}}} \leq 1$, it holds $T_{i,j}(t-1) \geq \frac{1}{\dummyy} N_{i,j,l}(t-1)2^{-l}$.
\end{definition}

\begin{restatable}{lemma}{restateNiceTriggering}\label{lemma:regret:nicetriggering}
The probability that \ucb does not have a nice triggering at time $t$ is at upper bounded by $\Pro{\neg \mathcal{N}^{\tau}_t} \leq \sum_{i \in \A}\sum_{j \in \C} \left[ \log_2(\frac{2(k+m)}{\Delta^{i,j}_{\min}})\right]^+ t^{-2}$.
\end{restatable}

We consider the following functions:
\begin{align*}
\ell_{l,T}(\Delta) =\bigg\lfloor \frac{96 \cdot 2^{-l}\cdot \dummyy (k+m)^2 \ln T}{\Delta^2} \bigg\rfloor
\end{align*}
\begin{align*}
\kappa_{l,T}(\Delta, s) = 
\begin{cases}
       4\cdot 2^{-l}, &\quad\text{if }s = 0,\\
       2\sqrt{\frac{4 \dummy \cdot 2^{-l} \ln(T)}{s}}, &\quad\text{if } 1 \leq s \leq \ell_{l,T}(\Delta),\\
       0, &\quad\text{if } s > \ell_{l,T}(\Delta).\\
    \end{cases}
\end{align*}

For any extreme point $Z \in \extr$, we denote by $\tilde{Z} = \{(i,j)\in \A\times\C | z^Z_{i,j}>0\}$ the set of arm context pairs in its support. Notice that by Lemma \ref{lemma:regret:sparse}, for any extreme point $Z \in \extr$, we have $|\tilde{Z}| \leq m +k$. 

For any $Z \in \extr$, let $\Gamma_{Z} = \max_{(i,j) \in \tilde{Z}}\{ \Delta^{i,j}_{\min}\}$ be the maximum $\Delta^{i,j}_{\min}$ over all pairs $(i,j) \in \tilde Z$. Our proof relies on the following technical lemma.

\begin{restatable}{lemma}{restateSubDecomposition}\label{lemma:regret:decomposition}
{\em (Suboptimality decomposition).} For any round $t \in [T]$, if $\{\Delta_{Z(t)} \geq \Gamma_{Z(t)}\}$ and $\mathcal{N}^s_t$, $\mathcal{N}^{\tau}_t$ hold, we have:
\begin{align*}
    \Delta_{Z(t)} \leq \sum_{(i,j) \in \tilde{Z}(t)} \kappa_{l_{i,j}, T}(\Delta^{i,j}_{\min}, N_{i,j,l_{i,j}}(t-1)),
\end{align*}
where $l_{i,j}$ the index of a TP group such that $Z(t) \in \extr_{i,j,l_{i,j}}$.
\end{restatable}

We are now ready to prove the regret bound, with respect to $\{\Delta^{i,j}_{\min}\}_{\forall i,j}$ and $\Delta_{\max}$. For simplicity, we replace $T-M$ with $T$ in the regret upper bound of Lemma $\ref{lemma:regret:mapping}$. Even though the rounds above $T-M$ might not correspond to UCB indices that were actually used in the run of \ucb, we still get an upper bound to our regret by assuming a larger instance in the underlying combinatorial bandit problem. We have that: 
\begin{align*}
&\Ex{\mathcal{R}_{N,\pit}}{\sum_{t \in [T]} \sum_{i \in \A} \sum_{j \in \C} \mu_{i,j} \left(z^*_{i,j} - z_{i,j}(t)\right)} \\
&= \Ex{\mathcal{R}_{N,\pit}}{\sum_{t \in [T]} \Delta_{Z(t)}} \\
&= \Ex{\mathcal{R}_{N,\pit}}{\sum_{t \in [T]} \Delta_{Z(t)} \event{\Delta_{Z(t)} \geq \Gamma_{Z(t)}}} + \Ex{\mathcal{R}_{N,\pit}}{\sum_{t \in [T]} \Delta_{Z(t)} \event{\Delta_{Z(t)} < \Gamma_{Z(t)}}}\\
&= \Ex{\mathcal{R}_{N,\pit}}{\sum_{t \in [T]} \Delta_{Z(t)} \event{\Delta_{Z(t)} \geq \Gamma_{Z(t)}}} \\
&= \Ex{\mathcal{R}_{N,\pit}}{\sum_{t \in [T]} \Delta_{Z(t)} \left(\event{\Delta_{Z(t)} \geq \Gamma_{Z(t)}, \mathcal{N}^{s}_t} + \event{\Delta_{Z(t)} \geq \Gamma_{Z(t)}, \neg\mathcal{N}^{s}_t}\right)}\\
&\leq \Ex{\mathcal{R}_{N,\pit}}{\sum_{t \in [T]} \Delta_{Z(t)} \left(\event{\Delta_{Z(t)} \geq \Gamma_{Z(t)}, \mathcal{N}^{s}_t} + \event{\neg\mathcal{N}^{s}_t}\right)}\\
&= \Ex{\mathcal{R}_{N,\pit}}{\sum_{t \in [T]} \Delta_{Z(t)} \left(\event{\Delta_{Z(t)} \geq \Gamma_{Z(t)}, \mathcal{N}^{s}_t, \mathcal{N}^{\tau}_t } + \event{\Delta_{Z(t)} \geq \Gamma_{Z(t)}, \mathcal{N}^{s}_t, \neg \mathcal{N}^{\tau}_t } + \event{\neg\mathcal{N}^{s}_t}\right)}\\
&\leq \Ex{\mathcal{R}_{N,\pit}}{\sum_{t \in [T]} \Delta_{Z(t)} \left(\event{\Delta_{Z(t)} \geq \Gamma_{Z(t)}, \mathcal{N}^{s}_t, \mathcal{N}^{\tau}_t } + \event{\neg \mathcal{N}^{\tau}_t } + \event{\neg\mathcal{N}^{s}_t}\right)}\\
&\leq \Ex{\mathcal{R}_{N,\pit}}{\sum_{t \in [T]} \Delta_{Z(t)} \event{\Delta_{Z(t)} \geq \Gamma_{Z(t)}, \mathcal{N}^{s}_t, \mathcal{N}^{\tau}_t }} + \Ex{\mathcal{R}_{N,\pit}}{\sum_{t \in [T]} \Delta_{Z(t)}\event{\neg \mathcal{N}^{\tau}_t }} + \Ex{\mathcal{R}_{N,\pit}}{\sum_{t \in [T]} \Delta_{Z(t)}\event{\neg\mathcal{N}^{s}_t}},
\end{align*}
where we use the fact that, if $\event{\Delta_{Z(t)} < \Gamma_{Z(t)}}$, it must be $\Delta_{Z(t)} = 0$, since, otherwise, it should be that either $\tilde Z(t) = \emptyset$, or $\Delta_{Z(t)} < \Gamma_{Z(t)} = \max_{(i,j) \in \tilde Z(t)} \Delta^{i,j}_{\min} = \Delta^{i',j'}_{\min}$, for some $(i',j') \in \tilde Z(t)$. However, by the structure of \eqref{lp:LP}, we know that $\forall Z \in \extr$, $\tilde Z \neq \emptyset$, while the fact that $\Delta_{Z(t)} < \Delta^{i',j'}_{\min}$, for some $(i',j') \in \tilde Z(t)$, is a contradiction to the definition of $\Gamma_{Z(t)}$.

By Lemma \ref{lemma:regret:nicesampling}, we have that: $\Ex{\mathcal{R}_{N,\pit}}{\sum_{t \in [T]} \Delta_{Z(t)}\event{\neg\mathcal{N}^{s}_t}} \leq \Delta_{\max} \sum_{t \in [T]} \Pro{\neg\mathcal{N}^{s}_t} \leq \frac{\pi^2}{3}\cdot k \cdot m \cdot \Delta_{\max} $. Moreover, by Lemma \ref{lemma:regret:nicetriggering}, we have that 
$\Ex{\mathcal{R}_{N,\pit}}{\sum_{t \in [T]} \Delta_{Z(t)}\event{\neg \mathcal{N}^{\tau}_t}}\leq \Delta_{\max} \sum_{t \in [T]}\Pro{\neg \mathcal{N}^{\tau}_t}\leq \frac{\pi^2}{6} \sum_{i \in \A}\sum_{j \in \C} \log_2\left(\frac{2(k+m)}{\Delta^{i,j}_{\min}}\right) \Delta_{\max}$. Finally, in order to complete our bound, it suffices to upper bound 
$$\Ex{\mathcal{R}_{N,\pit}}{\sum_{t \in [T]} \Delta_{Z(t)} \event{\Delta_{Z(t)} \geq \Gamma_{Z(t)}, \mathcal{N}^{s}_t, \mathcal{N}^{\tau}_t }}.$$
For any arm-context pair such that $(i,j) \in Z(t)$ for some extreme point $Z(t) \in \extr$, we define $l^{(t)}_{i,j}$ such that $Z(t) \in \extr_{i,j,l^{(t)}_{i,j}}$. By Lemma \ref{lemma:regret:decomposition}, we have:
\begin{align*}
\Ex{\mathcal{R}_{N,\pit}}{\sum_{t \in [T]} \Delta_{Z(t)} \event{\Delta_{Z(t)} \geq \Gamma_{Z(t)}, \mathcal{N}^{s}_t, \mathcal{N}^{\tau}_t }} &\leq \Ex{\mathcal{R}_{N,\pit}}{\sum_{t \in [T]} \sum_{(i,j) \in \tilde Z(t)} \kappa_{l^{(t)}_{i,j}, T}(\Delta^{i,j}_{\min}, N_{i,j,l^{(t)}_{i,j}}(t-1))} \\
&= \Ex{\mathcal{R}_{N,\pit}}{\sum_{i \in \A}\sum_{j \in \C} \sum^{+\infty}_{l = 1} \sum^{N_{i,j,l}(T)-1}_{s = 0}\kappa_{l, T}(\Delta^{i,j}_{\min}, s)},
\end{align*}
where the last equality follows by the fact that $N_{i,j,l_{i,j}}$ is increased if and only if $(i,j) \in \tilde Z(t)$. Now, for every arm $i \in \A$, context $j \in \C$ and $l \in \mathbb{N}_{+}$ and by definition of $\kappa_{l,T}(\Delta,s)$ we have:
\begin{align}
\sum^{N_{i,j,l}(T)-1}_{s = 0}\kappa_{l, T}(\Delta^{i,j}_{\min}, s) &\leq \sum^{\ell_{l,T}(\Delta^{i,j}_{\min})}_{s = 0}\kappa_{l, T}(\Delta^{i,j}_{\min}, s)\label{eq:regret:1}\\
&= \kappa_{l, T}(\Delta^{i,j}_{\min}, 0) + \sum^{\ell_{l,T}(\Delta^{i,j}_{\min})}_{s = 1}\kappa_{l, T}(\Delta^{i,j}_{\min}, s) \nonumber\\
&= \kappa_{l, T}(\Delta^{i,j}_{\min}, 0) + \sum^{\ell_{l,T}(\Delta^{i,j}_{\min})}_{s = 1}2\sqrt{\frac{4 \dummy \ln(T)\cdot 2^{-l}}{s}} \nonumber\\
&\leq \kappa_{l, T}(\Delta^{i,j}_{\min}, 0) + 4\sqrt{\dummy\cdot 2^{-l} \cdot \ln(T)} \sum^{\ell_{l,T}(\Delta^{i,j}_{\min})}_{s = 1}\sqrt{\frac{1}{s}} \nonumber\\
&\leq \kappa_{l, T}(\Delta^{i,j}_{\min}, 0) + 8\sqrt{\dummy\cdot 2^{-l} \cdot \ln(T)} \sqrt{\ell_{l,T}(\Delta^{i,j}_{\min})} \label{eq:regret:2}
\end{align}
, where \eqref{eq:regret:1} follows by the fact that $\kappa_{l,T}(\Delta,s) = 0$, for $s \geq \ell_{l,T}(\Delta) + 1$, while \eqref{eq:regret:2}, follows by the fact that for any integer $n \in \mathbb{N}_+$, we have: $\sum^n_{s=1}\sqrt{\frac{1}{s}} \leq \int^n_{s=0} \sqrt{\frac{1}{s}} ds = 2\sqrt{n}$. Using the definition of $\ell_{l,T}$, then \eqref{eq:regret:2} becomes:
\begin{align}
\sum^{N_{i,j,l}(T)-1}_{s = 0}\kappa_{l, T}(\Delta^{i,j}_{\min}, s) &\leq \kappa_{l, T}(\Delta^{i,j}_{\min}, 0) + 8\sqrt{\dummy\cdot 2^{-l} \cdot \ln(T)} \sqrt{\ell_{l,T}(\Delta^{i,j}_{\min})} \nonumber\\
&\leq \kappa_{l, T}(\Delta^{i,j}_{\min}, 0) + 8\sqrt{\dummy\cdot 2^{-l} \cdot \ln(T)} \sqrt{\frac{96\cdot 2^{-l}\cdot \dummyy \cdot (k+m)^2 \ln(T)}{\left(\Delta^{i,j}_{\min}\right)^2}} \nonumber\\
&= 4\cdot 2^{-l} + 8\sqrt{96\cdot \dummy \cdot \dummyy}\cdot 2^{-l} \cdot \frac{(k+m) \cdot \ln(T)}{\Delta^{i,j}_{\min}}.\nonumber
\end{align}

By summing over all $l$, for each pair $(i,j)$, we have:
\begin{align*}
    \sum^{+\infty}_{l=1}\sum^{N_{i,j,l}(T)-1}_{s = 0}\kappa_{l, T}(\Delta^{i,j}_{\min}, s) &\leq  4\cdot \sum^{+\infty}_{l=1}2^{-l} + 8\sqrt{96\cdot \dummy \cdot \dummyy}\cdot \frac{(k+m) \cdot \ln(T)}{\Delta^{i,j}_{\min}} \sum^{+\infty}_{l=1} 2^{-l}\\
    &\leq 4 + 8\sqrt{96\cdot \dummy \cdot \dummyy}\cdot \frac{(k+m) \cdot \ln(T)}{\Delta^{i,j}_{\min}}.
\end{align*}

By combining the aforementioned facts, we conclude that: 
\begin{align*}
&\Ex{\mathcal{R}_{N,\pit}}{\sum^{T}_{t = 1} \sum_{i \in \A} \sum_{j \in \C} \mu_{i,j} \left(z^*_{i,j} - z_{i,j}(t)\right)} \leq 8\sqrt{96\cdot \dummy \cdot \dummyy} \sum_{i \in \A}\sum_{j \in \C} \frac{\left(k+m\right) \ln{(T)}}{\Delta^{i,j}_{\min}} + 4\cdot k\cdot m \\
&\quad\quad + \frac{\pi^2}{6}\left(\sum_{i\in\A}\sum_{j\in \C} \log_2{\frac{2\left(k+m\right)}{\Delta^{i,j}_{\min}}}   + 2\cdot k \cdot m\right)\Delta_{\max}
\end{align*}

Finally, combining the above with the upper bound we get from Lemma \ref{lemma:regret:mapping}, we get:
\begin{align*}
&\alpha\Reg^{\pit}(T) \leq  \Ex{\mathcal{R}_{N,\pit}}{\sum^{T-M}_{t = 1} \sum_{i \in \A} \sum_{j \in \C} \mu_{i,j}\left(z^*_{i,j} - z_{i,j}(t)\right)} + +\frac{1}{3}\ln(T)\Delta_{\max} + 6 d_{\max} + 71\\
&\leq 8\sqrt{96\cdot \dummy \cdot \dummyy} \sum_{i \in \A}\sum_{j \in \C} \frac{\left(k+m\right) \ln{(T)}}{\Delta^{i,j}_{\min}} + 4\cdot k\cdot m \\
&\quad\quad\quad\quad\quad\quad+ \frac{\pi^2}{6}\left(\sum_{i\in\A}\sum_{j\in \C} \log_2{\frac{2\left(k+m\right)}{\Delta^{i,j}_{\min}}}   + 2\cdot k \cdot m + \frac{2}{\pi^2} \ln(T) \right)\Delta_{\max} + 6 d_{\max} + 71\\
&\leq 10898 \sum_{i \in \A}\sum_{j \in \C} \frac{\left(k+m\right) \ln{(T)}}{\Delta^{i,j}_{\min}} + 4\cdot k\cdot m \\
&\quad\quad\quad\quad\quad\quad+ \frac{\pi^2}{6}\left(\sum_{i\in\A}\sum_{j\in \C} \log_2{\frac{2\left(k+m\right)}{\Delta^{i,j}_{\min}}}   + 2\cdot k \cdot m + \frac{2}{\pi^2} \ln(T) \right)\Delta_{\max} + 6 d_{\max} + 71.
\end{align*}
The above regret bound completes our proof. 
\end{proof}

\subsection{Proof of Lemma \ref{lemma:regret:nicesampling} [Nice Sampling]}

\restateNiceSampling*
\begin{proof}
Let $\neg \mathcal{N}^s_t$ be the event that the algorithm does not have a nice sampling at some round $t \in [T]$. By union bound on the possible arm-context pairs, we have:
\begin{align*}
    \Pro{\neg \mathcal{N}^s_t} &= \Pro{\exists i\in \A, j\in \C, ~s.t.~|\hat\mu_{i,j,T_{i,j}(t)} - \mu_{i,j}| > \sqrt{\frac{3 \ln{(t)}}{2T_{i,j}(t)}}} \\ 
    &\leq \sum_{i \in \A} \sum_{j \in \C} \Pro{|\hat\mu_{i,j,T_{i,j}(t)} - \mu_{i,j}| > \sqrt{\frac{3 \ln{(t)}}{2T_{i,j}(t)}}} \\
    &\leq \sum_{i \in \A} \sum_{j \in \C} \sum^{t}_{s = 1} \Pro{\left|\hat\mu_{i,j,s} - \mu_{i,j} \right| > \sqrt{\frac{3 \ln{(t)}}{2s}}}.
\end{align*}
For any $s \in [t]$, $\hat\mu_{i,j,s}$ is the average of $s$ i.i.d. random variables, denoted by $X_{i,j}^{[1]}, \dots, X_{i,j}^{[s]}$, drawn from the reward distribution of arm $i \in \A$, when it is played under context $j \in \C$. For any fixed $s \in [t]$ and for any pair $i \in \A, j \in \C$, we have: 
\begin{align*}
    \Pro{|\hat\mu_{i,j,s} - \mu_{i,j}| > \sqrt{\frac{3 \ln{(t)}}{2s}}} &= \Pro{|\frac{\sum_{b \in [s]} X_{i,j}^{[b]}}{s} - \mu_{i,j}| > \sqrt{\frac{3 \ln{(t)}}{2s}}} \\
    &= \Pro{|\sum_{b \in [s]} X_{i,j}^{[b]} - \mu_{i,j} s| \geq \sqrt{\frac{3 s \ln{(t)}}{2}}} \\
    &\leq 2 \exp\left(- 2\frac{3s \ln(t)}{2s}\right) = t^{-3},
\end{align*}
where we use Hoeffding's inequality (see Appendix \ref{appendix:concentration}) for upper bounding the last probability. By combining the above inequalities, we have: 
\begin{align*}
    \Pro{\neg \mathcal{N}^s_t} &\leq \sum_{i \in \A} \sum_{j \in \C} \sum^{t}_{s = 1} \Pro{\left|\hat\mu_{i,j,s} - \mu_{i,j} \right| \geq \sqrt{\frac{3 \ln{t}}{2s}}} \\
    &\leq m k t^{-2}. 
\end{align*}
\end{proof}

\subsection{Proof of Lemma \ref{lemma:regret:nicetriggering} [Nice Triggering]}

\restateNiceTriggering*
\begin{proof}
Recall that $\dummy = 109\cdot e$ and $\dummyy = 24 \cdot e$ and consider the case where $t-1 \geq N_{i,j,l}(t-1)\geq \dummy \cdot 2^l \ln(t)$. By union bound, we have: 
\begin{align}
\Pro{\neg \mathcal{N}^{\tau}_t} &= \Pro{\exists i \in \A, \exists j \in \C, \exists l \in \bigg[1, \left[ \log_2(\frac{2(k+m)}{\Delta^{i,j}_{\min}})\right]^+\bigg], T_{i,j}(t-1) \leq \frac{1}{\dummyy} N_{i,j,l}(t-1)2^{-l}} \nonumber\\
&\leq \sum_{i \in \A} \sum_{j \in \C} \sum^{\left[ \log_2(\frac{2(k+m)}{\Delta^{i,j}_{\min}})\right]^+}_{l = 1} \Pro{T_{i,j}(t-1) \leq \frac{1}{\dummyy} N_{i,j,l}(t-1)2^{-l}}\nonumber\\
&\leq \sum_{i \in \A} \sum_{j \in \C} \sum^{\left[ \log_2(\frac{2(k+m)}{\Delta^{i,j}_{\min}})\right]^+}_{l = 1} \sum^t_{s= \lceil \dummy \cdot 2^l \ln(t) \rceil} \Pro{N_{i,j,l}(t-1) = s , T_{i,j}(t-1) \leq \frac{1}{\dummyy} N_{i,j,l}(t-1)2^{-l}} \label{eq:regret:trigg}.
\end{align}
By Lemma \ref{lemma:regret:counter-to-samples}, and since we consider only $N_{i,j,l}(t-1)\geq \dummy \cdot 2^l \log(t) = 109 \cdot e \cdot 2^l \log(t)$, inequality \eqref{eq:regret:trigg} can be further upper bounded by:
\begin{align*}
\Pro{\neg \mathcal{N}^{\tau}_t} &\leq \sum_{i \in \A} \sum_{j \in \C} \sum^{\left[ \log_2(\frac{2(k+m)}{\Delta^{i,j}_{\min}})\right]^+}_{l = 1} \sum^t_{s= \lceil \dummy \cdot 2^l \ln(t) \rceil} \frac{1}{t_3} \\
&\leq \sum_{i \in \A} \sum_{j \in \C} \left[ \log_2(\frac{2(k+m)}{\Delta^{i,j}_{\min}})\right]^+ t^{-2}.
\end{align*}
\end{proof}

\subsection{Proof of Lemma \ref{lemma:regret:decomposition} (Suboptimality Decomposition)}

\restateSubDecomposition*
\begin{proof}
Clearly, we are only interested in the rounds $t \in [T]$ such that $\Delta_{Z(t)} >0$, since, otherwise, the inequality holds trivially. By optimality of \eqref{lp:LP} at time $t$ (i.e. the solution of \eqref{lp:LP} at time $t$ using the indices $\{\bar\mu_{i,j}(t)\}_{\forall i,j}$), we have that:
\begin{align*}
\sum_{i \in \A}\sum_{j \in \C} \bar\mu_{i,j}(t) z_{i,j}(t) \geq \sum_{i \in \A}\sum_{j \in \C} \bar\mu_{i,j}(t) z^*_{i,j}.
\end{align*}

Moreover, by the nice sampling assumption on round $t$, we have: 
\begin{align*}
\sum_{i \in \A}\sum_{j \in \C} \bar\mu_{i,j}(t) z^*_{i,j} \geq \sum_{i \in \A}\sum_{j \in \C} \mu_{i,j} z^*_{i,j},
\end{align*}
given that under $\mathcal{N}^s_t$, each index overestimates the actual mean value, namely, $\bar\mu_{i,j}(t) \geq \mu_{i,j}, \forall i \in \A, j \in \C$.

Finally, by definition of the suboptimality gap, we have that $\Delta_{Z(t)} =\sum_{i \in \A}\sum_{j \in \C} \mu_{i,j} z^*_{i,j} - \sum_{i \in \A}\sum_{j \in \C} \mu_{i,j} z_{i,j}(t)$. By combining the above facts, we get: 
\begin{align*}
   \Delta_{Z(t)} &= \sum_{i \in \A}\sum_{j \in \C} \mu_{i,j} z^*_{i,j} - \sum_{i \in \A}\sum_{j \in \C} \mu_{i,j} z_{i,j}(t)\\
   &\leq \sum_{(i,j) \in \tilde{Z}(t)} \left(\tilde\mu_{i,j}(t) - \mu_{i,j}\right) z_{i,j}(t).
\end{align*}
Now, by assumption that $\Delta_{Z(t)} \geq \Gamma_{Z(t)} = \max_{(i,j) \in \tilde{Z}}\{ \Delta^{i,j}_{\min}\}$ and using the above inequality, we have: 
\begin{align*}
\Delta_{Z(t)} &\leq - \Gamma_{Z(t)} + 2 \sum_{(i,j) \in \tilde{Z}(t)} \left(\bar\mu_{i,j}(t) - \mu_{i,j}\right) z_{i,j}(t)\\
&= 2 \sum_{(i,j) \in \tilde{Z}(t)} \left( \left(\bar\mu_{i,j}(t) - \mu_{i,j}\right) z_{i,j}(t)  - \frac{\Gamma_{Z(t)}}{2|\tilde{Z}(t)|}\right)\\
&\leq 2 \sum_{(i,j) \in \tilde{Z}(t)} \left( \left(\bar\mu_{i,j}(t) - \mu_{i,j}\right) z_{i,j}(t)  - \frac{\Delta^{i,j}_{\min}}{2(k+m)}\right),
\end{align*}
where in the last inequality, we use the fact that, by Lemma \ref{lemma:regret:sparse}, we have $|\tilde{Z}(t)| \leq k +m$, and that for any pair $(i,j) \in \tilde{Z}(t)$, we have $\Gamma_{Z(t)} \geq \Delta^{i,j}_{\min}$.

For any $(i,j) \in \tilde{Z}(t)$, let $l_{i,j}$ be the index such that $Z(t) \in \extr_{i,j,l_{i,j}}$. For each $(i,j) \in \tilde{Z}(t)$, we are trying to upper bound $2\left( \left(\bar\mu_{i,j}(t) - \mu_{i,j}\right) z_{i,j}(t)  - \frac{\Delta^{i,j}_{\min}}{2(k+m)}\right)$, by distinguishing between two cases on the value of $l_{i,j}$.

\textbf{Case (a):} $1 \leq l_{i,j} \leq \lceil \log \frac{2(k+m)}{\Delta^{i,j}_{\min}} \rceil_{+}$. By $\mathcal{N}^{s}_t$, we have that $\bar\mu_{i,j}(t) - \mu_{i,j}\leq 2 \sqrt{\frac{3\ln{(t)}}{2T_{i,j}(t)}}$, while by definition of TB groups, we have $z_{i,j}(t) \leq 2^{-l_{i,j}+1}$. We further distinguish between sub-cases.

\textbf{Sub-case (i):} $N_{i,j,l_{i,j}}(t-1) = 0$. In that case, we have that $\kappa_{l_{i,j},T}(\Delta^{i,j}_{\min}, 0) = 4\cdot 2^{-l_{i,j}}$ and, thus: 
\begin{align*}
2\left(\left(\bar\mu_{i,j}(t) - \mu_{i,j}\right) z_{i,j}(t)  - \frac{\Delta^{i,j}_{\min}}{2(k+m)} \right) &\leq 2 \left(\bar\mu_{i,j}(t) - \mu_{i,j}\right) z_{i,j}(t) \\
&\leq 2\cdot 2^{-l_{i,j}+1} \\
&= 4 \cdot 2^{-l_{i,j}} \\
&= \kappa_{l_{i,j},T}(\Delta^{i,j}_{\min}, 0).
\end{align*}

\textbf{Sub-case (ii):} $\sqrt{\frac{\dummy \ln(t)}{N_{i,j,l_{i,j}}(t-1)2^{-l_{i,j}}}} \geq 1$. Then we have that: 
\begin{align*}
2\left(\left(\bar\mu_{i,j}(t) - \mu_{i,j}\right) z_{i,j}(t)  - \frac{\Delta^{i,j}_{\min}}{2(k+m)} \right) 
&\leq 2 \left(\bar\mu_{i,j}(t) - \mu_{i,j}\right)z_{i,j}(t) \\
&\leq 2 \cdot 2^{-l_{i,j}+1} \\
&\leq 2 \cdot 2^{-l_{i,j}+1} \sqrt{\frac{\dummy \ln(t)}{ N_{i,j,l_{i,j}}(t-1)2^{-l_{i,j}}}}\\ 
&\leq 2 \cdot \sqrt{\frac{4 \dummy \cdot  2^{-l_{i,j}} \ln(t)}{N_{i,j,l_{i,j}}(t-1)}}\\
&= \kappa_{l_{i,j},T}(\Delta^{i,j}_{\min}, N_{i,j,l_{i,j}}(t-1)).
\end{align*}

\textbf{Sub-case (iii):}
$\sqrt{\frac{\dummy \ln(t)}{N_{i,j,l_{i,j}}(t-1)2^{-l_{i,j}}}} \leq 1$. Then by $\mathcal{N}^{\tau}_t$ and $\mathcal{N}^{s}_t$, we have: 
\begin{align*}
    \bar\mu_{i,j}(t) - \mu_{i,j} \leq 2\sqrt{\frac{3 \ln(t)}{2 T_{i,j}(t-1)}} \leq 2\sqrt{\frac{3 \dummyy \ln(t)}{2 N_{i,j,l_{i,j}}(t-1)\cdot 2^{-l_{i,j}}}}.
\end{align*}
Therefore, we have that: 
\begin{align*}
    \left(\bar\mu_{i,j}(t) - \mu_{i,j}\right) z_{i,j}(t) \leq \min\bigg\{\sqrt{\frac{24 \dummyy \ln(t)\cdot 2^{-l_{i,j}}}{N_{i,j,l_{i,j}}(t-1)}}, 2\cdot2^{-l_{i,j}}\bigg\}.
\end{align*}
Now, in the case where $N_{i,j,l_{i,j}}(t-1) \geq \ell_{l_{i,j},T}(\Delta^{i,j}_{\min}) + 1$, we have: $\sqrt{\frac{24 \dummyy \ln(t)\cdot 2^{-l_{i,j}}}{N_{i,j,l_{i,j}}(t-1)}} \leq \sqrt{\frac{24 \dummyy \ln(t)\cdot 2^{-l_{i,j}} (\Delta^{i,j}_{\min})^2}{96 \cdot \dummyy 2^{-l_{i,j}} (k+m)^2 \ln{(T)}}} \leq \sqrt{\frac{(\Delta^{i,j}_{\min})^2}{4(k+m)^2}} = \frac{\Delta^{i,j}_{\min}}{2(k+m)}$, and, thus, $2\left(\left(\bar\mu_{i,j}(t) - \mu_{i,j}\right) z_{i,j}(t)  - \frac{\Delta^{i,j}_{\min}}{2(k+m)} \right) \leq 2\left(\frac{\Delta^{i,j}_{\min}}{2(k+m)} - \frac{\Delta^{i,j}_{\min}}{2(k+m)} \right) \leq 0 = \kappa_{l_{i,j},T}(\Delta^{i,j}_{\min},N_{i,j,l_{i,j}}(t-1))$. 
In the case where $N_{i,j,l_{i,j}}(t-1) \leq \ell_{l_{i,j},T}(\Delta^{i,j}_{\min})$, we simply use $2\left(\left(\bar\mu_{i,j}(t) - \mu_{i,j}\right) z_{i,j}(t)  - \frac{\Delta^{i,j}_{\min}}{2(k+m)} \right) \leq 2\left(\left(\bar\mu_{i,j}(t) - \mu_{i,j}\right) z_{i,j}(t)  - \frac{\Delta^{i,j}_{\min}}{2(k+m)} \right) \leq \sqrt{\frac{96 \dummyy \ln(t)\cdot 2^{-l_{i,j}}}{N_{i,j,l_{i,j}}(t-1)}} \leq 2\sqrt{\frac{4 \dummy \ln(t)\cdot 2^{-l_{i,j}}}{N_{i,j,l_{i,j}}(t-1)}}$.

\textbf{Case (b):} $l_{i,j} \geq \lceil \log \frac{2(k+m)}{\Delta^{i,j}_{\min}} \rceil_{+} +1$. Using the fact that $\bar\mu_{i,j}(t) - \mu_{i,j} \leq 1$ and the definition of TB groups, we have: 
\begin{align*}
    \left(\bar\mu_{i,j}(t) - \mu_{i,j}\right)z_{i,j}(t) \leq z_{i,j}(t) \leq 2^{-l_{i,j}+1} \leq 2^{- \lceil \log \frac{2(k+m)}{\Delta^{i,j}_{\min}} \rceil_{+}} \leq 2^{- \log \frac{2(k+m)}{\Delta^{i,j}_{\min}}} \leq \frac{\Delta^{i,j}_{\min}}{2(k+m)}.
\end{align*}
By using the non-negativity of $\kappa_{l_{i,j}, T}(\Delta^{i,j}_{\min}, N_{i,j,l_{i,j}}(t-1))$, the above implies that: 
\begin{align*}
\frac{\Delta^{i,j}_{\min}}{2(k+m)} - \frac{\Delta^{i,j}_{\min}}{2(k+m)} \leq 0 \leq \frac{1}{2}\kappa_{l_{i,j}, T}(\Delta^{i,j}_{\min}, N_{i,j,l_{i,j}}(t-1)),
\end{align*}
and, thus, 
\begin{align*}
2\left( \left(\bar\mu_{i,j}(t) - \mu_{i,j}\right) z_{i,j}(t)  - \frac{\Delta^{i,j}_{\min}}{2(k+m)}\right)\leq \kappa_{l_{i,j}, T}(\Delta^{i,j}_{\min}, N_{i,j,l_{i,j}}(t-1)).
\end{align*}

\end{proof}
\newpage
\section{Hardness results: omitted proofs} \label{appendix:hardness}

\subsection{Proof of Theorem \ref{hardness:thm:competitive}}

\hardnesscompetitive*

\begin{proof}
We now prove an upper bound on the (asymptotic) competitive ratio of the full-information case of our problem. It suffices to provide an instance $I$, such that the ratio between the expected reward collected by an (asymptotically) optimal online policy, denoted by $\lim_{T \to +\infty} \Rew_I^{\mathrm{opt}}(T)$, and by an optimal clairvoyant policy, denoted by $\lim_{T \to +\infty} \Rew_I^{*}(T)$, is upper bounded by $\frac{d_{\max}}{2d_{\max}-1}$. Recall, that a clairvoyant policy has a priori knowledge of all context realizations, $\{C_t\}_{\forall t \in [T]}$.

Consider the following instance $I$. Let $\A$ be a set of $k$ arms and let arm $i^m$ such that $i^m = \arg\max_{i' \in \A} d_{i'}$, namely, an arm of maximum possible delay. Let $\C = \{1, 2\}$ be a set of two contexts, such that $f_{1} = \epsilon$ and $f_{2} = 1 - \epsilon$, for some small $\epsilon \in (0,1)$. We assume that the rewards $\{X_{i,j,t}\}_{\forall i \in \A, j \in \C, t \in [T]}$ are constants, while the rewards of all arms except for $i^m$, i.e., $\A \setminus \{i^m\}$, are identically equal to zero for any possible context. The above implies that, without loss of generality, neither the optimal clairvoyant policy, nor the optimal online policy ever play these arms and, thus, we can assume that only arm $i^m$ is played. For arm $i^m$, we have that $X_{i^m,1,t} = \mu_{i^m,1} = \frac{R}{\epsilon}$ for some fixed $R > 0$ and $X_{i^m,2,t} = \mu_{i^m,2} = 1$, for all $t \in [T]$. We note that the reward $\tfrac{R}{\epsilon}$ may be greater than $1$, which can be fixed by dividing all the rewards by $\left(1 + \frac{R}{\epsilon}\right)$, in order to keep them within range $[0,1]$.

In this proof, we compute the average reward collected by an optimal (non-clairvoyant) online and we lower bound the average reward collected by an optimal clairvoyant policy for instance $I$. Given that $i^m$ is the only arm played in both cases, we focus only on this arm and we simplify the notation by referring to it as $i$.

\paragraph{Online Policy.} We focus on arm $i$ and consider the behavior of a specific online policy, denoted by $\mathrm{alg}(q_1, q_2)$. This online policy starts at time $t = r$ for $r \in \{0, \dots, d_i-1\}$ with probability $\pi(r)$ (to be specified later). At each time, if the arm $i$ is available and the context is $1$ (resp., $2$) it plays the arm $i$ with probability $q_1$ (resp., $q_2$).

The behavior of this online policy $\mathrm{alg}(q_1, q_2)$ can be analyzed using a Markov chain. Specifically, the Markov chain has $d_i$ states, $0, 1, \dots, {d_i-1}$, where each state $r$ indicates the fact that arm $i$ is blocked (i.e., not available) for the next $r$ rounds. Let $q_1, q_2 \in [0,1]$ (determined by the policy $\mathrm{alg}(q_1, q_2)$) denote the probabilities that arm $i$ is played, if available, given that the context is $1$ and $2$, respectively. At each time $t$, the Markov chain moves from state $0$ to state $(d_i-1)$ with probability $(q_1 f_{1} + q_2 f_{2})$ and gains the expected reward $(\mu_{i,1} q_1 f_{1} + \mu_{i,2} q_2 f_{2})$. Otherwise, the Markov chain remains in state $0$ with probability $(1 - q_1 f_{1} - q_2 f_{2})$. Given that at some time $t$, the state is $r$, for $r \geq 1$, the Markov chain deterministically moves to the state $(r-1)$ (collecting zero reward). Let $\pi(r)$  be the stationary probability of state $s_r$ in the above Markov chain, which is parameterized by $q_1$ and $q_2$. This is the same $\pi(r)$ that is used in the definition of the policy $\mathrm{alg}(q_1, q_2)$. 

We can compute the probability $\pi(0)$ by solving the system: $\sum_{r \in \{0 , \dots, d_i-1\}} \pi(r) = 1$ and that $\pi(1) = \pi(2) = \dots = \pi(d_i-1) = \pi(0) \left(q_1 f_{1} + q_2 f_{2}\right)$. Recall that the expected reward $(\mu_{i,1} q_1 f_{1} + \mu_{i,2} q_2 f_{2})$ is collected only when the Markov chain is at state $0$ (and moves to state $d_i-1$). Finally, we let the Markov chain start from stationary state, i.e.  at time $t = 0$ the Markov chain is in state $r$ w.p. $\pi(r)$. Due to stationarity, the expected average reward for the above online policy $\mathrm{alg}(q_1, q_2)$, for any time horizon $T$, denoted by $\Rew^{\mathrm{alg}(q_1,q_2)}_I(T)$, can be expressed as: 
\begin{align*}
\Rew^{\mathrm{alg}(q_1, q_2)}_I &= \Ex{\mathcal{R}_{N,\mathrm{alg}(q_1,q_2)}}{\frac{1}{T} \sum_{t \in [T]} \sum_{j \in \C} \event{A^{\pi(q_1,q_2)}_t = i, C_t = j}}  
= \frac{R q_1 + (1-\epsilon) q_2}{1 + (d_i-1)(\epsilon q_1 + (1- \epsilon) q_2)}
\end{align*}

We have already argued that for the setting under consideration, there exists an optimal online policy which only plays arm $i$. Further from the theory of Markov decision processes (MDP), as the time horizon $T$ tends to infinity, there exists an optimal online policy (which only plays arm $i$) that is represented by the above stationary Markov chain (c.f.,~\citep{P14}). In particular, this optimal online policy can be designed by maximizing the time-average expected reward over the probabilities $q_1$ and $q_2$. Therefore, computing the optimal time-average expected reward in our setting can be formulated as the following optimization program:
\begin{align}
\textbf{maximize: }f(q_1,q_2) = \frac{R q_1 + (1-\epsilon) q_2}{1 + (d_i-1)(\epsilon q_1 + (1- \epsilon) q_2)} \textbf{ s.t. }q_1,q_2 \in [0,1]. \label{eq:hardness:mp}
\end{align}
The following lemma specifies the solution of the above optimization problem for a specific range of $(R, \epsilon)$.

\begin{restatable}{lemma}{lemmaHardnessOpt}\label{lemma:hardness:mdp}
For $R > \epsilon + \frac{1}{d_i-1}$, the optimal solution to the mathematical program \eqref{eq:hardness:mp} is attained by setting $(q_1,q_2) = (1,0)$ and its value is equal to $\frac{R}{1+(d_i-1)\epsilon}$.
\end{restatable}

\paragraph{Lower bound on the optimal clairvoyant policy.} We now need to compute the expected average reward of an optimal clairvoyant policy on our instance, namely, a policy that has a priori knowledge of the context realizations of all rounds. However, given that providing a characterization of the optimal solution for any possible context realization is a difficult task, we instead attempt to lower bound the optimal expected reward. For this reason, we study a simpler and (possibly) suboptimal clairvoyant policy. This policy is based on partitioning the time horizon into blocks of size $B$, and, then, treating each block, separately, using a simple strategy. 

We define the block size to be $B = k d_i$, where $k \in \mathbb{N}_{+}$ is a positive natural number such that $k \geq 2$ and $d_i$ is the delay of arm $i$. We further assume without loss of generality that the time horizon $T$, which we later extend to infinity, is a multiple of the block size $B$. Our algorithm works separately, in each of the $\frac{T}{B}$ blocks, according to the following simple rule:\\
\textbf{Case (a)}: If context $1$ appears at exactly one time $t'$ within the first $B-d_i$ rounds of the block, then the algorithm plays the arm on time $t'$ and nothing else.\\
\textbf{Case (b)}: If context $1$ does not appear at all within the $B$ rounds of the block, the algorithm plays arm $i$ exactly $k-1$ times (every $d_i$ times), starting from the first round of the and excluding the last $d_i$ rounds of the block.\\ \textbf{Case (c)}: In any other case, the algorithm takes no action during the $B$ rounds of the block. 

It is important to notice that, in all the aforementioned cases, no action is taken within the last $d_i$ rounds of each block. This allows us to study the expected reward of each block independently.

The average reward collected by the above policy is at least,
\begin{align*}
&\tfrac{1}{B}\left(\sum_{t=1}^{B+1-d_i} \tfrac{R}{\epsilon} \epsilon (1-\epsilon)^{(B-1)} 
+ \tfrac{(B-d_i)}{d_i} (1-\epsilon)^B \right)
= R(1-\tfrac{d_i}{B})(1-\epsilon)^{(B-1)} + (\tfrac{1}{d_i} - \tfrac{1}{B}) (1-\epsilon)^B.
\end{align*}

Therefore, for $R > \epsilon + \tfrac{1}{d_i-1}$ and using Lemma~\ref{lemma:hardness:mdp}, the ratio of reward of the online policy over the designed clairvoyant policy is upper bounded by:
\begin{align*}
    \frac{\frac{R}{1+ (d_i-1)\epsilon}}{R(1-\frac{d_i}{B})(1-\epsilon)^{(B-1)} + (\frac{1}{d_i} - \frac{1}{B}) (1-\epsilon)^B}.
\end{align*}

We now consider a series of instances, where $B = d_i \lceil\tfrac{1}{\sqrt{\epsilon}}\rceil$, $R = 2\epsilon + \tfrac{1}{d_i-1}$ and $\epsilon$ approaches $0$. The limiting competitive ratio of an optimal online algorithm becomes:
\begin{align*}
    &\lim_{\epsilon \to 0}
    \frac{\frac{R}{1+ (d_i-1)\epsilon}}{R(1-\frac{d_i}{B})(1-\epsilon)^{(B-1)} + (\frac{1}{d_i} - \frac{1}{B}) (1-\epsilon)^B}
    = \frac{\frac{1}{d_i-1}}{\frac{1}{d_i-1} + \frac{1}{d_i}} = \frac{d_i}{2d_i-1},
\end{align*}
where we use the fact that $\lim_{\epsilon \to 0}\left(1 - \epsilon\right)^{\frac{1}{\sqrt{\epsilon}}} = 1$.

Given that $d_i = d_{\max}$ for instance $I$, we can conclude that the optimal asymptotic competitive ratio of the full-information case of our problem can be upper bounded by $\frac{d_{\max}}{2d_{\max}-1}$.
\end{proof}

\subsection{Proof of Lemma \ref{lemma:hardness:mdp}}
\lemmaHardnessOpt*
\begin{proof}
Let $f(q_1,q_2) = \frac{R q_1 + (1-\epsilon) q_2}{1 + (d_i-1)(\epsilon q_1 + (1- \epsilon) q_2)}$. Taking the partial derivative of $f(q_1,q_2)$ with respect to $q_1$, we get:
\begin{align*}
    \frac{\partial f(q_1,q_2)}{\partial q_1} &= \frac{R\left(1+ (d_i-1)(q_1\epsilon+q_2(1-\epsilon))\right) - \epsilon (d_i-1)(R q_1+q_2(1-\epsilon))}{(1+ (d_i-1)(q_1\epsilon+q_2(1-\epsilon)))^2} \\
    &= \frac{R + (d_i-1)q_2(1-\epsilon)(R-\epsilon)}{(1+ (d_i-1)(q_1\epsilon+q_2(1-\epsilon)))^2}.
\end{align*}
Therefore, for $R> \epsilon$ we have $\frac{\partial f(q_1,q_2)}{\partial q_1} > 0$ for all $q_1, q_2 \in [0,1]$ and, thus, the optimal solution in this case is attained at $q^*_1 = 1$. We now take the derivative of $f(q_1,q_2)$ with respect to $q_2$: 
\begin{align*}
    \frac{\partial f(q_1,q_2)}{\partial q_2} &= \frac{(1-\epsilon)(1+ (d_i-1)(q_1\epsilon+q_2(1-\epsilon))) - (d_i-1)(1-\epsilon)(R q_1+q_2(1-\epsilon))}{(1+ (d_i-1)(q_1\epsilon+q_2(1-\epsilon)))^2} \\
    &= \frac{(1-\epsilon)(1 - (d_i-1)q_1(R-\epsilon))}{(1+ (d_i-1)(q_1\epsilon+q_2(1-\epsilon)))^2}.
\end{align*}

Therefore, for $R> \epsilon$ at $q^*_1 = 1$ we have $\frac{\partial f(q_1,q_2)}{\partial q_2} > 0$, when $(d_i-1)(R-\epsilon) < 1$, and $\frac{\partial f(q_1,q_2)}{\partial q_2} \leq 0$. Therefore, for $R < \epsilon + \frac{1}{d_i-1}$ we have the optimal at $q^*_2 = 1$ and the optimal value is $\frac{R + 1-\epsilon}{d_i}$. For $R > \epsilon + \frac{1}{d_i-1}$ we have the optimal at $q^*_2 = 0$ and the optimal value is $\frac{R}{1+ (d_i-1)\epsilon}$.
\end{proof}
\newpage
\end{document}